\documentclass[11pt,a4paper]{article}

\usepackage[margin=1in]{geometry}

\usepackage[utf8]{inputenc}
\usepackage[english]{babel}

\usepackage[dvipsnames,usenames]{xcolor}
\usepackage[colorlinks=true,pdfpagemode=UseNone,linkcolor=blue,citecolor=blue,pdfstartview=FitH]{hyperref}

\usepackage{subcaption}

\usepackage{mathtools}
\usepackage{amsthm,thmtools} 
\usepackage{amsmath,amssymb}
\usepackage{amsfonts}
\usepackage{bm}
\usepackage{booktabs}

\usepackage[nameinlink]{cleveref}

\usepackage{tikz}
\usepackage{pgfplots}
\pgfplotsset{compat=newest}
\usetikzlibrary{shapes, shapes.symbols, backgrounds, matrix, calc, arrows, math, arrows.meta, decorations.pathreplacing, decorations.markings, patterns, calligraphy}

\usepackage{booktabs}

\usepackage{nicefrac}

\usepackage{natbib}

\usepackage{enumitem}
\usepackage[ruled]{algorithm2e} 

\SetAlFnt{\small}
\SetAlCapFnt{\small}
\SetAlCapNameFnt{\small}
\SetAlCapHSkip{0pt}
\IncMargin{-\parindent}

\providecommand{\email}[1]{\href{mailto:#1}{\nolinkurl{#1}\xspace}}

\usepackage{acronym}

\acrodef{LLM}[LLM]{large language model}
\acrodef{IPO}[IPO]{Identity Preference Optimization}
\acrodef{BT}[BT]{Bradley–Terry}

\def\a{\bm{a}} 
\def\b{\bm{b}} 
 
\def\e{\bm{e}} 
\def\p{\bm{p}} 
\def\s{\bm{s}} 
\def\w{\bm{w}} 
 
\def\v{\bm{v}} 
 
\def\x{\bm{x}} 
\def\y{\bm{y}} 
\def\z{\bm{z}} 
 
\def\softmax{\mathrm{softmax}} 
\def\btheta{\bm{\theta}} 
 
\def\bdelta{\bm{\delta}} 
\def\bpi{\bm{\pi}} 
\def\bmu{\bm{\mu}} 
\def\bnu{\bm{\nu}}

\def\EE{\mathbb{E}}

\newcommand{\bbR}{\mathbb{R}}

\def\H{\mathcal{H}}

\def\T{\mathcal{T}}

\newcommand{\calP}{\mathcal{P}}

\newcommand{\Pref}{\bpi_{\mathrm{ref}}}
\newcommand{\Smith}{\mathrm{Smith}}

\newcommand{\eps}{\varepsilon}

\ifx\theorem\undefined
\newtheorem{theorem}{Theorem}
\fi
\ifx\example\undefined
\newtheorem{example}[theorem]{Example}
\fi
\ifx\property\undefined
\newtheorem{property}{Property}
\fi
\ifx\lemma\undefined
\newtheorem{lemma}[theorem]{Lemma}
\fi
\ifx\proposition\undefined
\newtheorem{proposition}[theorem]{Proposition}
\fi
\ifx\definition\undefined
\newtheorem{definition}[theorem]{Definition}
\fi
\ifx\conjecture\undefined
\newtheorem{conjecture}[theorem]{Conjecture}
\fi
\ifx\definition\undefined

\fi

\title{How Sampling Shapes LLM Alignment:\\ From One-Shot Optima to Iterative Dynamics}

\author{Yurong Chen$^{1}$, Yu He$^{2}$, Michael I. Jordan$^{1,3}$, Fan Yao$^{4}$\\
    $^1$Inria, École Normale Supérieure, PSL Research University \\
    $^2$Northwestern University  \\
    $^3$University of California, Berkeley \\
    $^4$University of North Carolina at Chapel Hill \\
    \texttt{yurong.chen@inria.fr},
    \texttt{yuhe2030@u.northwestern.edu},\\
    \texttt{jordan@cs.berkeley.edu},
    \texttt{fanyao@unc.edu}
}
\date{}

\begin{document}
\maketitle

\begin{abstract}
Standard methods for aligning large language models with human preferences learn from pairwise comparisons among sampled candidate responses and regularize toward a reference policy. Despite their effectiveness, the effects of sampling and reference choices are poorly understood theoretically. We investigate these effects through Identity Preference Optimization, a widely used preference alignment framework and show that proper instance-dependent sampling can yield stronger ranking guarantees, while skewed on-policy sampling can induce excessive concentration under structured preferences. We then analyze iterative alignment dynamics in which the learned policy feeds back into future sampling and reference policies, reflecting a common practice of model-generated preference data. We prove that these dynamics can exhibit persistent oscillations or entropy collapse for certain parameter choices, and characterize regimes that guarantee stability. Our theoretical insights extend to Direct Preference Optimization, indicating the phenomena we captured are common to a broader class of preference-alignment methods. Experiments on real-world preference data validate our findings. 
\end{abstract}

\section{Introduction}
\label{sec:introduction}

\Acp{LLM} are being employed increasingly in decision-focused tasks where the outputs are actionable entities such as code, summaries, or long-form answers to questions~\citep{openai2024gpt4ocard,anthropic2024claude}.
In such tasks, there is a particularly salient need for reliable, transparent mechanisms that align model outputs with human preferences and normative standards~\citep{guan2025deliberative}. 
%

Standard alignment methods for \ac{LLM}s generally involve optimizing an
objective based on pairwise comparison data, regularized toward a reference policy~\citep{ouyang2022training}. 
Most existing theoretical analyses treat the data-collection procedure and the reference policy as exogenous when studying
the resulting regularized objectives~\citep{azar2024general}. Adaptive response sampling has been analyzed mainly under structured preference models such as the 
Bradley--Terry (BT) model~\citep{shi2024crucial}, under which the population optimum can be sampling-invariant for certain objectives (see, e.g., \Cref{prop:sample_invariant}). However, such invariance is fragile as it relies on 
the stringent BT assumption; for general preference matrices, how sampling shapes the learned policy remains underexplored.

Moreover, in long-run iterative fine-tuning, both the sampling distribution and the reference policy are increasingly endogenous: practitioners mix in
preference data generated by current models to improve training~\citep{liu2023statistical,dong2023raft}, and periodically refresh the reference model to track the latest foundation checkpoint~\citep{gorbatovski2024learn,trl_dpotrainer_docs,kim2025understanding}.
Such workflows induce a self-reinforcing loop: the current model influences the training design
in the next update, thereby steering the learned policy over time.
Consequently, minor design choices can compound across rounds, potentially leading to unintended behavior. The long-run consequences of this endogeneity are still poorly understood.

In this work, we take a first step toward understanding 
these effects in preference-based alignment by focusing on \ac{IPO}, an analytically tractable 
objective 
representative of modern direct alignment methods. 
In the offline, one-shot setting (\Cref{sec:offline}), we study 
the sampling effects on the IPO solution in terms of (i) satisfiability of natural ranking desiderata (informally, ``better responses should receive higher probability'') and (ii) its skewness or extremeness.
We show that sampling is a critical yet double-edged design choice. While \ac{IPO} under any fixed sampling fails all ranking axioms considered in this paper, proper instance-dependent sampling can recover strong ranking guarantees. Meanwhile, under structured preference models, skewed or on-policy sampling can provably amplify policy concentration, 
leading to more extreme and less randomized policies. 

We then turn to the online regime of iterative alignment~(\Cref{sec:self_ref_dynamics}). We define an iterative dynamic in which the policy learned in one round will influence the sampling distribution and reference policy in the next,
with parameters controlling the mixture of on-/off-policy sampling and the degree of reference refresh. We show that the long-run behavior depends sharply on these design knobs. In particular, when updates are overly aggressive and sampling is highly on-policy, the dynamics can diverge and oscillate under cyclic preferences, while under strongly transitive preferences, it can drive policy collapse toward extreme solutions. 
Together with our complementary analysis of Direct Preference Optimization (DPO) in \Cref{app:offline_dpo}, which exhibits similar dynamical behavior, these results suggest that instability and policy collapse are generic failure modes of self-reinforcing alignment loops and can arise across preference-optimization methods, underscoring the need for careful training design. 

Finally, we characterize parameter regimes that guarantee convergence,
yielding practical guidance 
for stabilizing intervention. 
Experiments from real-world data provide evidence consistent with our theoretical analysis (\Cref{sec:experiments}).

% ...



\subsection{Related work}

\noindent\textbf{Preference Optimization (PO).}
A major line of work formulates LLM alignment as an \emph{offline} preference-optimization problem. DPO
~\citep{rafailov2023direct} fits a policy directly from pairwise comparisons under a Bradley--Terry (BT) model~\citep{bradley1952rank}.
More recently, \citet{azar2024general} introduced the $\Psi$-Preference Optimization ($\Psi$PO) framework, which unifies a family of preference optimization objectives including DPO as its special case. In this work, we focus on Identity Preference Optimization (IPO) and DPO, both are representative instances of $\Psi$PO.

\noindent\textbf{Social choice properties of alignment methods. } A growing line of research connects LLM alignment to social choice, asking whether methods such as reinforcement learning from human feedback (RLHF)~\citep{ge2024axioms}, DPO, and Nash learning from human feedback (NLHF)~\citep{liu2025statistical} satisfy basic axioms when viewed as aggregating pairwise preferences. However, these analyses neglect the effect of sampling, typically assuming an arbitrary given dataset~\citep{noothigattu2020axioms}, or (implicitly) a fixed uniform sampling distribution~\citep{xiao2025theoretical}. 
Besides, little is known about IPO. We study how sampling affects the social choice behavior of IPO solution.

\noindent\textbf{Sampling effects in preference optimization.} Some empirical works design preference data sampling process to improve alignment outcomes. For example, \citet{liu2023statistical,dong2023raft} show that rejection-style sampling can improve performance by concentrating comparisons on more informative candidates. \citet{kim2025understanding,pan2025matters} emphasize improving coverage or quality of sampled responses to enhance performance. \citet{feng2025pilaf,shi2024crucial} design samplers to accelerate the training convergence rate. Despite these findings, a systematic understanding on how sampling distributions \emph{shape the optimizer itself}---and how this dependence propagates under iterative deployment---remains limited. Our work addresses this gap by characterizing the sampling-to-solution map in IPO/DPO and analyzing their long-term implications.

\noindent\textbf{Online and iterative preference learning.}
Recent works have studied \emph{online} preference learning in which each round of alignment uses preference data newly sampled from the previously deployed policy and report that such iterative PO methods outperform purely offline training~\citep{dong2024rlhf,tajwar2024preference,ye2024online,guo2024direct,pan2025matters}. There is also study in principled exploration mechanisms for online PO~\citep{xiong2023iterative,calandriello2024human}. Beyond sampling, several works consider dynamic \emph{reference} policies. Trust-Region DPO refreshes the reference model during training to stabilize updates~\citep{gorbatovski2024learn}, \citet{xu2024bpo} report results when setting dynamic reference models as moving averages of learned policies, and \citet{kim2025understanding} propose an online version of DPO that uses previous policies as the reference models. These iterative mechanisms have also been implemented in widely used open-source tooling (e.g., Hugging Face TRL)~\citep{trl_dpotrainer_docs}. We provide a systematic analysis of the iterative alignment dynamics.

\section{Preliminaries}
\label{sec:preliminaries}


We formalize the LLM alignment model considered in the paper. We work with a fix a prompt $s$ and a discrete (possibly large) set of candidate responses $[K]\coloneqq \{1,\ldots,K\}$. 
We work with the following objects.
\begin{itemize}
    \item The \emph{preference matrix}
    is a $P\in[0,1]^{K\times K}$ where $P_{ij}$ is the probability that $i$ is preferred over $j$~\footnote{Given an integer $K$, denote the $K-$dimensional probability simplex by $\Delta_K\coloneqq \{\bmu \in \bbR^K_{\geq 0}: \sum^K_{k=1}\mu_k = 1\}$. Given a set $S$, $S^\circ$ denotes the interior of $S$.}.
    Therefore, $P_{ii}=1/2$ for all $i$ and $P_{ij}+P_{ji}=1$ for all $i\neq j$. We use $P_i$ to represent the $i$-th row of $P$. We do not assume $P_{ij}$ conforms to the BT model (i.e., $P_{ij}=\sigma(r_i-r_j)$, where $\sigma$ is the sigmoid function), and  denote the set of all valid preference matrices as $\calP$. 
    \item The \emph{sampling distribution}
    is a distribution $\bmu\in\Delta_K^\circ$ over responses.
    We assume each comparison pair $(i,j)$ is generated by sampling $i\sim\bmu$ and $j\sim\bmu$ independently, so the probability of $(i,j)$ being sampled is $\mu_i\mu_j$.
    This models
    the common regime where preference data is collected from responses generated by some behavior policy (e.g., simply uniform distribution or a deployed \ac{LLM} model).
    Unless stated otherwise, we assume full support ($\mu_i>0$ for all $i$) to avoid degenerate identifiability issues.
    \item 
    A \emph{learned policy} $\bpi\in\Delta_K^{\circ}$, is parameterized by logits $\btheta\in\bbR^K$ via $\bpi=\softmax(\btheta)$, and is trained 
    with a regularizer toward
    \emph{reference policy} $\bpi_{\mathrm{ref}}\in\Delta_K^{\circ}$ to favor solutions near $\bpi_{\mathrm{ref}}$.
\end{itemize}

Our goal is to understand how the preference structure $P$, the sampling strategy $\bmu$, and the 
reference policy $\bpi_{\mathrm{ref}}$ 
jointly
determine the learned policy, and how this dependence 
evolves
when the learned policy is redeployed and used to generate future preference data or reference models.

\subsection{IPO as a KL-regularized preference objective}
\label{sec:prelim_ipo}

IPO is a special case of the $\Psi$-Preference Optimization ($\Psi$-PO) framework~\citep{azar2024general}, which optimizes a KL-regularized objective whose data term depends on transformed pairwise preference probabilities $\Psi(P_{ij})$.
When $\Psi(q)=q$ is the identity function, the population IPO objective takes the form \footnote{In practice, $P_{ij}$ is unknown, and one observes noisy binary outcomes indicating the preferred response with expectation $P_{ij}$. We analyze the population objective with known $P$ is for revealing the \emph{structural} dependence of the learned policy on $\bmu$ and $\Pref$,making the effects analytically transparent. }
\begin{equation}
\label{eq:ipo_population}
\max_{\bpi\in\Delta_K}
\;
\EE_{i\sim\bpi,\; j\sim\bmu}\!\left[P_{ij}\right]
-
\frac{1}{\beta}\,\mathrm{KL}(\bpi \,\|\, \bpi_{\mathrm{ref}}),
\end{equation}
where the inverse temperature $\beta>0$ controls the strength of the KL regularization.
Treating the learner as a row player with payoff matrix $P$, 
\eqref{eq:ipo_population}
can be viewed as maximizing its expected utility against the column strategy $\bmu$. 
IPO was proposed to mitigate overfitting in approaches that rely on parametric preference models like BT,
including DPO
and 
RLHF
pipelines. It admits a closed-form optimizer as a function of $\bmu$ and 
$\Pref$, facilitating a clean analysis of their individual roles. The following result is a special case of the analytical solution of $\Psi$-PO \citep{azar2024general}. We provide its proof in Appendix \ref{app:IPO_solution} for completeness. We also present an analogous result for DPO in Appendix~\ref{app:dpo_solution}, serving as the starting point for a parallel set of analysis of DPO.

\begin{restatable}{proposition}{propexistuniqueipo}
\label{prop:exist_unique_IPO}
Given preference matrix $P\in\calP$, the IPO problem with sampling strategy $\bmu\in\Delta_K^\circ$ and reference policy $\bpi_{\mathrm{ref}}\in\Delta_K^\circ$
is equivalent to solving the convex optimization problem 
\begin{equation}
\label{eq:IPO_convex_op}
\min_{\bpi\in\Delta_K}\;
\sum_{1\le i,j\le K}
\mu_i\mu_j\,P_{ij}\left(
\log\frac{\pi_i}{\pi_{\mathrm{ref},i}}
-
\log\frac{\pi_j}{\pi_{\mathrm{ref},j}}
-\frac{\beta}{2}\right)^2,
\end{equation}
which admits a unique closed-form solution $\bpi^{\star}$:
\begin{equation}
\label{eq:ipo_solution}
\bpi^\star \propto \bpi_{\mathrm{ref}}\odot\mathrm{softmax}(\beta P\bmu),
\end{equation}
where $\odot$ denotes element-wise multiplication.
\end{restatable}

The explicit mapping $(P,\bmu,\Pref)\mapsto\bpi^\star$ in \eqref{eq:ipo_solution} makes IPO 
well-suited for our goals:
it disentangles the roles of the sampling strategy and the reference policy, enabling sharp comparisons across different samplings and principled analyses of self-reinforcing deployment dynamics. 

%
\section{Sampling Shapes IPO Optima}
\label{sec:offline}

In this section, we consider the one-shot, offline setting and study the properties of the \ac{IPO} optimizer. Since $\Pref$ enters the solution only via an additive shift, we focus on the sampling effect and fix a uniform reference policy, 
$\Pref = \frac{1}{K}\mathbf{1}$ where $\mathbf{1}\in \bbR^K$ is the all-one vector.
%
Our results extend immediately to general reference policies by replacing $\pi_i$ with 
its increments over the reference policy, $\log \pi_i/\pi_{\mathrm{ref},i}$.

\subsection{Axiom of (im)possibility under sampling}
\label{sec:offline-social-choice}

We begin by examining whether \ac{IPO} satisfies certain natural desiderata,
treating it as 
a preference aggregation method that maps 
preference matrix to a 
distribution over responses.
For a sampling distribution $\bmu\in\Delta_K$, 
\eqref{eq:ipo_solution} defines a rule $F_{\bmu}(P)\coloneqq \bpi^\star(P;\bmu)$ induced by \ac{IPO}.
We say that \emph{\ac{IPO} under sampling $\bmu$ satisfies an axiom} if $F_{\bmu}$ satisfies it.

Since \ac{IPO}'s KL term usually induces full-support solutions, we 
state axioms for full-support distributions. The first and most natural set 
formalize the idea that ``better responses  should receive higher probabilities.'' To state these axioms, we first define structural notions for preference matrices.  

\begin{definition}
Let $P\in\calP$ be a  preference matrix. We make the following definitions:
\begin{itemize}
\item 
$P$ admits a \emph{Condorcet winner} if there exists $c\in[K]$ such that
$P_{cj}>\tfrac12$ for all $j\neq c$.
\item 
A nonempty set $S\subseteq[K]$ is \emph{dominant} if 
$P_{ij}>\tfrac12$ for all $i\in S$ and $j\notin S$.
The \emph{Smith set} $\Smith(P)$ is the unique inclusion-minimal dominant set.
\item 
$P$ is \emph{transitive} if there exists an ordering $a_1\succ a_2\succ\cdots\succ a_K$ such that
$P_{a_i a_j}\ge\tfrac12$ for all $i<j$. When all inequalities hold strictly, $P$ is \emph{strictly transitive}.
\end{itemize}
\end{definition}

When $P$ is transitive, we relabel so that $1\succ 2 \succ \dots \succ K$. 
Now we define the corresponding social choice axioms. 


%

\begin{definition}
Let $F: \calP \rightarrow \Delta_K$ be an aggregation rule.
\begin{itemize}
\item 
A rule $F$ is \emph{Condorcet-top}
if for every $P$ that admits a Condorcet winner $c$,
 $F$ ranks $c$ strictly above all others, i.e.,
$F(P)_c > F(P)_j$ for all $j\neq c$.
%
%
\item 
A rule $F$ is \emph{Smith-top} if
every element in $S$ receives strictly larger probability than all elements outside $S$, i.e., $F(P)_i > F(P)_j$ for  for all $i\in S, j\notin S$.
%
\item
A rule $F$ is \emph{order-preserving} if for every transitive $P$ with order
$1\succ 2\succ\cdots\succ K$, the output distribution preserves the order, i.e., $F(P)_{1} \geq \cdots \geq F(P)_{K}$.
%
%
\end{itemize}
\end{definition}


First, \Cref{thm:axiom-violation-fixed-sampling} highlights that 
fixed sampling
can be incompatible with these axioms. 

\begin{restatable}{proposition}{thmaxiomviolationfixedsampling}
\label{thm:axiom-violation-fixed-sampling}
For any sampling distribution $\bmu$, IPO under $\bmu$
is not {\emph{Condorcet-top, Smith-top,}} or {\emph{Order-preserving}}. 
\end{restatable}

In contrast, the next theorem shows that allowing $\bmu$ to depend on the preference instance provides additional leverage and can restore the axioms. 


\begin{restatable}{theorem}{thmpositivecondorcetsmith}
\label{thm:positive-condorcet-smith}
	
	 For any 
		$P$ with a Condorcet winner, there exists a 
		sampling distribution $\bmu$ such that $F_{\bmu}(P)$ assigns $c$ the unique highest probability. 
%
		Moreover, for any 
		$P$, 
there exists a 
sampling distribution $\bmu$ such that 
$F_{\bmu}(P)_i > F_{\bmu}(P)_j$ for any $i \in \Smith(P)$ and $j \notin \Smith(P)$. 
\end{restatable}

Beyond mere existence, the proof of \Cref{thm:positive-condorcet-smith} 
suggests the desiderata are achievable
by ``reasonable'' samplings instead of 
a contrived distribution. 
They place more mass on responses preferred under $P$, which is practically plausible 
since human users and well-trained LLMs tend to produce better responses more often than a uniform sampler. 

One may ask how to choose such instance-dependent sampling when $P$ is unknown. A natural approach is to choose $\bmu$ by estimating $P$ from observed comparisons, e.g., via dueling bandits or preference-learning methods tailored to the application. Developing such procedures is beyond the scope of this paper, and we leave it to future work. 

We also note that \citet{liu2025statistical} study non-full-support analogues of the Condorcet-top and Smith-top, and shows that non-regularized 
NLHF
satisfies them. From a game-theoretic perspective, NLHF can be viewed as a row player maximizing its worst-case expected utility over sampling strategies. 
In contrast, 
our flexibility in choosing the sampling distribution yields a stronger Smith-set guarantee:
a natural non-full-support analogue of Smith-top would require assigning positive probability to \emph{every} element of the Smith set, whereas \citet{liu2025statistical} only guarantee that the support of the learned policy is contained in the Smith set.


By contrast, order-preserving is more delicate: changing 
$\bmu$ can drastically alter the ranking induced by the IPO solution, and some transitive instances cannot be made order-preserving by any full-support sampling distribution.

\begin{restatable}{proposition}{thmfragilityorderpreserving}
\label{thm:fragility-order-preserving}
		 For any $K>2$, there exists a transitive preference matrix $P$, such that for any $i \in [K-1]$, there exists a sampling distribution $\bmu$ for which $F_{\bmu}(P)$ ranks the best response (response 1) in the $i$-th position.
		Moreover, there exists a transitive $P$ such that for every 
		sampling distribution $\bmu \in \Delta^\circ_K$, the output $F_{\bmu}(P)$ is not order-preserving. 

\end{restatable}


Nevertheless, there is a natural subclass of transitive matrices on which IPO is order-preserving for all samplings.

\begin{definition}
\label{def:st}
A transitive preference matrix $P$ with total order $1 \succ 2 \succ \cdots \succ K$  is \emph{strongly transitive} (ST) if for all $i<j<k, P_{ik} \ge \max\{P_{ij},\,P_{jk}\}$.
\end{definition}


ST captures a monotone strengthening of preference along an order. It is satisfied by many latent-variable preference models, including BT and Thurstone \citep{thurstone2017law}. 



\begin{restatable}{theorem}{thmrobuststrongtransitive}
\label{thm:robust-strong-transitive}
	For any sampling distribution $\bmu$, IPO under $\bmu$
    is order-preserving over strongly transitive matrices. 
\end{restatable}


Finally, 
\ac{IPO} satisfies other
desiderata that are stable in $\bmu$ and relevant for practical LLM fine-tuning. We defer the details to \Cref{sec:axiom-hold-by-ipo}.



\subsection{Skewed sampling amplifies skewness}
\label{sec:offline-diversity}

We now study how the sampling distribution shapes the \ac{IPO} solution’s \emph{concentration}, quantified through \emph{pairwise logit gaps}.
We focus on ST preference matrices, for which the induced ranking is invariant to $\bmu$, allowing us to compare concentration without the confound of rank reversals. 
Our main finding is that, under a natural structural condition on $P$, 
making $\bmu$ more aligned with 
the
ranking necessarily \emph{amplifies} the learned logit gaps
i.e., yields a policy that is more skewed toward dominant responses.
To state the result formally, we first recall the definition of majorization.

\begin{definition}\label{def:maj}
A vector $\bm{a}=(a_1,\dots,a_K)\in\mathbb{R}^K$ is \emph{majorized} (Maj) if
\begin{align*}
\frac{1}{m}\sum_{k=1}^m a_k\ge
\frac{1}{K}\sum_{k=1}^K a_k, \quad \forall m\in[K-1].
\end{align*}
Further, we denote by $\bm{a}\succ \b$ if $\bm{a}-\b$ is Maj.
\end{definition}

A vector $\bm{a}$ being majorized means that every prefix average of $\bm{a}$ is at least as large as the global average. 


\begin{restatable}{theorem}{thmpairwisegapipo}
\label{thm:pairwise_gap_ipo} 
Let $P$ be ST with order $1 \succ \cdots \succ K$ and $\bmu,\bmu' \in \Delta_K$ such that $\bmu-\bmu'$ is non-increasing. Define $\btheta(\bmu)=\beta^{-1}(\log\bpi^\star(P;\bmu)-\log\bpi_{\mathrm{ref}})$. Then if some $(i,j)$ satisfy that $P_i\ne P_j$ and $P_i \succ P_j$, we have
\begin{equation}\label{eq:maj_diff}
\theta_i(\bmu)-\theta_j(\bmu)>\theta_i(\bmu')-\theta_j(\bmu')>0.  
\end{equation}
\end{restatable}

\Cref{thm:pairwise_gap_ipo} formalizes a one-shot notion of \emph{self-reinforcement} for IPO under offline sampling. It states that as long as $P_i\succ P_j$, changing from an arbitrary baseline sampling distribution $\bmu'$ to $\bmu$ in a direction \emph{aligned} with the 
preference order induced by $P$ will always strictly enlarge the corresponding logit gap. In other words, whenever $P_i\succ P_j$, sampling that is \emph{more concentrated on higher-ranked responses relative to a baseline} makes the learned IPO policy more separated between actions $(i,j)$, hence more skewed toward dominant responses.

While majorization is directly checkable and can therefore serve as a useful diagnostic criterion, it is not immediate that it should hold for realistic preference structures.
We therefore identify a concrete and practically relevant scenario in \Cref{cor:bt_headtail_maj}:
for BT preferences with a head--tail separation structure, $P_i\succ P_j$ automatically holds for any pair $i<j$ within the head.

\begin{restatable}{corollary}{corbtheadtailmaj}
\label{cor:bt_headtail_maj}
Let $P\in\calP$ be induced by the BT model $P_{ij}=\sigma(\theta_i-\theta_j)$ with $ \theta_1>\theta_2>\cdots>\theta_K$. Given integer $0 <H <K$ and positive $\delta>3\ln2$, $P$ satisfies \emph{Head-Tail Separation (HTS)} with parameters $(H, \delta)$ 
if the index set $[K]$ can be separated by a head set $\H:=\{1,\dots,H\}$ and a (long) tail set $\T:=\{H+1,\dots,K\}$, such that

 
\begin{enumerate}
    \item (tails are uniformly dominated by heads) for any $k\in \H$, $\theta_k\geq 0$, whereas for any $k\in \T$, $\theta_k\le -\delta$.
    \item (heads are moderately separated) for every pair $i<j$ in $\H$, $\theta_i-\theta_j\ge 8\left(\frac{H}{K}+e^{-\delta}\right)$. 
\end{enumerate}

Then for every $i<j$ in $\H$, $P_i\succ P_j$ and therefore \eqref{eq:maj_diff} holds.
\end{restatable}

Corollary~\ref{cor:bt_headtail_maj} captures a practical and commonly encountered preference landscape: the set of candidate responses is enormous ($K$ large), but only a small subset of responses carries meaningful signal. This naturally yields a HTS structure: a small head set $\H$ contains informative responses that compete with one another, while the long tail $\T$ is uniformly dominated in pairwise comparisons. This regime also makes the second requirement a mild one as $\frac{H}{K}+e^{-\delta}=o(1)$ for moderate or large $\delta$. The main message of~\Cref{cor:bt_headtail_maj} is therefore clear: whenever preferences exhibit HTS, \emph{any} sampling shift that increases relative mass on higher-ranked responses relative to a baseline $\bmu'$ will sharpen the separation among head responses in the learned IPO solution. While it improves discrimination, it also systematically reduces diversity by concentrating probability mass on the top of the ranking. 
We emphasize that such an amplification phenomenon is not unique to IPO: an analogous dependence of pairwise separation on the sampling distribution also holds for DPO; see Appendix~\ref{app:offline_dpo}.

\section{Long-Term Effects of Iterative IPO}
\label{sec:self_ref_dynamics}


In this section, we study the long-run dynamics in iterative preference alignment pipelines, which are increasingly common in practice:
candidate responses are generated by the deployed policy and labeled by humans; the resulting data are used to fine-tune the next policy, which is then redeployed to collect the next round of data. Meanwhile, the reference model is often refreshed 
over time to track improvements in the current foundation model. 
Therefore, both the sampling distribution and the reference policy evolve endogenously over rounds and shape subsequent updates. 

 We model such a deployment pipeline as a discrete-time dynamical system over policies. At round $t$, the incumbent policy is $\bpi_t\in\Delta_K$ and the next-round alignment procedure derives the reference policy and sampling distribution from $\bpi_t$.
We capture these dependencies using two mixing parameters $\alpha,\lambda\in[0,1]$.
Specifically, for the reference policy that stabilizes the update, we consider a geometric mixture
\begin{equation}\label{eq:pi_ref_t}
\bpi_{\rm ref}^{(t)} \ \propto\ \bpi_t^{\alpha}\,\bpi_{\rm ref}^{\,1-\alpha},
\end{equation}
interpolating between the current policy $\bpi_t$ and a fixed base reference policy $\bpi_{\rm ref}$, capturing the practice of dynamic reference models in \citet{gorbatovski2024learn}. Similarly, for the sampling distribution that determines which pairs of responses are likely to be compared by annotators, we consider an affine mixture
\begin{equation}\label{eq:mu_t}
\bmu_t \ :=\ \lambda\,\bpi_t + (1-\lambda)\,\bpi_0,
\end{equation}
that interpolates between on-policy sampling from the deployed model and an off-policy distribution $\bpi_0$ (e.g., a frozen earlier model or a fixed data source), capturing a practice introduced in IPO-MD \citep{calandriello2024human}. This abstraction reflects a realistic deployment regime in which training and deployment are interleaved, and where practitioners explicitly mix on-policy and off-policy data to trade off adaptivity and stability.

By \Cref{prop:exist_unique_IPO}, the IPO update induced by these choices admits a closed-form policy mapping.
Substituting the mixed reference and mixed sampling distributions yields the following dynamical system.
\begin{definition}
\label{def:mixed_ipo_dyn}
The \emph{mixed reference/sampling IPO (MRS-IPO) dynamics} 
\begin{align*}
    \bpi_{t+1}=\mathrm{MRS}(\bpi_t;P,\alpha,\beta,\lambda,\bpi_{\mathrm{ref}},\bpi_0),
\end{align*} initialized by $\bpi_1=\bpi_1$ is defined by
\begin{equation}
\label{eq:mix_ref_softmax_form}
\bpi_{t+1}
=
\softmax\!\left(
\log \bpi_{\rm ref}^{(t)}
+\beta\,P\bmu_t
\right), t\geq 1,
\end{equation}
where $\bpi_{\rm ref}^{(t)}$ and $\bmu_t$ are defined in \eqref{eq:pi_ref_t}, \eqref{eq:mu_t}, $\alpha\in[0,1]$ controls how strongly the \emph{reference} model tracks the current deployment,
$\lambda\in[0,1]$ controls how strongly the \emph{sampling} is on-policy,
$\beta>0$ is the IPO inverse temperature,
and $P$ is the underlying pairwise preference matrix.
\end{definition}

The MRS-IPO dynamics \eqref{eq:mix_ref_softmax_form} includes several common iterative alignment practices. $\alpha=0$ indicates a fixed reference policy overtime while $\alpha=1$ corresponds to a fully self-referential deployment, i.e., encouraging newly aligned model to remain close to the previous version. Similarly, $\lambda=1$ corresponds to fully on-policy preference collection from the deployed model, whereas $\lambda<1$ captures the widespread practice of incorporating off-policy data to improve stability and coverage.

Despite the naturalness of the MRS-IPO dynamics, we show they can exhibit fundamental long-term failure modes: non-convergence/oscillations (\Cref{prop:condorcet_instability_mixed}) and policy collapse (\Cref{thm:true_vs_eps_collapse}), highlighting limitations of naive iterative deployment. On the positive side, we identify parameter regimes that ensure convergence (\Cref{thm:mix_contraction_generalP}). Together with the explicit parameter dependence in \Cref{thm:true_vs_eps_collapse} that quantifies the extent of collapse, these results provide guidance and motivation for stabilizing interventions. 

Similarly, we define the MRS-DPO dynamics in \Cref{app:iterative_dpo}, which shares the same structure as \eqref{eq:mix_ref_softmax_form}, but the explicit term $P\bmu_t$ turns out to be replaced with an implicit KKT-defined mapping $\btheta(P,\bmu_t)$. A parallel analysis yields the same qualitative conclusions, but for clarity we only focus on IPO here and defer DPO results to \Cref{app:iterative_dpo}.

\subsection{Condorcet cycles can induce policy oscillation}
\label{sec:dynamics_condorcet_cycles}

We begin with preference structures that violate transitivity.
A canonical failure mode is a \emph{Condorcet cycle} (e.g., rock--paper--scissors), where local pairwise comparisons 
cannot be globally reconciled into a single ranking.
Such a cyclic structure is common in practice whenever preferences depend on nuanced trade-offs among latent-attributes or arise from heterogeneous annotators.

We show that in the presence of a Condorcet cycle, the 
MRS-IPO dynamics can be \emph{unstable} and fail to converge. 
We 
give an explicit construction for $K=3$ and identify a parameter regime in which oscillations occur.

\begin{restatable}{proposition}{propcondorcetinstabilitymixed}
\label{prop:condorcet_instability_mixed}
Let $K=3$ and fix $a\in(0,1/2)$.
Consider the rock--paper--scissors preference matrix
\begin{equation}
\label{eq:rps_P}
P=\begin{pmatrix}
\frac12 & \frac12+a & \frac12-a\\
\frac12-a & \frac12 & \frac12+a\\
\frac12+a & \frac12-a & \frac12
\end{pmatrix},
\end{equation}
so that $1\succ 2$, $2\succ 3$, and $3\succ 1$ with probability $1/2+a$.
Let $\bpi_{\rm ref}=\bpi_0=(1/3,1/3,1/3)$.
Then the MRS-IPO dynamics \eqref{eq:mix_ref_softmax_form} admits a fixed point
$\bpi^\star=(1/3,1/3,1/3)$.
Moreover, for any choice of parameters $\alpha,\lambda\in[0,1]$ and $\beta>0$ satisfying
\begin{equation}\label{eq:instability_region}
\alpha^2+a^2\cdot\frac{\beta^2\lambda^2}{3}>1,
\end{equation}
the fixed point $\bpi^\star$ is \emph{linearly unstable}: there exists an open neighborhood $\mathcal{N}$ of $\bpi^\star$ such that for any initialization $\bpi_1\in\mathcal{N}\setminus\{\bpi^\star\}$, the iterates $\{\bpi_t\}_{t=1}^{+\infty}$ exhibits persistent oscillations around $\bpi^\star$ and does not converge. 
\end{restatable}

Condition~\eqref{eq:instability_region} clarifies how the stability of MRS-IPO depends on the design knobs $(\alpha,\beta,\lambda)$: larger $\alpha$ (more self-referencing), larger $\beta$ (more aggressive IPO updates), larger $\lambda$ (more on-policy sampling), or larger cycle margin $a$ all increase the tendency to oscillate. Among these parameters, $\alpha$ plays a qualitatively distinct role, as fully self-referential (i.e., $\alpha $) dynamics is always unstable. In other words, under cyclic preferences, an iterative PO pipeline that always stays close to the deployed model will generically \emph{chase a rotating preference signal} and fail to settle. When $\alpha<1$, stability can be recovered by weakening the on-policy sampling strength $\lambda$ and less aggressive update. This provides a formal guidance on stabilizing the iterative training of a widely used heuristic, i.e., mixing on- and off-policy data.

The underlying reason that causes instability is that the MRS-IPO update behaves like mirror descent type algorithm and the shift-invariant component of $P$ introduces a rotational field in the dynamics -- and it is widely-known that last-iterate convergence is such settings is not expected in general, trajectories often orbit the equilibrium rather than settling. The proof is deferred to \Cref{app:online}.

However, policy oscillatory is not inevitable as one can identify a parameter regime under which the MRS-IPO map becomes a contraction and thus the dynamics converge globally to a unique interior fixed point. The following theorem holds for \emph{any} preference structure $P\in\calP$.

\begin{restatable}{theorem}{thmmixcontractiongeneralP}
\label{thm:mix_contraction_generalP}
For any $\bpi_{\rm ref},\bpi_0\in\Delta_K^\circ$, consider the MRS-IPO dynamics $\bpi_{t+1}=\mathrm{MRS}(\bpi_t;\alpha,\beta,\lambda)$ defined in \eqref{eq:mix_ref_softmax_form}. If
\begin{equation}
\label{eq:mix_contraction_condition_general}
\alpha \;+\; \frac{\beta\lambda}{2}\,\left\|\widetilde A\right\|_2 \;<\; 1,
\end{equation}
where $\widetilde A:=P-\tfrac12\mathbf{1}\mathbf{1}^\top$, there exists a unique $\bpi^\star\in\Delta_K^\circ$ such that $\bpi^\star=\mathrm{MRS}(\bpi^\star;\alpha,\beta,\lambda)$,
and, for any initialization $\bpi_1\in\Delta_K^\circ$, iterates $\{\bpi_t\}$ converge to $\bpi^\star$. In addition, if $\widetilde A$ has at most $d$ nonzero entries in each row, \eqref{eq:mix_contraction_condition_general} can be replaced by $\alpha \;+\; \frac{\beta\lambda d}{4} \;<\; 1$.

\end{restatable}

\Cref{thm:mix_contraction_generalP} provides a conservative but transparent design rule for long-run stability.
The quantity $\|\widetilde A\|_2$ measures the component of $P$ capable of inducing cycling behavior. Condition \eqref{eq:mix_contraction_condition_general} shows that stability is ensured when $\beta\lambda$ is small relative to this intrinsic ``cyclicity scale,'' and/or when the update is less self-referential (smaller $\alpha$). Operationally, this suggests two stabilizing interventions in preference alignment pipelines: (1) reducing the inverse temperature $\beta$ or decreasing the on-policy sampling fraction $\lambda$, and (2) anchoring more strongly to a fixed reference model by choosing a smaller $\alpha$.

\subsection{Strong transitivity can lead to entropy collapse}
\label{sec:entropy_collapse}

We next present a complementary negative phenomenon, which is even more concerning from a long-term deployment perspective. \Cref{sec:offline-diversity} showed that for strongly transitive preferences, \emph{in a single round}, preference-aligned sampling can already strengthen concentration in the learned IPO solution. 
In \emph{iterative} alignment, since such a preference-aligned sampling can arise endogenously when deploying the current policy, we can show that MRS-IPO amplifies this effect over time: the deployed policy may become increasingly concentrated and eventually collapse to an extreme distribution, even when both the sampling distribution and the reference policy have full support. We formalize this observation and identify a key factor that drives this phenomenon. 
Throughout this section we assume $P$ is \emph{strictly strongly transitive} (SST), meaning that $P$ is ST and for any $i\in[K-1]$, $P_{ik}>P_{i+1,k}$ holds for at least one $k\in[K]$. This assumption strengthens ST by requiring strict dominance at least one columnwise \footnote{SST assumption is made purely for clarity; the corresponding results under the weaker ST assumption convey the same qualitative message but require heavier notation (see \Cref{app:eps_collapse_ipo}).}.  We also define a distribution $\bpi\in\Delta_K$ ($\pi_1\ge\pi_2\ge\cdots\ge\pi_K$) is \emph{$\varepsilon$-collapsed} if $\pi_1 \ge 1-\varepsilon$, meaning nearly all probability mass concentrates on the top-ranked response. The next \Cref{thm:robust-strong-transitive} shows that the long-term behavior of MRS-IPO depends crucially on the factor $\beta\lambda/(1-\alpha)$:


\begin{restatable}{theorem}{thmtruevsepscollapse}
\label{thm:true_vs_eps_collapse}
Assume $P$ is SST.
Let $\bpi_0\in\Delta_K^\circ$ and $\bpi_{\rm ref}\in\Delta_K^\circ$.
Consider the MRS-IPO dynamics $\bpi_{t+1}=\mathrm{MRS}(\bpi_t;\alpha,\beta,\lambda)$. These dynamics exhibit the following two forms of collapse.

\begin{itemize}

\item \textbf{($\varepsilon$-collapse when $\beta\lambda/(1-\alpha)$ is large).}
If $\alpha,\lambda\in[0,1)$, there exists a finite $T$ such that for all $t\ge T$,
\[
\bpi_{t,1}\ \ge\ 1 - \left(\exp\!\left(\frac{\beta\lambda\delta}{2(1-\alpha)}\right)-1\right)^{-1},
\]
where $
\delta
\coloneqq
\min_{i\in[K-1]}
\{(1-\lambda)
\sum_{j=1}^K \pi_{0,j}\,(P_{ij}-P_{i+1,j})
\}$ is a positive constant.
\item 
If $\alpha=1$, there exists a finite $T$ such that for all $t\ge T$,
\[
H(\bpi_{t+1})<H(\bpi_t),
\]
where $H(\bpi) \coloneqq-\sum_{i=1}^K \pi_i\log\pi_i$ is entropy. In addition, $\bpi_t \to \e_1$ and $H(\bpi_t)\to 0$ as $t\to+\infty$.
\end{itemize}
\end{restatable}

\Cref{thm:true_vs_eps_collapse} identifies two limit regimes for MRS-IPO. When the reference policy is fully self-referential ($\alpha=1$), preference advantages accumulate additively across iterations, leading to monotone entropy decay and winner-take-all collapse. When $\alpha<1$, anchoring prevents exact policy collapse to a vertex, but it does \emph{not} remove the tendency toward concentration: as $\beta\lambda/(1-\alpha)$ increases, the limiting policy becomes arbitrarily close to the extreme distribution that concentrates on the top response.

From a practical standpoint, this result highlights a common failure mode in iterative alignment pipelines.
Heavier reliance on the current model as a reference (larger $\alpha$ closer to $1$), more aggressive updates (larger $\beta$), or more on-policy sampling (larger $\lambda$) all amplify self-reinforcement and push the system toward concentration. A further notable aspect of \Cref{thm:true_vs_eps_collapse} is that the collapse mechanism depends primarily on the update hyperparameters $(\alpha,\beta,\lambda)$ and the preference structure, and is largely insensitive to the specific choice of reference policy $\bpi_{\rm ref}$ or baseline sampling distribution $\bpi_0$.
Consequently, entropy collapse should be viewed as an \emph{inherent property of the self-reinforcing dynamics} in strongly transitive environments: choosing a more diverse reference model or mixing with a diverse baseline distribution cannot fundamentally prevent long-run concentration, because the influence of any static external source is geometrically diluted over successive iterations.

\section{Experiments}
\label{sec:experiments}

We present experiments studying the MRS-IPO dynamics on real-world preference data to support our theoretical findings. More details and additional experiments on MRS-DPO dynamics (\Cref{app:iterative_dpo}) 
are in \Cref{app:experiments,app:dpo_experiments}.

\begin{figure*}[t]
    \centering
    \begin{minipage}{0.98\textwidth}
        \centering
        \includegraphics[width=\linewidth]{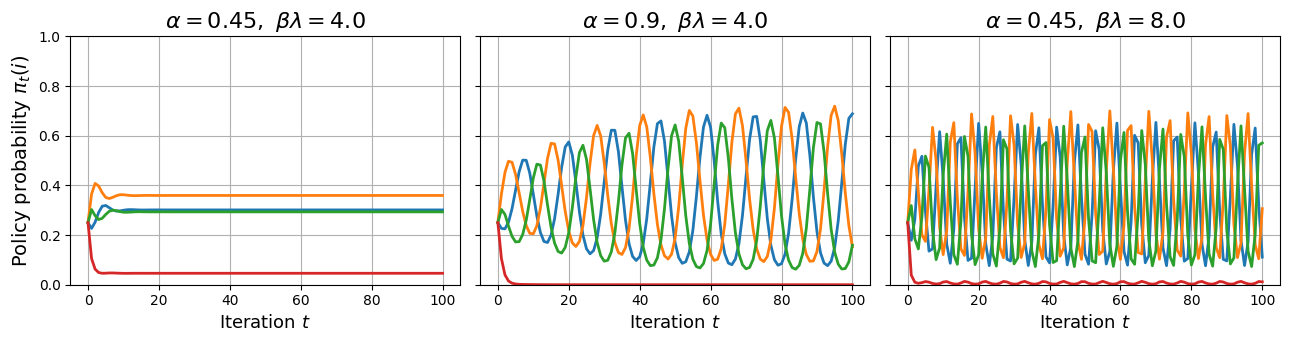} 
        \caption{Policy evolution of MRS-IPO on $P$ with a cyclic structure in the first $100$ iterations. Compared to the baseline (left), increasing either $\alpha$ (center) or $\beta\lambda$ (right) induces oscillations. }
        \label{fig:cyclic_oscillations}
    \end{minipage}
    \hfill
    \begin{minipage}{0.98\textwidth}
        \centering
        \includegraphics[width=\linewidth]{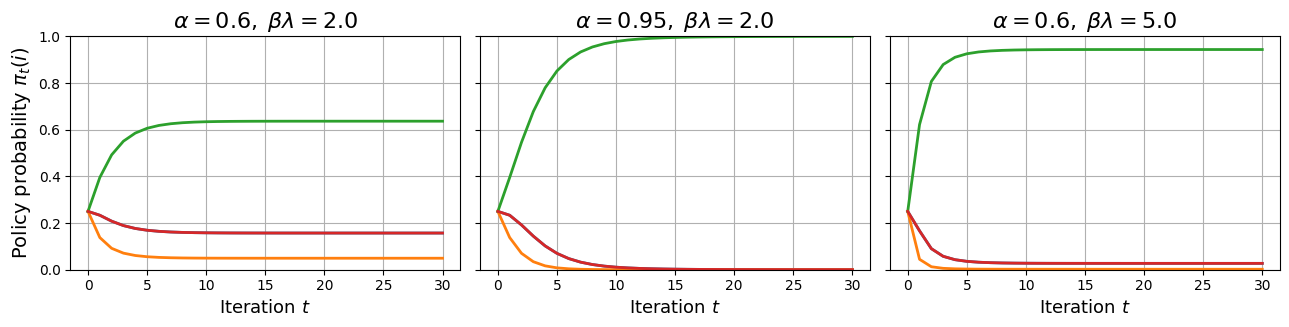} 
        \caption{Policy evolution of MRS-IPO on a ST preference matrix $P$ in the first 30 iterations. Compared to the baseline (left), increasing either $\alpha$ (center) or $\beta\lambda$ (right) induces more extreme policies. }
        \label{fig:transitive_collapse}
    \end{minipage}
\end{figure*}

\begin{figure*}[t]
    \centering
    \begin{minipage}[t]{0.48\linewidth}
        \centering
        \includegraphics[width=\linewidth]{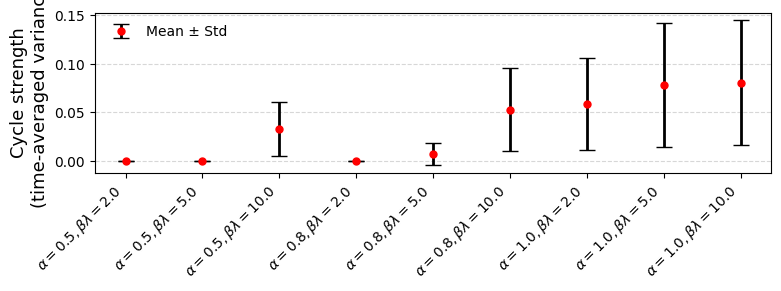}
        \caption{
        Mean $\pm$ standard deviation of time-averaged variance of $\{\pi_t\}^T_{t=1}$ across all cyclic matrices as $(\alpha,\beta\lambda)$ varies.}
        \label{fig:cycle_population}
    \end{minipage}
    \hfill
    \begin{minipage}[t]{0.48\linewidth}
        \centering
        \includegraphics[width=\linewidth]{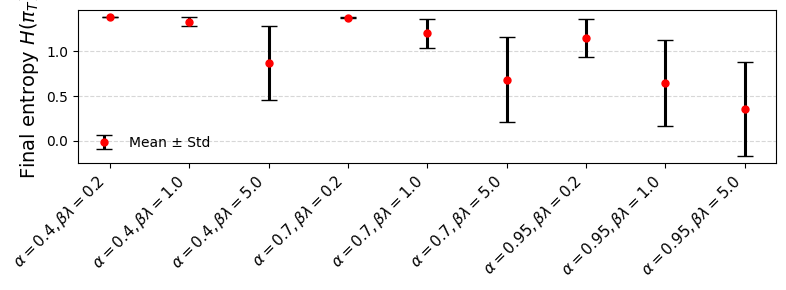}
        \caption{Mean $\pm$ standard deviation of the final policy entropy $H(\bpi_T)$ across all ST  matrices $(\alpha,\beta\lambda)$ varies.}
        
        \label{fig:entropy_population}
    \end{minipage}
\end{figure*}


\subsection{Experimental setup and system dynamics}

We use the \textsc{NVIDIA HelpSteer} dataset \citep{wang2024helpsteer}, which consists of human preference feedback for LLM alignment.
The dataset contains $35{,}331$ response-level records, each 
including a prompt paired with a model-generated response that is independently scored along multiple qualitative attribute (see examples in ~\ref{sample}).
Since most prompts in the dataset have at least four responses, we generate preference matrices $P$ with $K=4$. 

Specifically, for each unordered response pair $(i,j)$, we randomly select a
evaluation attribute as the comparison criterion and map the two responses’
scores $(s_i, s_j)$ w.r.t that attribute  to the preference probability via 
$P_{ij}=\sigma(s_i-s_j)$.
Applying this procedure independently to all response pairs generates a preference matrix $P$ that may contain cycles. 
We defer full construction details 
to \Cref{app:construction}.
%
In accordance with theoretical results in \Cref{sec:self_ref_dynamics}, we classify
the resulting matrices into cyclic instances, i.e., non-transitive matrices,
and ST instances. 
After filtering, the dataset contains 
$118$ cyclic instances and $4{,}924$ ST instances, which form the basis of all experiments reported below.


We run the MRS-IPO dynamics 
\eqref{eq:mix_ref_softmax_form} on these
instances.
Since $\beta$ and $\lambda$ appear multiplicatively in both the theory and the
update rule, we report results as functions of $(\alpha, \beta\lambda)$.
MRS-IPO are initialized with uniform distribution ($\bpi_{\rm ref}=\bpi_1=\frac
1K \bm 1$) and run for $T=3000$ iterations.
%
Robustness to finite sample variability is evaluated separately by explicitly
introducing preference noise, with the corresponding analyses deferred to
Appendix~\ref{app:protocol}.

\subsection{Illustrative examples on parameter sensitivity}
\label{sec:exp-illustrative}


Due to the limited space, we choose one cyclic matrix and one ST matrix and set three sets of parameters for each example, for which only one parameter differs pairwise, to show how the dynamics vary when one parameter changes unilaterally. The average performance is shown in \Cref{fig:cyclic_oscillations} for the cyclic matrix and \Cref{fig:transitive_collapse} for the ST matrix. We defer specific forms of the matrices and more experiments to \Cref{app:form-illustrative-examples}.  In \Cref{fig:cyclic_oscillations}, to show that the dynamics under $(0.9,0.4)$ and $(0.45, 8.0)$ indeed oscillate under the cyclic matrix and do not converge, we present the first 100 runs for better clarity; for the transitive matrix, we show only the first 15 runs, as the convergence is fast. 

Both figures suggest smaller $\alpha$ and $\beta\lambda$, i.e., heavier anchoring to fixed sampling and references, stabilize the dynamics. 
In the cyclic case, \Cref{fig:cyclic_oscillations} shows 
convergence when both parameters are small, but persistent oscillations when either increases, matching  
\Cref{prop:condorcet_instability_mixed} and \Cref{thm:mix_contraction_generalP}. 
In the ST case, \Cref{fig:transitive_collapse} shows 
that all three runs
converge. 
When $\alpha = 0.95$, the dynamics 
collapse to the extreme policy, placing all mass on the best response. 
The limiting policy shifts more mass on response $1$ as the parameters increase.

\subsection{Aggregated results}
\label{sec:aggregated-results}

We show more in-depth analyses by reporting aggregated results on the two constructed classes of preference matrices, i.e., cyclic and ST instances. To quantify oscillation/convergence under cyclic preferences, we use the time-averaged coordinate variance of the policy sequences $\{\bpi_{t}\}^T_{t=1}$ to measure the cycling strength, as shown in \Cref{fig:cycle_population}. Larger variances indicate larger oscillation amplitude.
While \Cref{prop:condorcet_instability_mixed} proves oscillations only for an explicit class of cyclic matrices,
\Cref{fig:cycle_population} suggests that oscillatory behavior 
arises broadly in preference matrices generated from real human feedback.  Moreover, the mean of cycle strength increases with either $\alpha$ or $\beta\lambda$, 
supporting the 
message that 
stronger self-reinforcement
amplifies cyclic instability.

For large parameter pairs, the standard deviations are substantial, consistent with the theory's instance-dependent nature: \Cref{prop:condorcet_instability_mixed} guarantees instability for some matrices but does not imply that every cyclic matrix oscillates under the same parameters. Besides, oscillation strength also depends on the matrix structure,
so some instances may require larger amplification for cycling to be observable. 

To quantify concentration under ST preferences, we use Shannon entropy $H(\bpi_{T})$ of the terminal policy  in \Cref{fig:entropy_population}, where lower entropy 
indicates
a more concentrated and extreme policy. Again, increasing either $\beta\lambda$ or $\alpha$ drives entropy sharply to (near) zero, indicating rapid collapse to near-deterministic policies, consistent with our collapse guarantees. The mean entropy varies systematically with $(\alpha, \beta \lambda)$, reflecting how these parameters control collapse strength, suggesting 
this collapse effect is widespread rather than driven by a few outliers.

\vspace{-1mm}
\section{Conclusions}

We have presented a theoretical investigation of how sampling and reference choices govern offline and iterative preference alignment, using IPO and DPO as case studies. Our results show that careless choices of these factors can yield compromised ranking guarantees, policy oscillations, or entropy collapse. We then identify parameter regimes that recover stability and avoid collapse in iterative preference learning. Our findings suggest that the sampling and dynamical effects are likely to extend across various preference-alignment methods beyond IPO/DPO.

Our findings point to a key open problem: how to design alignment pipelines with principled, adaptive control over sampling and reference updates---otherwise, instability is a predictable outcome of feedback-driven alignment. Addressing this challenge requires treating preference alignment as a stochastic online learning system rather than a static offline optimization problem. We view this shift in perspective as essential for building reliable alignment methods at scale.

\newpage
\appendix

\section{Proof of \Cref{prop:exist_unique_IPO}}\label{app:IPO_solution}

\propexistuniqueipo*

\begin{proof}
Let
\[
\theta_i \;\coloneqq\; \frac{1}{\beta}\log\frac{\pi_i}{\pi_{{\rm ref},i}}
\quad\Longleftrightarrow\quad
\pi_i\;=\;\frac{\pi_{{\rm ref},i}e^{\beta\theta_i}}{\sum_{\ell=1}^K \pi_{\rm ref}(\ell)e^{\beta\theta_\ell}}.
\]
Substituting $\log\frac{\pi_i}{\pi_{{\rm ref},i}}=\beta\theta_i$ into \eqref{eq:IPO_convex_op} and dropping the positive constant factor $\beta^2$, the IPO objective becomes
\begin{equation}\label{eq:ipo_theta_obj}
F(\btheta)
\;\coloneqq\;
\sum_{i,j=1}^K \mu_i\mu_j P_{ij}\Big( (\theta_i-\theta_j)-\tfrac12 \Big)^2,
\end{equation}
where $\btheta\in\mathbb R^K$ is unconstrained. Note that $F(\btheta)$ is invariant under constant shifts:
$F(\btheta+c\bm 1)=F(\btheta)$ for any $c\in\mathbb R$, since it depends only on differences $\theta_i-\theta_j$.
Meanwhile, the induced policy $\bpi(\cdot)\propto \bpi_{\rm ref}(\cdot)e^{\beta\theta(\cdot)}$ is also invariant to such shifts, so $\btheta$ is identifiable only up to an additive constant, which is immaterial for $\bpi$.

To show the convexity and strict convexity on the simplex, let's define the nonnegative weights
\[
w_{ij}\;\coloneqq\;\mu_i\mu_j P_{ij},
\qquad
W\;\coloneqq\;[w_{ij}]=\mathrm{Diag}(\bmu)\,P\,\mathrm{Diag}(\bmu).
\]
A direct expansion of \eqref{eq:ipo_theta_obj} shows that $F(\btheta)$ is a quadratic function in $\theta$, hence convex.
Moreover, its Hessian equals
\[
\nabla^2 F(\btheta) \;=\; 2\,L(\bmu),
\qquad
L(\bmu)\;\coloneqq\; D_{\rm out}+D_{\rm in}-(W+W^\top),
\]
where $D_{\rm out}=\mathrm{Diag}(W\bm 1)$ and $D_{\rm in}=\mathrm{Diag}(W^\top\bm 1)$.
The matrix $L(\bmu)$ is a (symmetrized) graph Laplacian with nonnegative edge-weights, hence it is positive semidefinite and satisfies
$L(\bmu)\bm 1=\bm 0$.
If $\bmu$ has full support and the comparison graph induced by $W+W^\top$ is connected, then
$\mathrm{null}(L(\bmu))=\mathrm{span}\{\bm 1\}$, so $F$ is \emph{strictly} convex on the quotient space
$\mathbb R^K/\mathrm{span}\{\bm 1\}$ (equivalently, strictly convex after fixing any gauge such as $\theta_K=0$ or $\bmu^\top\btheta=0$).
Since the mapping $\btheta \mapsto \bpi(\btheta)$ above is smooth and one-to-one modulo constant shifts, the induced objective in $\bpi$ is strictly convex on the affine simplex $\Delta_K$, and therefore admits a unique minimizer $\bpi^\star$. Next, we derive the first-order condition in $\btheta$. Differentiate \eqref{eq:ipo_theta_obj}. For each $k\in[K]$,
\begin{align*}
\frac{\partial F}{\partial \theta_k}
&=
2\sum_{j=1}^K w_{kj}\Big(\theta_k-\theta_j-\tfrac12\Big)
\;-\;
2\sum_{i=1}^K w_{ik}\Big(\theta_i-\theta_k-\tfrac12\Big).
\end{align*}
Setting the gradient to zero yields the linear system
\begin{equation}\label{eq:ipo_foc_general_theta}
L(\bmu)\,\btheta \;=\; \frac12\,(W-W^\top)\bm 1.
\end{equation}

Now we can derive the closed form of $L$ and $\btheta$ under the standard dueling normalization. Since $P\in\calP$, we have $
P_{ij}+P_{ji}=1\ \ (i\neq j),\qquad P_{ii}=\tfrac12$ and thus $P+P^\top=\bm 1\bm 1^\top$.
Then
\[
W+W^\top=\mathrm{Diag}(\bmu)(P+P^\top)\mathrm{Diag}(\bmu)=\mathrm{Diag}(\bmu)\bm 1\bm 1^\top\mathrm{Diag}(\bmu)=\bmu\bmu^\top.
\]
Also,
\[
W\bm 1=\mathrm{Diag}(\bmu)P\bmu,\qquad W^\top\bm 1=\mathrm{Diag}(\bmu)P^\top\bmu
=\mathrm{Diag}(\bmu)(\bm 1-P\bmu),
\]
where we used $P^\top\bmu=(\bm 1\bm 1^\top-P)\bmu=\bm 1-P\bmu$ (since $\bm 1^\top\bmu=1$).
Therefore,
\[
D_{\rm out}+D_{\rm in}=\mathrm{Diag}(W\bm 1)+\mathrm{Diag}(W^\top\bm 1)=\mathrm{Diag}(\bmu),
\qquad
L(\bmu)=\mathrm{Diag}(\bmu)-\bmu\bmu^\top.
\]
Moreover,
\[
\frac12(W-W^\top)\bm 1
=\frac12\mathrm{Diag}(\bmu)\big(P\bmu-P^\top\bmu\big)
=\mathrm{Diag}(\bmu)\Big(P\bmu-\tfrac12\bm 1\Big).
\]
Substituting into \eqref{eq:ipo_foc_general_theta} gives
\begin{equation}\label{eq:ipo_foc_simplified_theta}
\big(\mathrm{Diag}(\bmu)-\bmu\bmu^\top\big)\btheta
\;=\;
\mathrm{Diag}(\bmu)\Big(P\bmu-\tfrac12\bm 1\Big).
\end{equation}
If $\bmu$ has full support, left-multiply by $\mathrm{Diag}(\bmu)^{-1}$ to obtain
\begin{equation}\label{eq:ipo_foc_projector}
(I-\bm 1\bmu^\top)\btheta \;=\; P\bmu-\tfrac12\bm 1.
\end{equation}
Note that $\bm 1\bmu^\top\btheta=(\bmu^\top\btheta)\cdot \bm 1$, \eqref{eq:ipo_foc_projector} is equivalent to
\[
\btheta^\star \;=\; P\bmu-(\tfrac12-\bmu^\top\btheta^{\star})\bm 1,
\]
and since $\mathrm{softmax}(\cdot)$ is invariant to adding a constant vector $c\bm 1$, this is equivalent (for the induced policy) to choosing $\btheta^\star=P\bmu$. Using $\pi_i\propto \pi_{{\rm ref},i}e^{\beta\theta_i}$ and $\btheta^\star\equiv P\mu$ modulo constants, we obtain
\[
\pi^\star_i
\;\propto\;
\pi_{{\rm ref},i}\exp\!\big(\beta(P\bmu)_i\big),
\]
and normalizing over $i\in[K]$ yields exactly
\[
\bpi^\star
=
\bpi_{\rm ref}\odot \mathrm{softmax}(\beta P\bmu),
\]
which is \eqref{eq:ipo_solution}. This completes the proof.
\end{proof}

\section{Omitted Proofs in \Cref{sec:offline}}

\subsection{Proof of \Cref{thm:axiom-violation-fixed-sampling}}

\thmaxiomviolationfixedsampling*

\begin{proof}
	Notice that when a Condorcet winner exists, the corresponding singleton is the Smith set of the matrix, and when the matrix is transitive, the Condorcet winner ranks first in the order and, therefore, should also ranked top in the solution to preserve the order. Therefore, to show that IPO under uniform sampling does not satisfy the three axioms, it is enough to propose one transitive counterexample for Condorcet top.

For any $\bmu\in\Delta_K^\circ$, consider the following preference matrix with $K=3$, 
\begin{equation*}
P = 
\begin{bmatrix}
\frac{1}{2} & \frac{1}{2} + \frac{\mu_3}{4}  & \frac{1}{2} + \frac{\mu_3}{4}  \\
\frac{1}{2}-\frac{\mu_3}{4}  & \frac{1}{2} & \frac{1}{2} + \frac{2\mu_1+2\mu_2+\mu_3}{4} \\
\frac{1}{2}-\frac{\mu_3}{4} & \frac{1}{2} - \frac{2\mu_1+2\mu_2+\mu_3}{4} & \frac{1}{2}\\
\end{bmatrix}.
\end{equation*}
Since $\sum^3_{k=1}\mu_k = 1$, we know that $P\in [0,1]^{K\times K}$ is a valid preference matrix. Moreover, $1$ is the Condorcet Winner.

By the closed-form of IPO solution \eqref{eq:ipo_solution}, $\pi^\star_1 > \max_{j \neq 1} \pi^\star_{j}$ is equivalent to $\sum^K_{k=1} \mu_k P_{1k} > \max_{j \neq 1} \sum^K_{k=1} \mu_k P_{jk}$. Notice that since
\begin{align}
\sum^3_{k=1}\mu_k P_{2k} - \sum^3_{k=1}\mu_kP_{1k} &=\left(\frac{1}{2} -\frac{\mu_1\mu_3}{4}+ \frac{2\mu_1\mu_3+2\mu_2\mu_3+\mu^2_3}{4}\right) - \left( \frac{1}{2} + \frac{\mu_2\mu_3}{4}+\frac{\mu^2_3}{4} \right)\\
& = \frac{\mu_1\mu_3 + \mu_2\mu_3}{4} > 0
\end{align}
Therefore, the probability of $1$ in the solution is strictly less than that of $2$. This completes the proof. 
\end{proof}

\subsection{Proof of \Cref{thm:positive-condorcet-smith}}

\thmpositivecondorcetsmith*

\begin{proof} We first prove the statement 1. 
Let $c$ denote the Condorcet winner, then $P_{cj}> \frac{1}{2}$ and $P_{jc}<\frac{1}{2}$ for all $j \ne c$. We first consider a non-full-support sampling $\bmu'$ that only samples $c$. Then for any $j \ne c$,
\begin{align*}
	\sum^K_{k=1}P_{ck}\mu'_k = P_{cc} = \frac{1}{2} > P_{jc} = \sum^K_{k=1}P_{jk}\mu'_k~,
\end{align*}
which implies that $\pi^\star_{c}>\pi^\star_{j}$ for any $j\ne c$. One can then mix it with a uniform sampling to obtain a full-support sampling $\bmu = \eps \bmu_0 + (1-\eps)\bmu'$, where $\bmu_0 = \frac{1}{K}\mathbf{1}$ is the uniform sampling distribution and $\mathbf{1}$ is the all-one vector. Since the inequality is strict, it will preserve for a sufficiently small $\eps$. 

Now we prove the statement 2. 
W.L.O.G., assume $S=\Smith(P)=\{1,\dots, L\}$ for $L\leq K$ and let $N = [K]\setminus S$. Then for any $i < L$ and $j > L$, $P_{ij} > \frac{1}{2}$ and $P_{ji}<\frac{1}{2}$. Denote the submatrix with rows and columns restricted to $S$ to be $P_{SS}$, that with rows to $S$ and columns to $N$ to be $P_{SN}$. $P_{NS}$ and $P_{NN}$ are defined similarly. 

Again, we first consider non-full support sampling distributions that put all mass over responses in $S$. Consider the constant-sum games with payoff matrix $P_{SS}$ where the column player is the maximizer. Then it is symmetric with game value $\frac{1}{2}$. By the minimax theorem of zero-sum games (equivalently, constant-sum games), there always exist a mixed strategy $\bmu|_{S}\in \Delta(S)$ such that $P_{SS}\bmu|_{S} \leq \frac{1}{2}$ pointwise. However, since $P_{ij}<\frac{1}{2}$ for all $i >L $, $j < L$, $P_{NS}\bmu|_{S} > \frac{1}{2}$. Therefore, consider the sampling distribution $\bmu$ that equals $\bmu|_{S}$ at $S$. For any $i \in S$ and $j \in N$, we have $\pi^\star_i >\pi^\star_j$. One can then mix it with a uniform sampling to obtain a full-support sampling $\bmu = \varepsilon \bmu_0 + (1-\eps)\bmu'$. Since the inequality is strict, $\pi^\star_i > \pi^\star_j$ will still hold for $i \in S$ and $j \in N$ for sufficiently small $\eps$. 

\end{proof}

\subsection{Proof of \Cref{thm:fragility-order-preserving}}

\thmfragilityorderpreserving*

\begin{proof}
	We first prove statement 1. For a small positive $\eps>0$, consider the following matrix $P$ such that $P_{1,i} = \frac{1}{2}+\varepsilon$, for $ i \in [K]\setminus \{1\}$, and $P_{ij}=1-\eps$ for $i \in [K]\setminus\{1\}$ and $j > i$. Below is an example with $K=4$.
\begin{equation*}
P = 
\begin{bmatrix}
\frac{1}{2} & \frac{1}{2} + \eps  & \frac{1}{2} + \eps & \frac{1}{2} +\eps \\
\frac{1}{2}-\eps  & \frac{1}{2} & 1-\eps & 1-\eps  \\
\frac{1}{2}-\eps & \eps & \frac{1}{2}  & 1-\eps \\
\frac{1}{2}-\eps & \eps & \eps & \frac{1}{2}
\end{bmatrix}.
\end{equation*}

	Then fore sufficiently small $\eps$, $P_{j1}$ will be the $j-1$th highest. Therefore, a sampling distribution $\bmu$ that almost all the mass on $j$ will make response $1$ ranked the $i$th position.

	Now we prove statement 2. Consider the following matrix with $K=4$. 
\begin{equation*}
P = 
\begin{bmatrix}
\frac{1}{2} & \frac{3}{4}  & \frac{5}{8} & \frac{5}{8} \\
\frac{1}{4}  & \frac{1}{2} & \frac{7}{8} & \frac{7}{8}  \\
\frac{3}{8} & \frac{1}{8} & \frac{1}{2}  & \frac{5}{8} \\
\frac{3}{8} & \frac{1}{8} & \frac{3}{8} & \frac{1}{2}
\end{bmatrix}.
\end{equation*}
Suppose there exists a sampling distribution $\bmu$ such that $(P\bmu)_1 > (P\bmu)_2 > (P\bmu)_3 > (P\bmu)_4$, then 
\begin{align*}
	(P\bmu)_1 - (P\bmu)_2 = (P_{12} - \frac{1}{2})(\mu_1+\mu_2) - (P_{23}-P_{13})\mu_3 - (P_{24}-P_{14}) \mu_4 >0,	
\end{align*}
which is equivalent to 
\begin{align}
\label{eqn:transitivity-1}
	\mu_1 + \mu_2 > \mu_3 + \mu_4.	
\end{align}
Moreover,
\begin{align*}
	(P\bmu)_3 - (P\bmu)_4 = (P_{31} - P_{41})\mu_1 + (P_{32}-P_{42})\mu_2 +(P_{34}-\frac{1}{2})(\mu_3 + \mu_4) > 0,
\end{align*}
which is equivalent to
\begin{align}
\label{eqn:transitivity-2}
	\mu_3 + \mu_4 > 2(\mu_3 + \mu_4).
\end{align}
As $\bmu>\mathbf{0}$, \eqref{eqn:transitivity-1} and \eqref{eqn:transitivity-2} cannot hold simultaneously, leading to a contradiction. 
	
\end{proof}

\subsection{Proof of \Cref{thm:robust-strong-transitive}}

\thmrobuststrongtransitive*

\begin{proof}
For any sampling distribution $\bmu$, and any $i<j$, 
\begin{align*}
	\sum^K_{k=1}P_{ik}\mu_k \geq \sum^K_{k=1}P_{jk}\mu_k
\end{align*}
holds by the definition of strongly transitivity. 
\end{proof}

\subsection{Social choice axioms that IPO satisfies under arbitrary sampling}
\label{sec:axiom-hold-by-ipo}
Given two preference matrices $P^{(1)},P^{(2)}$,   define the pooled matrix with mixture weight $\lambda\in(0,1)$
\begin{align*}
P^{(\lambda)} \coloneqq \lambda P^{(1)} + (1-\lambda)P^{(2)}.
\end{align*}

Similarly, for two sampling distributions $\bmu^{(1)},\bmu^{(2)}\in\Delta_K$, define the pooled sampling distribution with mixture weight $\lambda \in (0,1)$
\begin{align*}
\bmu^{(\lambda)} \coloneqq \lambda \bmu^{(1)} + (1-\lambda)\bmu^{(2)}.
\end{align*}

\begin{definition}
An aggregation rule $F: \calP \rightarrow \Delta_K$ is said to satisfy the following axioms if 
\begin{enumerate}
\item \textbf{Pooling consistency.}
\begin{enumerate}
\item[(a)] (\emph{Preference pooling}) 
For two
matrices $P^{(1)}$, $P^{(2)}$, 
$F(P^{(1)})=F(P^{(2)})$ implies $F(P^{(\lambda)})=F(P^{(1)})$ for all $\lambda\in(0,1)$.

\item[(b)] (\emph{Data pooling}, IPO-level) 
For any $P$ and samplings $\bmu^{(1)}, \bmu^{(2)}$,
$F_{\bmu^{(1)}}(P)=F_{\bmu^{(2)}}(P)$ implies $F_{\bmu^{(\lambda)}}(P)=F_{\bmu^{(1)}}(P)$ for all $\lambda\in(0,1)$.
\end{enumerate}

\item \textbf{Pairwise monotonicity.}
If $P$ and $P'$ differ only at $(i,j)$ with $P_{ij}>P'_{ij}$, then
\begin{align*}
\frac{F(P)_i}{F(P)_j}\ \ge\ \frac{F(P')_i}{F(P')_j}.
\end{align*}
\end{enumerate}
\end{definition}

Intuitively, pooling consistency captures the idea that aggregating indistinguishable sources of data should not change the outcome: if two preference matrices (or sampling distributions) lead to the same \ac{IPO} solution, then any mixture of them should do so as well. Pairwise monotonicity formalizes a basic responsiveness requirement: strengthening the preference of $i$ over $j$, while leaving all else unchanged, should not decrease the relative weight assigned to $i$ over $j$.

\begin{theorem}
For any sampling distribution $\bmu$, the induced rule $F_{\bmu}$ satisfies preference-pooling consistency and
pairwise monotonicity; and IPO satisfies data-pooling consistency.
\end{theorem}

\begin{proof}
We first prove the preference pooling consistency.  $F_{\bmu}(P^{(1)}) = F_{\bmu}(P^{(2)})$ implies that for any $i,j$,
\begin{align*}
\frac{\sum^K_{k=1}P^{(1)}_{ik}\mu_k}{\sum^K_{k=1}P^{(1)}_{jk}\mu_k}	 = \frac{\sum^K_{k=1}P^{(2)}_{ik}\mu_k}{\sum^K_{k=1}P^{(2)}_{jk}\mu_k}~,
\end{align*}
and therefore
\begin{align*}
	\frac{\sum^K_{k=1}(\lambda P^{(1)}_{ik} + (1-\lambda)P^{(2)}_{ik})\mu_k}{\sum^K_{k=1}(\lambda P^{(1)}_{jk} + (1-\lambda)P^{(2)}_{jk})\mu_k}	 = \frac{\lambda \sum^K_{k=1}P^{(1)}_{ik}\mu_k + (1-\lambda)\sum^K_{k=1}P^{(2)}_{ik}\mu_k}{\lambda \sum^K_{k=1}P^{(1)}_{jk}\mu_k + (1-\lambda)\sum^K_{k=1}P^{(2)}_{jk}\mu_k} = \frac{\sum^K_{k=1}P^{(1)}_{ik}\mu_k}{\sum^K_{k=1}P^{(1)}_{jk}\mu_k}~.
\end{align*}
The data pooling consistency can be proved similarly. Specifically, $F_{\bmu^{(1)}}(P) = F_{\bmu^{(2)}}(P)$ implies that for any $i,j$,
\begin{align*}
\frac{\sum^K_{k=1}P_{ik}\mu^{(1)}_k}{\sum^K_{k=1}P_{jk}\mu^{(1)}_k}	 = \frac{\sum^K_{k=1}P_{ik}\mu^{(2)}_k}{\sum^K_{k=1}P^{(2)}_{jk}\mu^{(2)}_k}~,
\end{align*}
and therefore,
\begin{align*}
\frac{\sum^K_{k=1}P_{ik}(\lambda \mu^{(1)}_{k} + (1-\lambda)\mu^{(2)}_{k})}{\sum^K_{k=1}P_{jk}(\lambda \mu^{(1)}_{k} + (1-\lambda)\mu^{(2)}_{k})}	 = \frac{\lambda \sum^K_{k=1}P_{ik}\mu^{(1)}_k + (1-\lambda)\sum^K_{k=1}P_{ik}\mu^{(2)}_k}{\lambda \sum^K_{k=1}P_{jk}\mu^{(1)}_k + (1-\lambda)\sum^K_{k=1}P_{jk}\mu^{(2)}_k} = \frac{\sum^K_{k=1}P_{ik}\mu^{(1)}_k}{\sum^K_{k=1}P_{jk}\mu^{(1)}_k}~.
\end{align*}

As for Pairwise monotonicity, as $P$ and $P'$ differ only at $(i,j)$ with $P_{ij} > P'_{ij}$, then for any sampling distribution $\bmu$, 
\begin{align*}
\log\left(\frac{F_{\bmu}(P)_i}{F_{\bmu}(P)_j} \right) &= \sum^{K}_{k=1}P_{ik}\mu_k - \sum^{K}_{k=1}P_{jk}\mu_k \\
& = \sum_{k:k\ne j} P_{ik}\mu_k - \sum_{k: k\ne i}P_{jk}\mu_k + P_{ij}\mu_j - P_{ji}\mu_i\\
& > \sum_{k:k\ne j} P'_{ik}\mu_k - \sum_{k: k\ne i}P'_{jk}\mu_k + P'_{ij}\mu_j - P'_{ji}\mu_i\\
& = \sum^{K}_{k=1}P'_{ik}\mu_k - \sum^{K}_{k=1}P'_{jk}\mu_k\\
& = \log\left(\frac{F_{\bmu}(P')_i}{F_{\bmu}(P')_j} \right),
\end{align*}
and therefore, 
\begin{align*}
\frac{F(P)_i}{F(P)_j}\ \ge\ \frac{F(P')_i}{F(P')_j}.
\end{align*}

\end{proof}

\subsection{Proof of Theorem \ref{thm:pairwise_gap_ipo}}
\thmpairwisegapipo*

\begin{proof}
Fix any $1\le i<j\le K$ and define the row-difference vector
\[
\a := P_i - P_j \in \mathbb{R}^K,
\qquad
a_k := P_{ik} - P_{jk}.
\]

Consider the convex program OP~\eqref{eq:IPO_convex_op} with sampling distribution $\bmu$. From  \eqref{eq:ipo_solution}, the optimal solution
$\bpi(\bmu)$ satisfies
\[
\log \pi_\ell(\bmu) - \log \pi_{\mathrm{ref},\ell}
=
\beta \sum_{k=1}^K \mu_k P_{\ell k} + c(\bmu),
\qquad \forall \ell \in [K],
\]
where $c(\bmu)\in\mathbb{R}$ is a normalization constant independent of $\ell$.
By the definition
\[
\btheta(\bmu)
=
\beta^{-1}\big(\log\bpi(\bmu)-\log\bpi_{\mathrm{ref}}\big),
\]
we obtain
\[
\theta_\ell(\bmu)
=
\sum_{k=1}^K \mu_k P_{\ell k} + \frac{c(\bmu)}{\beta}.
\]
Taking the difference between indices $i$ and $j$, the constant term cancels, yielding
\begin{equation}\label{eq:theta_gap_linear}
\theta_i(\bmu)-\theta_j(\bmu)
=
\sum_{k=1}^K \mu_k (P_{ik}-P_{jk})
=
\a^\top \bmu.
\end{equation}
Analogously,
\begin{equation}\label{eq:theta_gap_baseline}
\theta_i(\bmu')-\theta_j(\bmu')
=
\a^\top \bmu'.
\end{equation}

Now define $\x := \bmu - \bmu'$. By construction, both $\bmu$ and $\bmu'$ are probability distributions on $\Delta(K)$, hence $\sum_{k=1}^K x_k = 0$. Moreover, under the sampling assumption of the theorem, it holds that $x_1 \ge x_2 \ge \cdots \ge x_K$. Next, we prove the following Lemma to establish the equivalence between the definition of majorization and a useful property of $\bm{a}$:

\begin{lemma}[Equivalent Characterizations of Majorization]
\label{lem:maj_equiv}
For any $\a=(a_1,\dots,a_K)\in\mathbb R^K$, the following statements are equivalent:
\begin{enumerate}
    \item[(i)] $\a$ satisfies the majorization condition, i.e.,
    \[
    \frac{1}{m}\sum_{k=1}^m a_k
    \;\ge\;
    \frac{1}{K}\sum_{k=1}^K a_k,
    \qquad
    \forall m=1,\dots,K-1.
    \]
    \item[(ii)] For all vectors $\x\in\mathbb R^K$ satisfying
    \[
    x_1 \ge x_2 \ge \cdots \ge x_K,
    \qquad
    \sum_{k=1}^K x_k = 0,
    \]
    we have
    \[
    \a^\top \x \;\ge\; 0.
    \]
\end{enumerate}
Moreover, the inequality in (ii) is strict for all nonzero such $\x$ if and only if
at least one inequality in (i) is strict.
\end{lemma}

\begin{proof}[Proof of Lemma \ref{lem:maj_equiv}]
We prove $(i)\Rightarrow(ii)$ and $(ii)\Rightarrow(i)$ separately.

Define the extreme rays of the monotone zero-sum cone
\[
\mathcal C
=
\left\{
\x\in\mathbb R^K :
x_1 \ge x_2 \ge \cdots \ge x_K,
\;
\sum_{k=1}^K x_k = 0
\right\}.
\]
This is a closed, convex, polyhedral cone.
A standard fact from convex analysis is that $\mathcal C$ is generated by the
$K-1$ extreme rays
\[
\x^{(m)} :=
(\underbrace{K-m,\dots,K-m}_{m\text{ times}},
\underbrace{-m,\dots,-m}_{K-m\text{ times}}),
\qquad m=1,\dots,K-1.
\]
Each $\x^{(m)}$ is nonincreasing and satisfies $\sum_k u^{(m)}_k=0$.
Moreover, every $\x\in\mathcal C$ admits a representation
\begin{equation}\label{eq:cone_decomp}
\x = \frac{1}{K}\sum_{m=1}^{K-1} (x_m-x_{m+1})\,\x^{(m)},
\end{equation}
where the coefficients $x_m-x_{m+1}\ge 0$ due to monotonicity of $\x$.
Hence, to verify $\a^\top \x\ge 0$ for all $\x\in\mathcal C$,
it suffices to verify it for each $\x^{(m)}$.

\paragraph{``If'' direction: $(i)\Rightarrow(ii)$.}
Assume $\a$ satisfies the majorization inequalities, that is, for any $m\in\{1,\dots,K-1\}$,
\[
\a^\top \x^{(m)}
=
(K-m)\sum_{k=1}^m a_k
-
m\sum_{k=m+1}^K a_k.
\]
Using $\sum_{k=m+1}^K a_k = \sum_{k=1}^K a_k - \sum_{k=1}^m a_k$,
we rewrite
\[
\a^\top \x^{(m)}
=
K\sum_{k=1}^m a_k
-
m\sum_{k=1}^K a_k.
\]
Thus,
\[
\a^\top \x^{(m)} \ge 0
\quad\Longleftrightarrow\quad
\sum_{k=1}^m a_k
\;\ge\;
\frac{m}{K}\sum_{k=1}^K a_k,
\]
which holds by assumption.
By linearity and the decomposition \eqref{eq:cone_decomp},
this implies $\a^\top \x\ge 0$ for all $\x\in\mathcal C$.

\paragraph{``Only if'' direction: $(ii)\Rightarrow(i)$.}
Assume that $\a^\top \x\ge 0$ for all $\x\in\mathcal C$.
In particular, this holds for each extreme ray $\x^{(m)}$.
Repeating the calculation above yields
\[
K\sum_{k=1}^m a_k
-
m\sum_{k=1}^K a_k
=
\a^\top \x^{(m)} \ge 0,
\]
which is exactly the majorization inequality for prefix $m$.
Since this holds for all $m=1,\dots,K-1$, $\a$ is majorized.

\paragraph{Additional argument for the strictness:} If at least one majorization inequality is strict, then
$\a^\top \x^{(m)}>0$ for the corresponding $m$,
and hence $\a^\top \x>0$ for all nonzero $\x\in\mathcal C$.
Conversely, if all majorization inequalities hold with equality,
then $\a^\top \x^{(m)}=0$ for all $m$, and hence
$\a^\top \x=0$ for all $\x\in\mathcal C$.
\end{proof}

Since $P_i\succ P_j$ by assumption, the equivalent characterization of \Cref{def:maj} yields
\[
\a^\top \x \ge 0.
\]
Substituting $\x=\bmu-\bmu'$ gives
\[
\a^\top \bmu \ge \a^\top \bmu'.
\]
Combining with \eqref{eq:theta_gap_linear}--\eqref{eq:theta_gap_baseline}, we obtain
\begin{equation}\label{eq:gap_monotone}
\theta_i(\bmu)-\theta_j(\bmu)
\ge
\theta_i(\bmu')-\theta_j(\bmu').
\end{equation}
If the majorization is strict and $\bmu\neq\bmu'$, then the inequality above is strict.

Get back to the main proof, since $P$ satisfies strong transitivity with order $1\succ 2\succ\cdots\succ K$, it holds that for all $i<j$ and all $k$, $P_{ik} \ge P_{jk}$. Moreover, strict inequality holds for at least one coordinate because $P_i\neq P_j$. In particular, $a_i = P_{ii}-P_{ji} = \tfrac12 - P_{ji} > 0$,
where we used $P_{ii}=\tfrac12$ and $P_{ji}<\tfrac12$ due to $i\succ j$.
Hence $\a$ is nonnegative and not identically zero.

Since $\bmu'$ has full support,
\[
\a^\top \bmu' > 0,
\]
which implies
\begin{equation}\label{eq:543}
\theta_i(\bmu')-\theta_j(\bmu') > 0.   
\end{equation}

Combining  \eqref{eq:gap_monotone} with  \eqref{eq:543} yields
\begin{equation}\label{eq:549}
\theta_i(\bmu)-\theta_j(\bmu)
>
\theta_i(\bmu')-\theta_j(\bmu')
> 0,    
\end{equation}
which completes the proof.
\end{proof}

\subsection{Proof of Corollary \ref{cor:bt_headtail_maj}}

\corbtheadtailmaj*

\begin{proof}
First of all, we prove an useful auxiliary lemma:

\begin{lemma}[Head--Tail Bounded Sequence Implies Majorization]\label{lm:head_tail_maj}
\label{lem:head_tail_majorization}
Let $a=(a_1,\dots,a_K)\in\mathbb [0,1]^K$ be a nonnegative sequence.
Fix an integer $H\in\{1,\dots,K-1\}$.
Assume the following conditions hold:
\begin{enumerate}
    \item[(L1)] (\textbf{Tail upper bound and monotonicity}) There exists $\varepsilon\ge 0$ such that
    \[
    \epsilon \ge a_{H+1}\ge a_{H+2}\ge \cdots \ge a_K.
    \]
    \item[(L2)] (\textbf{Head lower bound}) The elements in the head have a uniform lower bound
    \[
    a_k \ge \frac{H+(K-H)\varepsilon}{K},\qquad \forall\, k\le H.
    \]
\end{enumerate}

Then, the sequence $\bm{a}$ satisfies Maj specified in Definition \ref{def:maj}.
\end{lemma}

\begin{proof} Let $S:=\sum_{k=1}^K a_k$ and $A_m:=\sum_{k=1}^m a_k$. First of all, by (L2) on the head and (L1) on the tail we can upper bound on the total sum
\begin{equation}\label{eq:A.1}
S=\sum_{k=1}^H a_k + \sum_{k=H+1}^K a_k
\;\le\;
H\cdot 1 + (K-H)\varepsilon
=
H+(K-H)\varepsilon.    
\end{equation}

Combining with (L2) and  \eqref{eq:A.1}, we obtain
\begin{equation}\label{eq:A.2}
A_m\ge m\cdot \frac{H+(K-H)\varepsilon}{K} \ge \frac{m}{K}S, \qquad \forall\, m\le H.    
\end{equation}

Next, we show that  \eqref{eq:A.2} also holds for $m>H$ and thus concluding the verification of the Majorization condition for $\bm{a}$. By the tail upper bound and monotonicity assumption (L1),
\[
\sum_{k=m+1}^K a_k
\le
(K-m)\varepsilon, \quad\forall H\le m\le K.
\]
Equivalently,
\[
A_m
=
S-\sum_{k=m+1}^K a_k
\ge
S-(K-m)\varepsilon.
\]
Hence,
\[
A_m-\frac{m}{K}S
\ge
\Bigl(1-\frac{m}{K}\Bigr)S-(K-m)\varepsilon
=
(K-m)\Bigl(\frac{S}{K}-\varepsilon\Bigr).
\]
Since $S/K=\frac{H+(K-H)\epsilon}{K}\ge \frac{H\epsilon+(K-H)\epsilon}{K} = \varepsilon$, we have $A_m\ge \frac{m}{K}S$ holds for all $m>H$ as well. Combining the two cases proves that $\bm{a}$ satisfies the majorization condition.
\end{proof}

Now we prove Corollary \ref{cor:bt_headtail_maj}. Fix any $i<j$ with $i,j\in\H$ and define the row-difference vector
\[
\a := P_i-P_j \in \mathbb{R}^K,
\qquad
a_k := P_{ik}-P_{jk}
= \sigma(\theta_i-\theta_k)-\sigma(\theta_j-\theta_k).
\]
Since $\theta_i>\theta_j$ and $\sigma(\cdot)$ is increasing with range $[0,1]$, we have $0\le a_k\le 1$ for all $k\in [K]$.

For the first step, we verify the tail upper bound and monotonicity. Given the head-tail separation condition that $\theta_k\le -\delta, k\in\T$ and $\theta_k\ge 0, k\in\H$, it holds that $u_k:= \theta_j-\theta_k \ \ge\ \delta \ >\ 0$. Let $\Delta_{ij}:=\theta_i-\theta_j\ge 8\left(\frac{H}{K}+e^{-\delta}\right) := f(H,K,\delta)$. Using the identity
\[
\sigma(u+\Delta)-\sigma(u)=\int_u^{u+\Delta}\sigma'(t)\,dt
\]
and the fact that $\sigma'(t)=\sigma(t)(1-\sigma(t))\le e^{-t}\quad (t\ge 0)$, we obtain
\[
a_k
=
\sigma(u_k+\Delta_{ij})-\sigma(u_k)
\le
\int_{u_k}^{u_k+\Delta_{ij}} e^{-t}\,dt
=
e^{-u_k}\bigl(1-e^{-\Delta_{ij}}\bigr)
\le
e^{-\delta}\bigl(1-e^{-f(H,K,\delta)}\bigr)
=:\varepsilon.
\]
Moreover, since $\theta_k$ is nonincreasing in $k$, the sequence $u_k=\theta_j-\theta_k$
is nondecreasing over $k\in\T$. For any fixed $\Delta>0$, the map
\[
u\mapsto \sigma(u+\Delta)-\sigma(u)
\]
is nonincreasing for $u\ge 0$ (In fact, a direct differentiation shows its derivative is
$\sigma'(u+\Delta)-\sigma'(u)\le 0$ on $u\ge 0$).
Therefore, $(a_{H+1},\dots,a_K)$ is nonincreasing, i.e.,
\[
a_{H+1}\ge a_{H+2}\ge \cdots \ge a_K,
\qquad
\text{and}\qquad
a_k\le \varepsilon\ \ \forall k\in\T.
\]

Next, we verify the head lower bound.
At column $k=j$ we have $P_{jj}=\tfrac12$ and $P_{ij}=\sigma(\theta_i-\theta_j)\ge \sigma(f(H,K,\delta))$. Hence
\[
a_j
=
P_{ij}-P_{jj}
=
\sigma(\theta_i-\theta_j)-\tfrac12
\ge
\sigma(f(H,K,\delta))-\tfrac12
=:\alpha.
\]
In particular, for every $m\ge j$ we have $A_m:=\sum_{k=1}^m a_k \ge a_j\ge \alpha$.

Finally, we verify majorization via a head--tail bound. To conclude $\bm{a}$ satisfies Maj by Lemma \ref{lm:head_tail_maj}, it is sufficient to show that
\begin{equation}\label{eq:f_choice_for_maj}
\alpha
=
\sigma(f(H,K,\delta))-\tfrac12
\ \ge\
\frac{H+(K-H)\varepsilon}{K},
\qquad
\varepsilon=e^{-\delta}\bigl(1-e^{-f(H,K,\delta)}\bigr),
\end{equation}
given the choice of $f(H,K,\delta)=8\left(\frac{H}{K}+e^{-\delta}\right)$. The detailed proof can be found in Lemma \ref{lm:A.3}. Therefore, $P_i-P_j$ is Majorized for every $i<j$ in the head set $\H$, and Theorem~\ref{thm:pairwise_gap_ipo} implies
\[
\theta_i(\bmu)-\theta_j(\bmu)>\theta_i(\bmu')-\theta_j(\bmu')>0,
\qquad
\forall\, 1\le i<j\le H.
\]
This concludes the proof.

\end{proof}

\begin{lemma}\label{lm:A.3}
    For any $\delta>3\ln2$ and $K>H>0$ such that $f(H,K,\delta)=8\left(\frac{H}{K}+e^{-\delta}\right)<1$, it holds that 
    \begin{equation}
        \sigma(f(H,K,\delta))-\tfrac12
\ \ge\
\frac{H+(K-H)\varepsilon}{K},
    \end{equation}
    where $\varepsilon=e^{-\delta}\bigl(1-e^{-f(H,K,\delta)}\bigr)$.
\end{lemma}

\begin{proof}
Let $p:=H/K$ and $a:=e^{-\delta}$. By assumption, $p+a\le \tfrac18$, hence
$f=8(p+a)\in[0,1]$.

For any $x\ge 0$,
\[
\sigma(x)-\tfrac12
= \frac{1-e^{-x}}{2(1+e^{-x})}
\;\ge\;
\frac{1-e^{-x}}{4},
\]
since $1+e^{-x}\le 2$. Moreover, for $x\in[0,1]$, convexity of $e^{-x}$ implies
\[
1-e^{-x}
\;\ge\;
(1-e^{-1})x
\;\ge\;
\tfrac12 x.
\]
Applying these bounds with $x=f$ yields
\begin{equation}\label{eq:842}
\sigma(f)-\tfrac12
\;\ge\;
\frac{f}{8}
\;=\;
p+a.    
\end{equation}

On the other hand,
\[
\varepsilon
= a(1-e^{-f})
\;\le\;
a,
\]
and therefore
\begin{equation}\label{eq:859}
\frac{H+(K-H)\varepsilon}{K}
= p+(1-p)\varepsilon
\;\le\;
p+(1-p)a
\;\le\;
p+a.
\end{equation}
Combining \eqref{eq:859} and \eqref{eq:842}, we complete the proof.
\end{proof}

\section{Omitted Proofs in \Cref{sec:self_ref_dynamics}}\label{app:online}

\subsection{Proof of Proposition \ref{prop:condorcet_instability_mixed}}

\propcondorcetinstabilitymixed*

\begin{proof}

Let $\bpi^{\star}:=(1/3,1/3,1/3)$ and assume $\bpi_{\rm ref}=\bpi_0=\bpi^{\star}$.
Define the MRS-IPO map $T:\Delta_3^\circ\to\Delta_3^\circ$ by
\[
T(\bpi)
=
\softmax\!\Big(
\alpha\log\bpi+(1-\alpha)\log \bpi^{\star}
+\beta\,P\big(\lambda\bpi+(1-\lambda)\bpi^{\star}\big)
\Big),
\]
so that the dynamics is $\bpi_{t+1}=T(\bpi_t)$.

First of all, it is straightforward to see $\bpi^{\star}$ is a fixed point: each row of $P$ in \eqref{eq:rps_P} sums to
$\frac12+(\frac12+a)+(\frac12-a)=\frac32$, hence $P\bpi^{\star}=\frac12\mathbf{1}$. Therefore,
$P(\lambda \bpi^{\star}+(1-\lambda)\bpi^{\star})=P\bpi^{\star}=\frac12\mathbf{1}$ and $T(\bpi^{\star})=\softmax\!\Big(\log \bpi^{\star}+\tfrac{\beta}{2}\mathbf{1}\Big)
=\bpi^{\star}$. 

Next we analyze the stability of $\bpi^{\star}$. We compute the Jacobian of $T$ at $\bpi^{\star}$: write $T(\bpi)=\softmax(g(\bpi))$ where $g(\bpi)=\alpha\log\bpi+(1-\alpha)\log \bpi^{\star}
+\beta\,P\big(\lambda\bpi+(1-\lambda)\bpi^{\star}\big)$. The Jacobian of softmax is $D[\softmax(\z)]=J(\p):=\mathrm{Diag}(\p)-\p\p^\top$ for $\p=\softmax(\z)$. By the chain rule, $D[T(\bpi)]=J(T(\bpi))\,D[g(\bpi)]$. At $\bpi^{\star}$, $T(\bpi^{\star})=\bpi^{\star}$, hence $D[T(\bpi^{\star})]=J(\bpi^{\star})\,D[g(\bpi^{\star})]$. We compute
\[
D[g(\bpi^{\star})]
=
\alpha\,\mathrm{Diag}(1/\bpi^{\star})+\beta\lambda P
=
3\alpha I+\beta\lambda P,
\]
since $\pi^\star_i=1/3$.
Also,
\[
J(\bpi^{\star})=\mathrm{Diag}(\bpi^{\star})-\bpi^{\star}(\bpi^{\star})^\top=\frac13 I-\frac19\mathbf{1}\mathbf{1}^\top.
\]
Restricting to the tangent space $\mathbf{1}^\perp:=\{\v:\mathbf{1}^\top \v=0\}$, we have
$J(\bpi^{\star})\v=\frac13 \v$ and $P\v=A\v$ where
\[
A:=P-\tfrac12\mathbf{1}\mathbf{1}^\top
=
\begin{pmatrix}
0 & a & -a\\
-a & 0 & a\\
a & -a & 0
\end{pmatrix},
\qquad A^\top=-A.
\]
Therefore, for all $\v\in\mathbf{1}^\perp$,
\[
D[T(\bpi^{\star})]\v
=
\frac13(3\alpha I+\beta\lambda P)\v
=
\Big(\alpha I+\frac{\beta\lambda}{3}A\Big)\v.
\]
Hence the relevant linearization on the simplex is
\[
M:=D[T(\bpi^{\star})]\big|_{\mathbf{1}^\perp}
=
\alpha I+\frac{\beta\lambda}{3}A\big|_{\mathbf{1}^\perp}.
\]

A direct multiplication yields $A^2=-3a^2\Big(I-\frac13\mathbf{1}\mathbf{1}^\top\Big)$, so on $\mathbf{1}^\perp$ we have $A^2=-3a^2 I$. Thus the eigenvalues of $A|_{\mathbf{1}^\perp}$ are $\pm i\sqrt{3}\,a$, and the eigenvalues of $M$ are $\lambda_\pm
=
\alpha\pm i\frac{\beta\lambda}{3}\sqrt{3}\,a
=
\alpha\pm i\frac{\beta\lambda a}{\sqrt{3}}$ and their modulus is
\[
|\lambda_\pm|
=
\sqrt{\alpha^2+\frac{a^2\beta^2\lambda^2}{3}}.
\]
Therefore, if \eqref{eq:instability_region} holds, $|\lambda_\pm|>1$ and the spectral radius of $D[T(\bpi^{\star})]|_{\mathbf{1}^\perp}$ exceeds $1$, implying that $\bpi^{\star}$ is a linearly unstable fixed point.

Since $\lambda_\pm$ are complex with nonzero imaginary part whenever $\beta\lambda a\neq 0$, the linearized map on $\mathbf{1}^\perp$ is a rotation combined with expansion by factor
$\rho:=|\lambda_\pm|>1$.
Thus for any initialization $\bpi_1$ sufficiently close to $\bpi^{\star}$ but not equal to $u$, the iterates do not converge to $\bpi^{\star}$ and instead exhibit oscillatory (spiraling) behavior around $\bpi^{\star}$. This completes the proof.
\end{proof}

\subsection{Proof of Theorem \ref{thm:mix_contraction_generalP}}

\thmmixcontractiongeneralP*

\begin{proof}
Define the MRS-IPO update map $T:\Delta_K^\circ\to\Delta_K^\circ$ by
\[
T(\bpi)
:=
\softmax\!\Big(
\alpha\log\bpi+(1-\alpha)\log\bpi_{\rm ref}
+\beta\,P\big(\lambda\bpi+(1-\lambda)\bpi_0\big)
\Big),
\]
so that the dynamics is $\bpi_{t+1}=T(\bpi_t)$.
We will show that under \eqref{eq:mix_contraction_condition_general}, $T$ is a contraction (in an equivalent coordinate system), which implies a unique fixed point and global convergence.

First, we work with centered logit coordinates by letting $\mathbf{1}\in\bbR^K$ be the all-ones vector and letting $\Pi := I-\frac{1}{K}\mathbf{1}\mathbf{1}^\top$ be the orthogonal projector onto $\mathbf{1}^\perp:=\{x\in\bbR^K:\mathbf{1}^\top x=0\}$.
For any $\bpi\in\Delta_K^\circ$, define the centered logit
\[
\z(\bpi):=\Pi\log\bpi\in\mathbf{1}^\perp.
\]
Since $\log\bpi=\z(\bpi)+c(\bpi)\mathbf{1}$ for some scalar $c(\bpi)$ and $\softmax(\x+c\mathbf{1})=\softmax(\x)$, we have the bijection
\[
\bpi=\softmax(\z(\bpi)),
\qquad
\z(\softmax(\x))=\x \quad \text{for all } \x\in\mathbf{1}^\perp.
\]
Hence it suffices to study the induced map on $\mathbf{1}^\perp$.

Let $\x_t:=\z(\bpi_t)\in\mathbf{1}^\perp$ so that $\bpi_t=\softmax(\x_t)$.
Taking centered logits of $\bpi_{t+1}=T(\bpi_t)$ gives
\begin{align*}
\x_{t+1}
&=\Pi\log\bpi_{t+1}
=\Pi\Big(
\alpha\log\bpi_t+(1-\alpha)\log\bpi_{\rm ref}
+\beta\,P\big(\lambda\bpi_t+(1-\lambda)\bpi_0\big)
\Big)\\
&=\alpha \Pi\log\bpi_t + \Pi\Big((1-\alpha)\log\bpi_{\rm ref}+\beta(1-\lambda)P\bpi_0\Big)
+\beta\lambda\,\Pi P\,\bpi_t.
\end{align*}
Define the constant vector
\[
\mathbf{b}:=\Pi\Big((1-\alpha)\log\bpi_{\rm ref}+\beta(1-\lambda)P\bpi_0\Big)\in\mathbf{1}^\perp.
\]
Moreover, under the standard dueling property $P_{ij}+P_{ji}=1$ and $P_{ii}=1/2$, we can write
\[
P=\tfrac12\mathbf{1}\mathbf{1}^\top+\widetilde A,
\qquad
\widetilde A:=P-\tfrac12\mathbf{1}\mathbf{1}^\top,
\]
so $\widetilde A\mathbf{1}=\mathbf{0}$ and $P\bpi=\tfrac12\mathbf{1}+\widetilde A\bpi$.
Since $\Pi(\tfrac12\mathbf{1})=\mathbf{0}$, we have $\Pi P\bpi=\Pi\widetilde A\bpi$.
Thus the centered-logit dynamics becomes
\[
\x_{t+1}
=
F(\x_t)
:=
\alpha \x_t
+\beta\lambda\,\Pi\widetilde A\,\softmax(\x_t)
+\mathbf{b},
\qquad \x_t\in\mathbf{1}^\perp.
\]

Let $\bpi=\softmax(\x)$. The Jacobian of softmax is
\[
D\,[\softmax(\x)]=J(\bpi):=\mathrm{Diag}(\bpi)-\bpi\bpi^\top,
\]
which satisfies the uniform spectral norm bound $\|J(\bpi)\|_2\le \tfrac12$ for all $\bpi\in\Delta_K$.
Indeed, for any unit vector $\v$,
\[
\v^\top J(\bpi)\v
=
\sum_{i=1}^K \pi_i v_i^2-\Big(\sum_{i=1}^K \pi_i v_i\Big)^2
=\mathrm{Var}(v_I),
\qquad I\sim\bpi,
\]
and $\mathrm{Var}(v_I)\le \frac{(\max_i v_i-\min_i v_i)^2}{4}\le \frac{(\sqrt{2})^2}{4}=\frac12$.
Taking the supremum over $\|\v\|_2=1$ gives $\|J(\bpi)\|_2\le \tfrac12$.

Now we are ready to derive the contraction of $F$. Differentiating $F$ yields, for $\x\in\mathbf{1}^\perp$ and $\bpi=\softmax(\x)$,
\[
D[F(\x)]
=
\alpha I + \beta\lambda\,\Pi\widetilde A\,J(\bpi).
\]
Using $\|\Pi\|_2=1$ and submultiplicativity,
\begin{equation}\label{eq:232}
\|D[F(\x)]\|_2
\le
\alpha + \beta\lambda\,\|\widetilde A\|_2\,\|J(\bpi)\|_2
\le
\alpha + \frac{\beta\lambda}{2}\|\widetilde A\|_2.    
\end{equation}

Therefore, if \eqref{eq:232} and \eqref{eq:mix_contraction_condition_general} holds, then $L<1$, so $F$ is a contraction on the complete metric space $(\mathbf{1}^\perp,\|\cdot\|_2)$:
\[
\|F(\x)-F(\y)\|_2 \le L\|\x-\y\|_2,\qquad \forall\,\x,\y\in\mathbf{1}^\perp.
\]

By the Banach fixed-point theorem, $F$ admits a unique fixed point $\x^\star\in\mathbf{1}^\perp$, and for any initialization $\x_1\in\mathbf{1}^\perp$,
the iterates $\x_{t+1}=F(\x_t)$ satisfy $\x_t\to \x^\star$.
Define
\[
\bpi^{\star}:=\softmax(\x^\star)\in\Delta_K^\circ.
\]
Since the map $\x\mapsto\softmax(\x)$ is continuous, we have $\bpi_t=\softmax(\x_t)\to \softmax(\x^\star)=\bpi^{\star}$.
Finally, because $\x^\star=F(\x^\star)$ is equivalent (via the bijection between $\mathbf{1}^\perp$ and $\Delta_K^\circ$) to $\bpi^{\star}=T(\bpi^{\star})$,
the limit $\bpi^{\star}$ is precisely the unique fixed point of the original MRS-IPO map $T$ on $\Delta_K^\circ$.

To further refine the bound, we propose and prove the following proposition that derives sufficient conditions to replace the $P$-dependent stability condition with $P$-agnostic ones. We provide two upper bounds of the $P$-related norm in \eqref{thm:mix_contraction_generalP}, depending on whether $P$ is a general preference structure or enjoys sparsity.

\begin{proposition}[Explicit norm bound under general/sparse preference structures]
\label{prop:norm_bound_dueling}
Assume $P\in[0,1]^{K\times K}$ satisfies the dueling structure
$P_{ii}=\tfrac12$ and $P_{ij}+P_{ji}=1$ for all $i\neq j$.
Define $\widetilde A := P-\tfrac12\mathbf{1}\mathbf{1}^\top$. Then it holds that
\begin{equation}
\label{eq:norm_bound_dueling}
\|\widetilde A\|_2 \;\le\; \frac{K-1}{2}.
\end{equation}

Moreover, if $P$ enjoys a sparse structure in the sense that for each $i$, there are at most $d$ indices $j\neq i$ such that $P_{ij}\neq 1/2$.
Then 
\begin{equation}
\label{eq:norm_bound_dueling_sparse}
\|\widetilde A\|_2\le \frac{d}{2}.
\end{equation}
\end{proposition}

Consequently, a sufficient condition for the contraction inequality
$\alpha+\frac{\beta\lambda}{2}\|\widetilde A\|_2<1$ becomes $\alpha+\frac{\beta\lambda}{4}(K-1)<1$, or $\alpha+\frac{\beta\lambda d}{4}<1$ if $P$ is $d$-sparse.

\begin{proof}
The skew-symmetry follows immediately as $\widetilde A_{ji}=P_{ji}-\tfrac12=(1-P_{ij})-\tfrac12=-(P_{ij}-\tfrac12)=-\widetilde A_{ij}$, and $\widetilde A_{ii}=P_{ii}-\tfrac12=0$.
Since $P_{ij}\in[0,1]$, we have $|\widetilde A_{ij}|\le\tfrac12$ for all $i\neq j$.

For the first bound, note that each row has at most $K-1$ nonzero entries, each with magnitude at most $1/2$, hence
\[
\|\widetilde A\|_\infty
=
\max_i \sum_{j=1}^K |\widetilde A_{ij}|
\le
\frac{K-1}{2}.
\]
Similarly, each column satisfies the same bound, so
\[
\|\widetilde A\|_1
=
\max_j \sum_{i=1}^K |\widetilde A_{ij}|
\le
\frac{K-1}{2}.
\]
Using the standard inequality $\|M\|_2\le \sqrt{\|M\|_1\|M\|_\infty}$ yields
\[
\|\widetilde A\|_2 \le \sqrt{\|\widetilde A\|_1\|\widetilde A\|_\infty}
\le
\frac{K-1}{2}.
\]

For the Frobenius bound, since $\widetilde A_{ii}=0$ and $|\widetilde A_{ij}|\le 1/2$ for $i\neq j$,
\[
\|\widetilde A\|_F^2
=
\sum_{i\neq j}\widetilde A_{ij}^2
\le
K(K-1)\cdot\frac{1}{4},
\]
so $\|\widetilde A\|_2\le \|\widetilde A\|_F \le \frac12\sqrt{K(K-1)}$.
Finally, substituting $\|\widetilde A\|_2\le (K-1)/2$ into
$\alpha+\frac{\beta\lambda}{2}\|\widetilde A\|_2<1$
gives the explicit sufficient condition
$\alpha+\frac{\beta\lambda}{4}(K-1)<1$.

Now we prove the second claim for $d$-sparse $P$. Since $|\widetilde A_{ij}|\le 1/2$ and each row has at most $d$ nonzero off-diagonal entries,
\[
\|\widetilde A\|_\infty
=
\max_i \sum_{j=1}^K |\widetilde A_{ij}|
\le d\cdot \frac12=\frac{d}{2}.
\]
By skew-symmetry, $\|\widetilde A\|_1=\|\widetilde A\|_\infty$.
Using $\|M\|_2\le \sqrt{\|M\|_1\|M\|_\infty}$ gives
$\|\widetilde A\|_2\le \|\widetilde A\|_\infty\le d/2$.
\end{proof}
\end{proof}

\subsection{Proof of \Cref{thm:true_vs_eps_collapse}}\label{app:eps_collapse_ipo}

We show an extended version of \Cref{thm:true_vs_eps_collapse} under a relaxed assumption on $P$. We first introduce the following definitions

\begin{definition}
\label{def:sst}
Suppose $P$ is a strongly transitive (ST) preference matrix $P$ with total order $1 \succ 2 \succ \cdots \succ K$. Then we say 
\begin{enumerate}
    \item $P$ is SST, if for any $i\in[K-1]$, $P_{ik}>P_{i+1,k}$ holds for at least one $k\in[K]$;
    \item $P$ is ST$^+$, if $P_{1k}>P_{2k}$ holds for at least one $k\in[K]$.
\end{enumerate}
\end{definition}

We can see SST means for any $1\le i<j\le K$, $P_i$ weakly dominates $P_j$ element-wise with at least one strict inequality, and ST$^+$ only requires it holds for $(i,j)=(1,2)$ --- there can be many tied responses within $\{2,\cdots,K\}$ with identical $P_i$. The next full version of \Cref{thm:true_vs_eps_collapse} shows collapse guarantee under ST$^+$ preference structures. 

\begin{theorem}\label{thm:454}
Let $\bpi_0\in\Delta_K^\circ$ and $\bpi_{\rm ref}\in\Delta_K^\circ$.
Consider the MRS-IPO dynamics $\bpi_{t+1}=\mathrm{MRS}(\bpi_t;\alpha,\beta,\lambda)$. These dynamics exhibit the following two forms of collapse.
\begin{itemize}

\item \textbf{($\varepsilon$-collapse when $\beta\lambda/(1-\alpha)$ is large).}
If $\alpha,\lambda\in[0,1)$, there exists a finite $T$ such that for all $t\ge T$,
\begin{itemize}
    \item if $P$ is SST, then \[
\bpi_{t,1}\ \ge\ 1 - \left(\exp\!\left(\frac{\beta\lambda\delta}{2(1-\alpha)}\right)-1\right)^{-1},
\]
\item if $P$ is ST$^+$, then \[
\bpi_{t,1}\ \ge\ 1 - (K-1) \exp\!\left(-\frac{\beta\lambda\delta_1}{2(1-\alpha)}\right),
\]
\end{itemize}

where $\delta_i\coloneqq(1-\lambda)
\sum_{j=1}^K \pi_{0,j}\,(P_{ij}-P_{i+1,j})
$ and $
\delta
:=
\min_{i\in[K-1]}
\delta_i$ are positive constants.
\item 
If $\alpha=1$, there exists a finite $T$ such that for all $t\ge T$,
\[
H(\bpi_{t+1})<H(\bpi_t),
\]
where $H(\bpi) \coloneqq-\sum_{i=1}^K \pi_i\log\pi_i$ is the Shannon entropy. In addition, $\bpi_t \to \e_1$ and $H(\bpi_t)\to 0$ as $t\to+\infty$.
\end{itemize}
\end{theorem}

Next, we prove this Theorem. For the convenience of notation in its proof, we define log-ratios w.r.t. any policy $\bpi_t\in \Delta_K^\circ$ and rewrite the MRS-IPO dynamics in the following Lemma: 

\begin{lemma}[Log-ratio recursion]
\label{lem:logratio_recursion_mixed}
For $i=1,\dots,K-1$, define adjacent log-ratios
\[
r_{t,i}:=\log\frac{\pi_{t,i}}{\pi_{t,i+1}},
\qquad
r_{{\rm ref},i}:=\log\frac{\pi_{{\rm ref},i}}{\pi_{{\rm ref},i+1}}.
\]
Then the MRS-IPO dynamics imply the exact recursion
\begin{equation}
\label{eq:r_recursion}
r_{t+1,i}
=
\alpha r_{t,i}
+(1-\alpha)r_{{\rm ref},i}
+\beta\lambda\Big((P\bpi_t)_i-(P\bpi_t)_{i+1}\Big).
\end{equation}
\end{lemma}

\begin{proof}
For any vector $\z$, $\softmax(\z)_i/\softmax(\z)_{i+1}=\exp(z_i-z_{i+1})$.
Apply this to the logits inside \eqref{eq:mix_ref_softmax_form} yields 
\eqref{eq:r_recursion}.
\end{proof}

The next Lemma shows how entropy collapse can be measured by a large log-ratio.
\begin{lemma}[From large log-ratios to $\varepsilon$-collapse]
\label{lem:ratio_to_topmass}
Let $\bpi\in\Delta_K^{\circ}$ satisfy $r_i:=\log(\pi_i/\pi_{i+1})\ge R$ for all $i\in [K-1]$, then
\begin{equation}
\label{eq:topmass_bound}
1-\pi_1 \;\le\; \frac{1}{e^{R}-1}.
\end{equation}
In particular, if $R\ge \log\!\big(1+1/\varepsilon\big)$ then $\bpi$ is $\varepsilon$-collapsed. If $r_i\ge R$ only holds for $i=1$, then $1-\pi_1\le (K-1)e^{-R}$.
\end{lemma}

\begin{proof}
The condition $r_i\ge R$ implies $\pi_{i+1}\le e^{-R}\pi_i$ for each $i$.
Thus $\pi_k\le e^{-(k-1)R}\pi_1$ for all $k\ge 2$, and hence
\[
1-\pi_1
=
\sum_{k=2}^K \pi_k
\le
\sum_{k=2}^K e^{-(k-1)R}
<\sum_{k=2}^{+\infty} e^{-(k-1)R}=\frac{1}{e^{R}-1}.
\]
As a result, when $R\ge \log\!\big(1+1/\varepsilon\big)$, $1-\pi_1<\varepsilon$.
If we only have $r_1\ge R$, we have $\pi_{i+1}\le e^{-R}\pi_1, \forall i\ge 1$ and
\[
1-\pi_1
=
\sum_{k=2}^K \pi_k
\le
(K-1) e^{-R}.
\]
\end{proof}

In addition, we need the following Lemma to leverage the assumption of $P$ being SST or ST$^+$.
\begin{lemma}[Uniform drift induced by off-policy mixing]
\label{lem:delta_positive}
Assume $P$ is SST and suppose $\bpi_0\in\Delta(K)^\circ$. Fix any $\lambda\in[0,1)$ and define
\begin{equation}
\label{eq:delta_i_def}
\delta_i
:=
(1-\lambda)\sum_{j=1}^K \pi_{0,j}\,(P_{ij}-P_{i+1,j}),
\qquad i=1,\dots,K-1.
\end{equation}
Then $\delta_i>0$ for all $i$.
Moreover, for all $t\ge 0$,
\begin{equation}
\label{eq:gap_lower_bound}
(P\bpi_t)_i-(P\bpi_t)_{i+1} \;\ge\; \delta_i,
\qquad
\bpi_t:=\lambda\pi_t+(1-\lambda)\pi_0.
\end{equation}
Moreover, if $P$ is ST$^+$ and not SST, we have $\delta_1>0$ and $\delta_i\ge 0, \forall i\ge 2$.
\end{lemma}

\begin{proof}
Let $d^{(i)}\in\mathbb{R}^K$ be defined by $d^{(i)}_j:=P_{ij}-P_{i+1,j}$.
By SST, $d^{(i)}_j\ge 0$ for all $j$ and there exists $j$ with $d^{(i)}_j>0$.
Since $\bpi_0\in\Delta^\circ$ has $\pi_{0,j}>0$ for all $j$, we have $\langle d^{(i)},\bpi_0\rangle>0$ and thus $\delta_i=(1-\lambda)\langle d^{(i)},\bpi_0\rangle>0$.

Finally,
\[
(P\bpi_t)_i-(P\bpi_t)_{i+1}
=
\sum_{j=1}^K \pi_{t,j}\,d^{(i)}_j
\ge
(1-\lambda)\sum_{j=1}^K \pi_{0,j}\,d^{(i)}_j
=
\delta_i,
\]
proving \eqref{eq:gap_lower_bound}. If $P$ is ST$^+$, the same argument holds for $i=1$.
\end{proof}

Based on \Cref{lem:logratio_recursion_mixed,lem:ratio_to_topmass,lem:delta_positive}, we next prove a full version of \Cref{thm:454}, as shown in the following \Cref{thm:true_vs_eps_collapse_full}.

\begin{theorem}[True collapse for $\alpha=1$ and $\varepsilon$-collapse for $\alpha<1$]
\label{thm:true_vs_eps_collapse_full}
Assume $P$ is ST$^+$. Let $\bpi_0\in\Delta_K^\circ$ and $\bpi_{\rm ref}\in\Delta_K^\circ$. Fix any $\lambda\in[0,1)$, define $\delta_i$ as in \eqref{eq:delta_i_def}. Consider the MRS-IPO dynamics $\bpi_{t+1}=\mathrm{MRS}(\bpi_t;P,\alpha,\beta,\lambda,\bpi_{\mathrm{ref}},\bpi_0)$ given by \eqref{eq:mix_ref_softmax_form}. The following claims hold:

\begin{enumerate}
\item \textbf{(True collapse when $\alpha=1$).}
If $\alpha=1$, then for every $i=1,\dots,K-1$,
\begin{equation}
\label{eq:linear_drift_alpha1}
r_{t,i}\ \ge\ r_{0,i} + t\,\beta\lambda\,\delta_i,
\end{equation}
so $r_{t,1}\to+\infty$ linearly and therefore $\bpi_t\to \e_1$.
Moreover, there exists a finite $T$ such that for all $t\ge T$,
\[
H(\bpi_{t+1})<H(\bpi_t)
\quad\text{unless $\pi_t$ is already a vertex,}
\]
and consequently $H(\bpi_t)\to 0$.

\item \textbf{(No exact collapse when $\alpha<1$).}
If $\alpha\in[0,1)$, then adjacent log-ratios remain uniformly bounded:
\[
|r_{t,i}-r_{{\rm ref},i}|
\le
\alpha^t|r_{0,i}-r_{{\rm ref},i}|+\frac{\beta\lambda}{1-\alpha}
\qquad \forall t\ge 0.
\]
In particular, $\min_i \pi_{t,i}$ is bounded below by a positive constant depending on \\$(K,\alpha,\beta,\lambda,\bpi_{\rm ref},\bpi_0)$, and hence $H(\bpi_t)$ is bounded away from $0$; thus $\bpi_t$ does not converge to any vertex of $\Delta_K$.

\item \textbf{($\varepsilon$-collapse for large $\beta\lambda/(1-\alpha)$).}
If $\alpha\in[0,1)$, there exists a large $T(P,\bpi_0,\bpi_{\rm ref},\alpha,\beta,\gamma)$ such that for every $i$ and all $t\ge T$,
\begin{equation}
\label{eq:lower_bound_r}
r_{t,i}> r_{{\rm ref},i}+\frac{\beta\lambda\delta_i}{2(1-\alpha)}.
\end{equation}
Consequently, the iterates become $\varepsilon$-collapsed for 
\begin{equation}
\varepsilon=
\begin{cases}
    \left(\exp\left(\frac{\beta\lambda\delta}{2(1-\alpha)}\right)-1\right)^{-1}, & \mathrm{if~} P \mathrm{~is~SST},\\
    (K-1)\exp\left(-\frac{\beta\lambda\delta_1}{2(1-\alpha)}\right), & \mathrm{if~} P \mathrm{~is~ST}^+,
\end{cases}
\end{equation}
where $\delta\coloneqq \min_i \delta_i>0$.
\end{enumerate}
\end{theorem}

\begin{proof}
\begin{enumerate}
    \item Claim-1: We first prove $\alpha=1$ implies true collapse.
When $\alpha=1$, the recursion \eqref{eq:r_recursion} becomes
\[
r_{t+1,i}=r_{t,i}+\beta\lambda\big((P\bpi_t)_i-(P\bpi_t)_{i+1}\big).
\]
By Lemma~\ref{lem:delta_positive}, $(P\bpi_t)_i-(P\bpi_t)_{i+1}\ge\delta_i$, hence
$r_{t+1,i}\ge r_{t,i}+\beta\lambda\delta_i$ and \eqref{eq:linear_drift_alpha1} follows by induction. For ST$^+$ $P$, $\delta_1>0$ and therefore $r_{t,1}\to+\infty$. In addition, $r_{t,i}$ is lower bounded by a constant for all $i\ge 2$, which forces $\bpi_t\to \e_1$.

To obtain strict entropy decrease after some finite time, note that $r_{t,1}\to+\infty$ and $r_{t,i}\ge C, i\ge 2$ imply that for all sufficiently large $t$, $\pi_{t,1}\ge\pi_{t,i}, \forall i\ge 2$. Also under the ST$^+$ assumption, the score vector $\s_t:=P\bpi_t$ is also aligned for every $t$:
\[
s_{t,i}-s_{t,i+1}
=
(P\bpi_t)_i-(P\bpi_t)_{i+1}
\ge \delta_i\ge 0,
\]
and $\delta_1>0$. Therefore the weights $w_{t,i}:=\exp(\beta s_{t,i})$ satisfy $w_{t,1}>w_{t,1}\ge \cdots\ge w_{t,K}$.
When $\alpha=1$, the update is a monotone exponential tilt $\bpi_{t+1}\propto \bpi_t\odot \w_t$.
Standard majorization arguments for monotone tilting imply $\bpi_{t+1}$ majorizes $ \bpi_t$ for all large $t$,
and since Shannon entropy is strictly Schur-concave, $H(\bpi_{t+1})<H(\bpi_t)$ unless $\bpi_t$ is a vertex.
Finally, $\bpi_t\to \e_1$ implies $H(\bpi_t)\to 0$.

\item $\alpha<1$ prevents exact collapse.
From \eqref{eq:r_recursion} and the bound $|(P\bpi_t)_i-(P\bpi_t)_{i+1}|\le 1$ (since $P\in\calP\in[0,1]^{K\times K}$),
unrolling this recursion we have 
\begin{align*}
|r_{t+1,i}-r_{{\rm ref},i}|
&=
\left|\alpha(r_{t,i}-r_{{\rm ref},i})+\beta\lambda\big((P\bpi_t)_i-(P\bpi_t)_{i+1}\big)\right|
\\ & \le
\alpha|r_{t,i}-r_{{\rm ref},i}|+\beta\lambda \\ 
&\le \alpha(\alpha|r_{t-1,i}-r_{{\rm ref},i}|+\beta\lambda)+\beta\lambda \\ 
& \le \cdots \\ 
& \le \alpha^t|r_{1,i}-r_{{\rm ref},i}|+\beta\lambda (1+\alpha+\cdots+\alpha^{t-1}),
\end{align*}
which implies
$|r_{t,i}-r_{{\rm ref},i}|\le \alpha^t|r_{0,i}-r_{{\rm ref},i}|+\beta\lambda/(1-\alpha)$.
Hence all adjacent log-ratios are uniformly bounded, which implies that all coordinates of $\bpi_t$ remain bounded away from $0$,
so $\bpi_t$ cannot converge to a vertex and $H(\bpi_t)$ is bounded below by a positive constant.
\item Finally we prove $\varepsilon$-collapse for large $\beta\lambda/(1-\alpha)$. By Lemma~\ref{lem:delta_positive}, $(P\bpi_t)_i-(P\bpi_t)_{i+1}\ge \delta_i$.
Substitute this into \eqref{eq:r_recursion} to obtain the lower bound recursion
\begin{align}\notag
r_{t+1,i}\ &\ge\ \alpha r_{t,i}+(1-\alpha)r_{{\rm ref},i}+\beta\lambda\delta_i \\ &>\alpha r_{t,i}+(1-\alpha)r_{{\rm ref},i}+\beta\lambda\delta_i/2+\beta\lambda\delta_i/2 \\ \notag
&>\alpha^2 r_{t-1,i}+(1+\alpha)((1-\alpha)r_{{\rm ref},i}+\beta\lambda\delta_i/2)+\beta\lambda\delta_i/2\\ \notag
&>\cdots\\
&> \alpha^t r_{1,i}+(1+\alpha+\cdots+\alpha^{t-1})((1-\alpha)r_{{\rm ref},i}+\beta\lambda\delta_i/2)+\beta\lambda\delta_i/2 \\ \label{eq:669}
&\ge (1-\alpha^t)\left(r_{{\rm ref},i}+\frac{\beta\lambda\delta_i}{2(1-\alpha)}\right)+\beta\lambda\delta_i/2.
\end{align}
Now choose a sufficiently large $T=T(r_{{\rm ref},i},\alpha, \beta,\lambda,\delta_i)$ such that $\beta\lambda\delta_i/2\ge \alpha^t \left(r_{{\rm ref},i}+\frac{\beta\lambda\delta_i}{2(1-\alpha)}\right)$, from \eqref{eq:669} we have 
$$r_{t+1,i}> r_{{\rm ref},i}+\frac{\beta\lambda\delta_i}{2(1-\alpha)}.
$$
When $P$ is SST, $\frac{\beta\lambda\delta_i}{2(1-\alpha)}$ is strictly positive for all $i\in[k]$ and therefore Lemma~\ref{lem:ratio_to_topmass} implies $\pi_{t,1}\ge 1-\varepsilon$ for all $t\ge T$, given $\varepsilon=\left(\exp\left(\frac{\beta\lambda\delta}{2(1-\alpha)}\right)-1\right)^{-1}$. When $P$ is ST$^+$, $\frac{\beta\lambda\delta_i}{2(1-\alpha)}$ is strictly positive for $i=1$ and nonnegative for $i\ge 2$, Lemma~\ref{lem:ratio_to_topmass} gives $\varepsilon=(K-1)\exp\left(-\frac{\beta\lambda\delta}{2(1-\alpha)}\right)$.

\end{enumerate}

\end{proof}

\section{Additional Discussion of DPO}\label{app:offline_dpo}

In this section, we now present a parallel set of theoretical results for DPO. In contrast to IPO, DPO generally does not admit a closed-form expression for the optimal policy $\bpi^\star$; instead, $\bpi^\star$ is characterized implicitly through a first-order optimality condition (FOC) given by the following \Cref{thm:exist_unique_DPO}. Such an implicit characterization makes extending the IPO analysis to DPO substantially more challenging. Nevertheless, in this section, we develop DPO analogues of the main paper IPO results wherever possible, noting that some statements must be presented in a weaker form due to additional technical obstacles.

\subsection{DPO optima under general sampling distributions}\label{app:dpo_solution}
Similar to Proposition~\ref{prop:exist_unique_IPO}, which characterizes the optimizer of the IPO population loss, we have the existence and uniqueness result for the DPO solution as follows:
\begin{proposition}[Existence and Uniqueness of the DPO Optimum]\label{thm:exist_unique_DPO}
Given preference matrix $P\in\calP$, the DPO problem with sampling strategy $\bmu\in\Delta_K^\circ$ and reference policy $\bpi_{\mathrm{ref}}\in\Delta_K^\circ$
is equivalent to solving the convex optimization problem 
\begin{equation}\label{eq:DPO_convex_op}
    \min_{\bpi\in\Delta_K}
    \left[
    -\sum_{1\le i,j\le K}
    \mu_i\mu_j\,P_{ij}\,
    \log \sigma\!\left(
    \beta^{-1}\log\frac{\pi_i}{\pi_{\mathrm{ref},i}}
    -
    \beta^{-1}\log\frac{\pi_j}{\pi_{\mathrm{ref},j}}
    \right)
    \right],
\end{equation}

which admits a unique solution $\bpi^{\star}\in\Delta_K$. Moreover,
let $\btheta^{\star}=\beta^{-1}(\log \bpi^{\star}-\log \bpi_{\mathrm{ref}})$ and let $Q$ be the Bradley--Terry matrix induced by $\btheta^{\star}$, i.e.\ $Q_{ij}=\sigma(\theta_i^{\star}-\theta_j^{\star})$, then the first-order optimality condition for its solution $\bpi^{\star}$ is
\begin{equation}\label{eq:first-order-Q}
    \operatorname{Diag}(\bmu)\,(P-Q)^{\top}\bmu = \bm{0}.
\end{equation}
\end{proposition}

\begin{proof}

Define the reparametrization
\[
\theta_i \;=\;\beta^{-1}\Bigl(\log \pi_i-\log \pi_{\mathrm{ref},i}\Bigr),
\qquad i\in[K].
\]
Then $\pi_i=\pi_{\mathrm{ref},i}e^{\beta\theta_i}$ and the DPO objective becomes
\[
L(\btheta)\;=\;
-\sum_{i,j=1}^K \mu_i\mu_j\,P_{ij}\,\log\sigma(\theta_i-\theta_j),
\]the simplex constraint
$\sum_i \pi_i=1$ is equivalent to the smooth equality constraint
\begin{equation}\label{eq:theta_constraint}
\sum_{i=1}^K \pi_{\mathrm{ref},i}e^{\beta\theta_i}=1.
\end{equation}

First of all, we show the convexity of $L(\btheta)$. Let $\ell(x) := -\log \sigma(x) = \log(1+e^{-x})$ which satisfies $\ell''(x)=\sigma(x)(1-\sigma(x))>0, \forall x\in\bbR$
and is therefore strictly convex. The objective in~\eqref{eq:DPO_convex_op} can be written as $L(\btheta)=-\sum_{i,j=1}^K \mu_i\mu_j P_{ij}\,\ell(\theta_i-\theta_j)$, with the understanding that $\btheta$ is defined up to an additive constant
(since $\bpi(\btheta)=\bpi(\btheta+c\mathbf 1)$ for any $c\in\bbR$). Thus optimizing over $\bpi\in\Delta_K$ is equivalent to optimizing over a gauge-fixed affine subspace, e.g.\ $\{\btheta:\sum_i\theta_i=0\}$. Each term $\ell(\theta_i-\theta_j)$ is the composition of a strictly convex function
with a linear map, hence is convex in $\btheta$.
Since $\mu_i\mu_jP_{ij}\ge 0$, $\mathcal{L}(\theta)$ is a nonnegative weighted sum of convex functions
and is therefore convex.

In addition, we show strict convexity of $L$ modulo the constant-shift invariance by computing its Hessian. Using $\nabla(\theta_i-\theta_j)=\e_i-\e_j$, we have
\[
\nabla^2 \ell(\theta_i-\theta_j)
=
\ell''(\theta_i-\theta_j)\,(\e_i-\e_j)(\e_i-\e_j)^\top.
\]
Hence
\begin{equation}
\label{eq:dpo_hessian}
\nabla^2 L(\btheta)
=
\sum_{i,j=1}^K
\mu_i\mu_j P_{ij}\,\ell''(\theta_i-\theta_j)\,(\e_i-\e_j)(\e_i-\e_j)^\top.
\end{equation}
Using the dueling identity $P_{ij}+P_{ji}=1$ and the symmetry $\ell''(-x)=\ell''(x)$,
we may group terms $(i,j)$ and $(j,i)$ to obtain
\begin{equation}
\label{eq:dpo_hessian_laplacian}
\nabla^2 L(\btheta)
=
\sum_{1\le i<j\le K}
\mu_i\mu_j\,\ell''(\theta_i-\theta_j)\,(\e_i-\e_j)(\e_i-\e_j)^\top.
\end{equation}

For any $\v\in\bbR^K$,
\[
\v^\top \nabla^2 L(\btheta)\,\v
=
\sum_{i<j}
\mu_i\mu_j\,\ell''(\theta_i-\theta_j)\,(v_i-v_j)^2.
\]
If $\bmu$ has full support, then each coefficient
$\mu_i\mu_j\,\ell''(\theta_i-\theta_j)$ is strictly positive. Therefore,
\[
\v^\top \nabla^2 L(\btheta)\,\v = 0
\quad\Longleftrightarrow\quad
v_i=v_j\ \ \forall i,j,
\]
i.e.\ $\v\in\mathrm{span}\{\mathbf 1\}$.
Thus the Hessian is positive semidefinite with nullspace exactly equal to the constant-shift direction, and $L(\btheta)$ is strictly convex on any gauge-fixed subspace
(e.g.\ $\sum_i\theta_i=0$ or $\theta_K=0$). Hence the original OP~\eqref{eq:DPO_convex_op} is strictly convex over the reparameterized space $\{\btheta\in\mathbb{R}^K, \sum_{i\in[K]}\theta_i=0\}$,
with a unique optimizer $\btheta^\star$.

Finally, since the mapping
\[
\btheta \;\longleftrightarrow\; \pi_i=\frac{\pi_{\mathrm{ref},i}e^{\beta\theta_i}}
{\sum_k \pi_{\mathrm{ref},k}e^{\beta\theta_k}}
\]
is a bijection between the gauge-fixed $\btheta$-space $ \{\btheta\in\mathbb{R}^K, \sum_{i\in[K]}\theta_i=0\}$ and $\Delta_K^\circ$, we conclude that OP~\ref{eq:DPO_convex_op} admits a unique optimum $\bpi^{\star}\in \Delta_K$.

Now we derive the first-order condition for $\bpi^{\star}$. Note that solving OP \eqref{eq:DPO_convex_op} is equivalent to minimizing $L(\btheta)$ subject to
\eqref{eq:theta_constraint}, and we let $\btheta^\star=\beta^{-1}(\log\bpi^\star-\log\bpi_{\mathrm{ref}})$
and define the Bradley--Terry matrix $Q$ by
\[
Q_{ij}\;=\;\sigma(\theta_i^\star-\theta_j^\star),\qquad i,j\in[K].
\]

For the first step, we derive the KKT condition, which reduces to $\nabla F(\btheta^\star)=\bm 0$. Consider the Lagrangian
\[
\mathcal{L}(\btheta,\lambda)\;=\;L(\btheta)
+\lambda\!\left(\sum_{i=1}^K \pi_{\mathrm{ref},i}e^{\beta\theta_i}-1\right).
\]
At an optimal solution $(\btheta^\star,\lambda^\star)$, the KKT stationarity
conditions read
\begin{equation}\label{eq:KKT_stationarity}
\nabla L(\btheta^\star) + \lambda^\star \nabla\!\left(\sum_{i=1}^K \pi_{\mathrm{ref},i}e^{\beta\theta_i}\right)\Big|_{\btheta=\btheta^\star}
=\bm 0.
\end{equation}
Since $L(\btheta)$ depends only on differences $\theta_i-\theta_j$, it is invariant
to common shifts: $L(\btheta+c\bm 1)=L(\btheta)$ for all $c\in\mathbb{R}$. Differentiating
with respect to $c$ at $c=0$ gives
\begin{equation}\label{eq:shift_invariance}
\bm 1^\top \nabla L(\btheta)=0\qquad\text{for all }\btheta.
\end{equation}
On the other hand,
\[
\nabla\!\left(\sum_{i=1}^K \pi_{\mathrm{ref},i}e^{\beta\theta_i}\right)\Big|_{\btheta=\btheta^\star}
= \beta\,\bpi^\star,
\]
and $\bm 1^\top (\beta\bpi^\star)=\beta>0$.
Taking the inner product of \eqref{eq:KKT_stationarity} with $\bm 1$ and using
\eqref{eq:shift_invariance} yields $\lambda^\star\beta=0$, hence $\lambda^\star=0$ and therefore
\begin{equation}\label{eq:grad_zero}
\nabla L(\btheta^\star)=\bm 0.
\end{equation}

Next we compute $\nabla L(\btheta)$ and identify the matrix form. Recall that $\frac{d}{dx}\log\sigma(x)=\sigma(-x)$. For each $k\in[K]$,
only the terms with $i=k$ or $j=k$ depend on $\theta_k$, so
\begin{align*}
\frac{\partial L}{\partial \theta_k}(\btheta)
&=
-\sum_{j=1}^K \mu_k\mu_j\,P_{kj}\,\sigma\!\bigl(-(\theta_k-\theta_j)\bigr)
+\sum_{i=1}^K \mu_i\mu_k\,P_{ik}\,\sigma\!\bigl(-(\theta_i-\theta_k)\bigr)\\
&=
-\sum_{j=1}^K \mu_k\mu_j\,P_{kj}\,\sigma(\theta_j-\theta_k)
+\sum_{i=1}^K \mu_i\mu_k\,P_{ik}\,\sigma(\theta_k-\theta_i).
\end{align*}
Evaluating at $\btheta^\star$ and using $Q_{jk}=\sigma(\theta_j^\star-\theta_k^\star)$ and
$Q_{ki}=\sigma(\theta_k^\star-\theta_i^\star)$ gives
\[
\frac{\partial L}{\partial \theta_k}(\btheta^\star)
=
-\sum_{j=1}^K \mu_k\mu_j\,P_{kj}\,Q_{jk}
+\sum_{i=1}^K \mu_i\mu_k\,P_{ik}\,Q_{ki}.
\]
Now use the defining property of preference matrices $P\in\calP$, namely
$P_{ik}=1-P_{ki}$, together with the Bradley--Terry symmetry $Q_{ki}=1-Q_{ik}$.
Then, for each $i$,
\[
P_{ik}Q_{ki}-P_{ki}Q_{ik}
=(1-P_{ki})(1-Q_{ik})-P_{ki}Q_{ik}
=1-P_{ki}-Q_{ik}
=P_{ik}-Q_{ik}.
\]
Hence

\[
\frac{\partial L}{\partial \theta_k}(\btheta^\star)
=
\mu_k\sum_{i=1}^K \mu_i\,(P_{ik}-Q_{ik}).
\]
Stacking these $K$ equalities gives the vector identity
\[
\nabla L(\btheta^\star)
=\operatorname{Diag}(\bmu)\,(P-Q)^\top \bmu.
\]

Combining the above expression with \eqref{eq:grad_zero} yields
\[
\operatorname{Diag}(\bmu)\,(P-Q)^\top \bmu=\bm 0,
\]
which is exactly \eqref{eq:first-order-Q}.
\end{proof}

\subsection{The dependency of DPO solutions on sampling}\label{app:dpo_solution_sampling}

To study how DPO solution depends on sampling strategy, we assume the reference policy in the DPO population loss $\bpi_{\rm ref}=\bmu_0$ be the uniform distribution on $\Delta_K$ without loss of generality and let $\bmu$ have full support. Then, the first-order condition (FOC) \eqref{eq:first-order-Q} given in \Cref{thm:exist_unique_DPO} can be further simplified as \begin{equation}\label{eq:first-order-dpo}
F(\btheta^{\star},\bmu):=(P-Q(\btheta^{\star}))^{\top}\bmu = \bm{0}.
\end{equation}

The new FOC \eqref{eq:first-order-dpo} gives an implicit mapping from any sampling strategy $\bmu\in \Delta_K^{\circ}$ to the unique DPO solution $\bpi^{\star}(P;\bmu)$ by specifying its logits $\btheta^{\star}(P;\bmu)=\beta^{-1}\log \bpi^{\star}(P;\bmu)$ up to a constant shift. The following result shows that the unique solution is invariant to sampling when and only when the underlying preference structure $P$ is derived from a BT model.

\begin{proposition}[Sampling Invariance Occurs \emph{Only} Under True Bradley--Terry Preferences]\label{prop:sample_invariant}
Let $P\in\calP$ be an arbitrary preference matrix. The solution $\bpi^{\star}(P;\bmu)$ of OP~\eqref{eq:DPO_convex_op} is independent of the sampling strategy $\bmu$ (i.e., $\bpi^{\star}(P;\bmu)$ is a constant over all $\bmu\in\Delta_K$ with full support) \emph{if and only if} $P$ is exactly representable as a Bradley--Terry model, i.e., 
\[
P_{ij}=\sigma(r_i-r_j)
\quad\text{for some } r\in\mathbb{R}^K.
\]
\end{proposition}

Proposition~\ref{prop:sample_invariant} highlights a crucial but underappreciated aspect of RLHF-style preference learning:
\begin{itemize}
    \item If the true pairwise preferences \emph{do not} come from a BT model (which is almost always the case), then \emph{DPO does not recover an intrinsic ``ground truth" alignment policy}.  
    Instead, it recovers a policy that also depends on the sampling strategy. Such a message is self-evident for IPO due to the weighted square-sum format of its loss function, but is not explicitly pointed out for DPO to the best of our knowledge.
    \item The only sampling-invariant situation is the perfectly idealized world where human preferences admit a single global utility function $r_i$, and where $P_{ij}$ depends solely on $r_i-r_j$.  
    This assumption is extremely restrictive: real-world preference judgments violate transitivity, context independence, and independence of irrelevant alternatives---all incompatible with the BT model.
\end{itemize}

Thus, the sampling policy is not merely a data-collection choice: it acts as an implicit regularizer that selects \emph{which} among many possible BT approximations the DPO objective converges to. This has an important implication:
\begin{quote}
\emph{Even in the infinite-data limit, DPO learns a policy determined jointly by the preference structure $P$ and the sampling distribution $\bmu$.}
\end{quote}
In other words, “more data” does not wash out the effect of sampling; rather, the sampling strategy systematically biases the learned policy. 

\begin{proof}
We prove the ``if'' and ``only if'' directions of the claim separately.

\textbf{(If).} Suppose $P$ is BT-consistent, i.e., $\exists \bar{\btheta}$ such that $P_{ij}=\sigma(\bar{\theta}_i-\bar{\theta}_j), \forall i\neq j$. Then for any $\bmu$, the FOC \eqref{eq:first-order-dpo} gives 
\begin{equation}
    F(\btheta,\bmu)=\sum_{i<j} \mu_i\mu_j (\e_i-\e_j)\left(P_{ij}-\sigma(\bar{\theta}_i-\bar{\theta}_j)\right)=\sum_{i<j} \mu_i\mu_j (\e_i-\e_j)\left(\sigma(\bar{\theta}_i-\bar{\theta}_j)-\sigma(\bar{\theta}_i-\bar{\theta}_j)\right)=0.
\end{equation}

And it is obvious that $\btheta=\bar{\btheta}$ is a stationary point of $F(\btheta,\bmu)=0$. By strict convexity of the OP \eqref{eq:DPO_convex_op} on $\Delta_K$, the stationary point of $F(\btheta,\bmu)=0$ is unique and corresponds to the global minimizer. Therefore, the if direction holds. 

\textbf{(Only If).} Conversely, assume there exists a $\bpi^{\star}$ such that $\bpi^{\star}(P;\bmu)=\bpi^{\star}$ for every $\bmu$. Let $\bar{\btheta}$ be any logit of $\bpi^{\star}$, then $F(\bar{\btheta},\bmu)=0$ holds for all $\bmu$. Pick any pair $i\neq j$ and construct a distribution $\bmu=\bmu_{\varepsilon}^{(i,j)}$ with probability mass concentrating on $i,j$, i.e., $\mu_i=\mu_j=\frac{1}{2}(1-\varepsilon),\mu_k=\frac{\varepsilon}{K-2}, \forall k\neq i,j$, and let $\varepsilon\rightarrow 0$. The continuity of $F$ gives 
\begin{equation}
    \bm{0}=\lim_{\varepsilon\rightarrow 0^+} F(\bar{\btheta},\bmu_{\varepsilon}^{(i,j)}) = \frac{1}{4}(P_{ij}-\sigma(\bar{\theta}_i-\bar{\theta}_j))(\e_i-\e_j).
\end{equation}
Therefore, $P_{ij}=\sigma(\bar{\theta}_i-\bar{\theta}_j)$ must hold for pair $(i,j)$. Since $(i,j)$ was arbitrarily chosen, $P$ must be BT-consistent with $\bpi^{\star}$.

\end{proof}
\subsection{How offline sampling influences the concentration of DPO solutions}
In this section, we present a result in parallel to \Cref{thm:pairwise_gap_ipo} that shows the same message: a more skewed sampling distribution $\bmu'$ that are perturbed in the direction aligning with the original order of $\bpi$ results in more concentrated learned policy. 

\begin{proposition}[Effect of Sampling on Pairwise Gaps of DPO Solution]\label{prop:pairwise_gap_dpo} For any fixed $P\in\calP$, let $\bpi(\bmu)$ be the solution of OP~\eqref{eq:DPO_convex_op} under sampling distribution $\bmu\in\Delta_k^{\circ}$, and, without loss of generality, assume that the corresponding logits satisfy $\frac{\bpi(\bmu)}{\bpi_{\mathrm{ref}}}
\;\text{ is in descending order, i.e., }\;
\theta_1>\theta_2>\cdots>\theta_K$, where $\btheta(\bmu)=\beta^{-1}(\log\bpi(\bmu)-\log\bpi_{\mathrm{ref}})$.
Consider a perturbed sampling strategy of the following form
\[
\bmu'=\bmu+\delta\bmu_0
\]
for some small $\delta>0$. Let $\bpi'=\bpi(\bmu')$ be the solution of OP~\eqref{eq:DPO_convex_op} under $\bmu'$ and $\btheta'$ be the corresponding logit. Then we have

\begin{enumerate}
    \item If the perturbation direction is $\bmu_0 = 1/\bmu$ (elementwise), it holds that
    \begin{equation}\label{eq:400}
        \sum_{i<j}\frac{P_{ji}}{\mu_i\mu_j}(\theta_i-\theta_j) >
        \sum_{i<j}\frac{P_{ji}}{\mu_i\mu_j}(\theta'_i-\theta'_j).
    \end{equation}

    \item If the perturbation direction is $\bmu_0 = \e_K$, it holds that
    \begin{equation}\label{eq:406}
        \sum_{i\neq K}\frac{P_{Ki}}{\mu_i}(\theta_i-\theta_K)>
        \sum_{i\neq K}\frac{P_{Ki}}{\mu_i}(\theta'_i-\theta'_K).
    \end{equation}
\end{enumerate}
\end{proposition}

\paragraph{Remarks.}
The quantities in \eqref{eq:400} and \eqref{eq:406} can be interpreted as \emph{weighted averages of pairwise logit gaps} under the learned policy. \eqref{eq:400} aggregates pairwise gaps across all response pairs, while \eqref{eq:406} focuses on gaps between the lowest-scored response and the rest. Larger values of these metrics correspond to more \emph{concentrated} policies where probability mass is more skewed toward a subset of high-scoring responses.

Proposition \ref{prop:pairwise_gap_dpo} shows that when the sampling distribution is perturbed
\begin{itemize}
    \item toward a more uniform distribution (via $\bmu_0 = 1/\bmu$), or
    \item toward increased exploration of the currently worst-scored response (via $\bmu_0 = \e_K$),
\end{itemize}
the corresponding weighted pairwise logit gaps must decrease.
In other words, these perturbations make the DPO solution less extreme, redistributing probability mass more evenly across responses and thereby increasing policy diversity.

Compared to the sharp, pairwise-specific characterizations obtained for IPO in Theorem~\ref{thm:pairwise_gap_ipo} and Corollary~\ref{cor:bt_headtail_maj}, our counterpart result on DPO only characterizes directional changes in aggregated gap metrics under small local sampling perturbations, rather than providing a global or itemwise guarantee. We note that this gap reflects intrinsic technical difficulties in the DPO analysis rather than a fundamental difference in behavior. Indeed, despite the weaker theoretical form, Proposition~\ref{prop:pairwise_gap_dpo} conveys the same qualitative message: \emph{skewed sampling amplifies policy concentration, while more uniform or exploratory sampling mitigates it}. We further demonstrate in our experiments that enlarged pairwise logit gaps under skewed sampling arise for DPO across a much broader range of regimes than those covered by the theory.

For the rest of this section, we present a full version of \Cref{prop:pairwise_gap_dpo} in the following Theorem, explicitly characterizing how the weighted pairwise logit gaps depend on the environment parameters and provide its proof. 

\begin{theorem}[Effect of Sampling on Pairwise Gaps of DPO Solution]\label{thm:pairwise_gap_dpo}
For any fixed $P\in\calP$, let $\bpi(\bmu)$ be the solution of OP~\eqref{eq:DPO_convex_op} under sampling distribution $\bmu\in\Delta_k^{\circ}$, and, without loss of generality, assume that the corresponding logits satisfy $\frac{\bpi(\bmu)}{\bpi_{\mathrm{ref}}}
\;\text{ is in descending order, i.e., }\;
\theta_1>\theta_2>\cdots>\theta_K$, where $\btheta(\bmu)=\beta^{-1}(\log\bpi(\bmu)-\log\bpi_{\mathrm{ref}})$.
Consider a perturbed sampling strategy of the following form
\[
\bmu'=\bmu+\delta\bmu_0
\]
for some small $\delta>0$. Let $\bpi'=\bpi(\bmu')$ be the solution of OP~\eqref{eq:DPO_convex_op} under $\bmu'$ and $\btheta'$ be the corresponding logit. Then we have

\begin{enumerate}
    \item If the perturbation direction is $\bmu_0 = 1/\bmu$ (elementwise), then the functional
    \begin{equation}
        G^{(1)}(\bpi;\bmu)
        \;=\;
        \sum_{i<j}\frac{P_{ji}}{\mu_i\mu_j}\log\frac{\pi_i}{\pi_j}
    \end{equation}
    strictly decreases when we move from $\bpi$ to $\bpi'$. Specifically, we have
    \[
    G^{(1)}(\bpi;\bmu) - G^{(1)}(\bpi';\bmu)> \beta\Delta_1(K-1)^{1.5}\delta - O\left(\beta\sum_{i<j}e^{^{-(\theta_i-\theta_j)}}\right),
    \]
    where $\Delta_1=\min_{1\le i\le K-1}\left(\sum_{j\neq i}\frac{P_{ij}-\sigma(\theta_i-\theta_j)}{\mu_j}\right)^2$ is a strictly positive constant.

    \item If the perturbation direction is $\bmu_0 = \e_K$ (only the $K$-th coordinate increases), then the functional
    \begin{equation}
        G^{(2)}(\bpi;\bmu)
        \;=\;
        \sum_{i\neq K}\frac{P_{Ki}}{\mu_i}\log\frac{\pi_i}{\pi_K}
    \end{equation}
    strictly decreases, i.e.,
    \[
    G^{(2)}(\bpi;\bmu) - G^{(2)}(\bpi';\bmu)>\beta\Delta_2(K-1)^{1.5}\delta - O\left(\beta\sum_{i<j}e^{^{-(\theta_i-\theta_j)}}\right),
    \]
    where $\Delta_2=\min_{1\le i\le K-1}\left(P_{iK}-\sigma(\theta_i-\theta_K)\right)^2$ is a strictly positive constant.
\end{enumerate}
\end{theorem}

\begin{proof} Let $\btheta(\bmu)=\beta^{-1}(\log \bpi(\bmu)-\log \bpi_{\mathrm{ref}})$. According to our assumption, $\theta_1>\theta_2>\cdots>\theta_K$. By definition, $\btheta$ is the solution of the following OP
\begin{equation}\label{eq:DPO_convex_op_theta}
    \min_{\btheta}\left[-\sum_{1\le i,j\le K}\mu_i\mu_j\,P_{ij}\,\log\sigma\!\left(\theta_i-\theta_j\right)\right],
\end{equation}
and according to \Cref{thm:exist_unique_DPO}, it satisfies the first-order condition (FOC) \begin{equation}\label{eq:foc2}
    F(\btheta,\bmu):=(P-Q(\btheta))\bmu=\bm{0},
\end{equation}
where $Q(\btheta)$ is the BT matrix induced by $\btheta$, i.e., its $i,j$-th element is $\sigma_{ij}=\sigma(\theta_i-\theta_j)$.

The implicit mapping determined by \eqref{eq:foc2} is a bijection if we restrict one degree of freedom of $\btheta$ by letting $\theta_K=0$. To simplify our notation, we use $\tilde{\x}$ to denote the first $K-1$ element of any vector $\x$ and $\tilde{A}$ to denote the $(K-1)$-by-$(K-1)$ leading principal submatrix of any matrix $A$. Let matrix 
\begin{equation}
   M= \begin{bmatrix}
\sum_{i\neq 1}\mu_i\sigma_{1i}(1-\sigma_{1i}) & -\mu_2\sigma_{12}(1-\sigma_{12}) & \cdots & -\mu_K\sigma_{1K}(1-\sigma_{1K}) \\
-\mu_1\sigma_{21}(1-\sigma_{21}) & \sum_{i\neq 2}\mu_i\sigma_{2i}(1-\sigma_{2i}) & \cdots & -\mu_K\sigma_{2K}(1-\sigma_{2K}) \\
\vdots & \vdots &  & \vdots \\ 
-\mu_1\sigma_{K1}(1-\sigma_{K1}) & -\mu_2\sigma_{K2}(1-\sigma_{K2}) & \cdots & \sum_{i\neq K}\mu_i\sigma_{Ki}(1-\sigma_{Ki})
\end{bmatrix}.
\end{equation}

By implicit function theorem, it holds that 
\begin{equation}
\frac{\partial \tilde{\btheta}}{\partial \bmu}= -\left(\frac{\partial\tilde{F}}{\partial\tilde{\btheta}}\right)^{-1}\cdot\frac{\partial\tilde{F}}{\partial\bmu}=\tilde{M}^{-1}(P-Q(\btheta))_{1:(K-1),1:K}
\end{equation}

For any function $G(\tilde{\btheta}(\bmu))$, we can derive its derivative w.r.t. $\bmu$ using chain rule as the following: 
\begin{align*}
    \frac{\partial G}{\partial \bmu}=\left(\frac{\partial \tilde{\btheta}}{\partial \bmu}\right)^{\top}\cdot \frac{\partial G}{\partial \tilde{\btheta}},
\end{align*}
and therefore if $\bmu$ is perturbed by a small $\delta \bmu_0$, the change of $G$ is 

\begin{align}\notag
    d G &= \delta \bmu_0^{\top} \cdot \left(\frac{\partial \tilde{\btheta}}{\partial \bmu}\right)^{\top}\cdot \frac{\partial G}{\partial \tilde{\btheta}} \\ \notag
    &=\delta [(P-Q(\btheta))_{1:(K-1),1:K}\bmu_0]^{\top}\tilde{M}^{-\top}\frac{\partial G}{\partial \tilde{\btheta}} \\ \label{eq:quadratic_dG}
&=\delta [(P-Q(\btheta))_{1:(K-1),1:K}\bmu_0]^{\top} \tilde{D}^{-\frac{1}{2}}\tilde{\Sigma}^{-\top}(\bpi,\bmu)\tilde{D}^{\frac{1}{2}}\frac{\partial G}{\partial \tilde{\btheta}},
\end{align}
where 
\begin{align*}
\Sigma(\bpi,\bmu)=\begin{bmatrix}
\sum_{i\neq 1}\mu_i\sigma_{1i}(1-\sigma_{1i}) & -\sqrt{\mu_1\mu_2}\sigma_{12}(1-\sigma_{12}) & \cdots & -\sqrt{\mu_1\mu_K}\sigma_{1K}(1-\sigma_{1K}) \\
-\sqrt{\mu_2\mu_1}\sigma_{21}(1-\sigma_{21}) & \sum_{i\neq 2}\mu_i\sigma_{2i}(1-\sigma_{2i}) & \cdots & -\sqrt{\mu_2\mu_K}\sigma_{2K}(1-\sigma_{2K}) \\
\vdots & \vdots &  & \vdots \\ 
-\sqrt{\mu_K\mu_1}\sigma_{K1}(1-\sigma_{K1}) & -\sqrt{\mu_K\mu_2}\sigma_{K2}(1-\sigma_{K2}) & \cdots & \sum_{i\neq K}\mu_i\sigma_{Ki}(1-\sigma_{Ki})
\end{bmatrix},
\end{align*}
and $D=\text{Diag}(\mu_1,\cdots,\mu_K)$. $\tilde{\Sigma}=\Sigma_{1:(K-1),1:(K-1)}$ and $\tilde{D}=\text{Diag}(\mu_1,\cdots,\mu_{K-1})$. To proceed our proof, we need the following Lemma:

\begin{lemma}[Spectral bound for the Fisher--Laplacian]
\label{lem:spectral-bound-Sigma}
For any full--support $\bmu$ and any $\btheta$, the leading principal submatrix $\tilde\Sigma$ is symmetric positive definite and satisfies
\begin{equation}
\lambda_{\max}(\tilde\Sigma)\;\le\;\frac14 + \frac{\sqrt{K-1}}{2}.
\end{equation}
\end{lemma}

\begin{proof}[Proof of Lemma~\ref{lem:spectral-bound-Sigma}]
By construction, $\tilde\Sigma$ is symmetric and strictly Diagonally dominant with nonpositive off--Diagonal entries, so it is positive definite. Next, we upper bound its largest eigenvalue. For its Diagonal entries we have
\begin{equation}
\tilde\Sigma_{ii}
=
\sum_{j\neq i}
\mu_j\,\sigma_{ij}(1-\sigma_{ij})
\;\le\;
\frac14\sum_{j\neq i}\mu_j
\;\le\;
\frac14,
\end{equation}
because $0\le \sigma_{ij}(1-\sigma_{ij})\le \tfrac14$ and $\sum_{j\neq i}\mu_j\le 1$.

For the off--Diagonal entries,
\begin{equation}
|\tilde\Sigma_{ij}|
=
\sqrt{\mu_i\mu_j}\,\sigma_{ij}(1-\sigma_{ij})
\;\le\;
\frac14\sqrt{\mu_i\mu_j}.
\end{equation}
By Gershgorin's circle theorem,
\begin{equation}
\lambda_{\max}(\tilde\Sigma)
\le
\max_{1\le i\le K-1}\Bigl(
\tilde\Sigma_{ii} + \sum_{j\neq i}|\tilde\Sigma_{ij}|
\Bigr)
\le
\frac14 + \frac14\sum_{j\neq i}\sqrt{\mu_i\mu_j}.
\end{equation}
Applying Cauchy--Schwarz,
\begin{equation}
\sum_{j\neq i}\sqrt{\mu_i\mu_j}
\le
\sqrt{K-1}\sqrt{\mu_i(\sum_{j\ne i}\mu_j)}=\sqrt{K-1}\sqrt{\mu_i(1-\mu_i)}
\;\le\; \frac{\sqrt{K-1}}{2}.
\end{equation}
Thus
\begin{equation}
\lambda_{\max}(\tilde\Sigma)
\le
\frac14 + \frac{\sqrt{K-1}}{2},
\end{equation}
as claimed.
\end{proof}

\begin{enumerate}
    \item If $\bmu_0=1/\bmu$, let 
\begin{equation}
    \hat{G}(\btheta)=\sum_{i<j}\frac{1}{\mu_i\mu_j}\left[P_{ij}(\theta_i-\theta_j)-\ln(1+\exp(\theta_i-\theta_j))\right],
\end{equation}
and $\hat{G}(\tilde{\btheta})=\hat{G}(\btheta)|_{\theta_K=0}$. We can verify $\hat{G}(\tilde{\btheta})$ satisfies
\begin{equation}\label{eq:construct_G}
 \tilde{D}^{\frac{1}{2}}\frac{\partial \hat{G}}{\partial \tilde{\btheta}} = \tilde{D}^{-\frac{1}{2}}(P-Q(\btheta))_{1:(K-1),1:K}\bmu_0\triangleq \bm{t}.    
\end{equation}

By the definition of $\bm{t}$, we have $t_i=\frac{1}{\sqrt{\mu_i}}\sum_{j\neq i}\frac{P_{ij}-\sigma_{ij}}{\mu_j}$ and thus $\|\bm{t}\|_2^2=\sum_{i=1}^{K-1}\frac{1}{\mu_i}\left(\sum_{j\neq i}\frac{P_{ij}-\sigma_{ij}}{\mu_j}\right)^2$. As a result, from \eqref{eq:quadratic_dG} we conclude that when $\bmu$ is perturbed by $\bmu:=\bmu+\frac{\delta}{\bmu}$, the change of $\hat{G}$ can be expressed as a positive definite quadratic form 
\begin{align}\notag
    dG=\delta \bm{t}^{\top}\tilde{\Sigma}^{-\top}(\bpi,\bmu)\bm{t}&\ge  \frac{\delta\|\bm{t}\|_2^2}{\lambda_{\max}(\tilde{\Sigma})} \\ \label{eq:407}
    &\ge \frac{\delta}{\sqrt{K-1}}\sum_{i=1}^{K-1}\frac{1}{\mu_i}\left(\sum_{j\neq i}\frac{P_{ij}-\sigma_{ij}}{\mu_j}\right)^2 \\\label{eq:458} 
    & \ge \frac{\delta\Delta_1}{\sqrt{K-1}}\sum_{i=1}^{K-1}\frac{1}{\mu_i}\sum_{i=1}^{K-1}\mu_i\geq \Delta_1(K-1)^{1.5}\delta,
\end{align}
where \eqref{eq:407} holds because of Lemma \ref{lem:spectral-bound-Sigma} and the fact that $\sqrt{K-1}>1/2$, \eqref{eq:458} holds because of the Cauchy-Schwarz inequality and the definition of $\Delta_1$. 

Observe that $\ln(1+\exp(x))=x+O(e^{-x})$, we have 
\begin{align*}
   \hat{G}(\btheta) &= \sum_{i<j}\frac{1}{\mu_i\mu_j}\left[P_{ij}(\theta_i-\theta_j)-(\theta_i-\theta_j)\right] + O\left(\sum_{i<j}\exp(-(\theta_i-\theta_j))\right) \\
   &= - \sum_{i<j}\frac{P_{ji}(\theta_i-\theta_j)}{\mu_i\mu_j}+O\left(\sum_{i<j}\exp(-(\theta_i-\theta_j))\right) \\ 
   &=- \sum_{i<j}\frac{P_{ji}}{\beta\mu_i\mu_j}\log\frac{\pi_i}{\pi_j}+O\left(\sum_{i<j}\exp(-(\theta_i-\theta_j))\right) \\ 
   &=-\beta^{-1}G^{(1)}(\bpi;\bmu)+O\left(\sum_{i<j}\exp(-(\theta_i-\theta_j))\right).
\end{align*}
Under the assumption that the distances between each pair $\theta_i,\theta_j$ are sufficiently large, the term $O\left(\sum_{i<j}\exp(-(\theta_i-\theta_j))\right)$ is negligible. Therefore, from $\hat{G}(\btheta(\bmu+\frac{\delta}{\bmu}))<\hat{G}(\btheta(\bmu))$ and the gap indicated by the RHS of \eqref{eq:458}, we arrive at the conclusion that  
\begin{equation}
    G^{(1)}(\bpi;\bmu) - G^{(1)}(\bpi(\bmu');\bmu) > \beta\Delta_1(K-1)^{1.5}\delta - O\left(\beta\sum_{i<j}e^{^{-(\theta_i-\theta_j)}}\right).
\end{equation}
    \item If $\bmu_0=\e_K$, let 
\begin{equation}
    \hat{G}(\tilde{\btheta})=\sum_{j\neq K}\frac{1}{\mu_j}\left[P_{jK}\theta_j-\ln(1+\exp(\theta_j))\right].
\end{equation}
Similarly, we can verify that $\hat{G}(\tilde{\btheta})$ satisfies \eqref{eq:construct_G} with some $\bm{t}$ such that the change of $\hat{G}$ when $\bmu$ is changing along the direction of $\bmu_0$ can be expressed as the positive definite quadratic form $dG=\delta \bm{t}^{\top}\tilde{\Sigma}^{-\top}(\bpi,\bmu)\bm{t}>0$. Specifically, $t_i=\frac{P_{iK}-\sigma_{iK}}{\sqrt{\mu_i}}$ and thus $\|\bm{t}\|_2^2=\sum_{i=1}^{K-1}\frac{(P_{iK}-\sigma_{iK})^2}{\mu_i}$. As a result, from \eqref{eq:quadratic_dG} we conclude that when $\bmu$ is perturbed by $\bmu:=\bmu+\frac{\delta}{\bmu}$, the change of $\hat{G}$ can be expressed as a positive definite quadratic form 
\begin{align}\notag
    dG=\delta \bm{t}^{\top}\tilde{\Sigma}^{-\top}(\bpi,\bmu)\bm{t}&\ge  \frac{\delta\|\bm{t}\|_2^2}{\lambda_{\max}(\tilde{\Sigma})} \\ \label{eq:408}
    &\ge \frac{\delta}{\sqrt{K-1}}\sum_{i=1}^{K-1}\frac{(P_{iK}-\sigma_{iK})^2}{\mu_i} \\\label{eq:459} 
    & \ge \frac{\delta\Delta_2}{\sqrt{K-1}}\sum_{i=1}^{K-1}\frac{1}{\mu_i}\sum_{i=1}^{K-1}\mu_i\geq \Delta_2(K-1)^{1.5}\delta,
\end{align}
where \eqref{eq:408} holds because of Lemma \ref{lem:spectral-bound-Sigma} and the fact that $\sqrt{K-1}>1/2$, \eqref{eq:459} holds because of Cauchy-Schwarz inequality and the definition of $\Delta_2$.

Use expansion $\ln(1+\exp(x))=x+O(e^{-x})$, we have 
\begin{align*}
\hat{G}(\tilde{\btheta})=-\beta^{-1}G^{(2)}(\bpi;\bmu)+O\left(\sum_{j\neq K}\exp(-\theta_j)\right).
\end{align*}

Under the assumption that the distances between each pair $\theta_i,\theta_j$ are sufficiently large, the term $O\left(\sum_{j\neq K}\exp(-\theta_j)\right)$ is negligible. Therefore, from $\hat{G}(\btheta(\bmu+\frac{\delta}{\bmu}))<\hat{G}(\btheta(\bmu))$, we arrive at the conclusion that  
\begin{equation}
    G^{(2)}(\bpi;\bmu) - G^{(2)}(\bpi(\bmu');\bmu)>  \beta\Delta_2(K-1)^{1.5}\delta - O\left(\beta\sum_{i<j}e^{^{-(\theta_i-\theta_j)}}\right).
\end{equation}
\end{enumerate}

\end{proof}

\subsection{Long-term effects of iterative DPO}\label{app:iterative_dpo}
We first define the mixed reference/sampling dynamics for DPO as a counterpart to \Cref{def:mixed_ipo_dyn}.

\begin{definition}
\label{def:mixed_dpo_dyn}
The \emph{mixed reference/sampling DPO (MRS-DPO) dynamics}  $$\bpi_{t+1}=\mathrm{MRS}^{DPO}(\bpi_t;P,\alpha,\beta,\lambda,\bpi_{\mathrm{ref}},\bpi_0),$$ initialized by $\bpi_t|_{t=1}=\bpi_1$ is defined by
\begin{equation}
\label{eq:mix_ref_softmax_form_dpo}
\bpi_{t+1}
=
\softmax\!\Big(
\log \bpi_{\rm ref}^{(t)}
+\beta\,\btheta^{\star}(P,\bmu_t)
\Big), t\geq 1,
\end{equation}
where $\btheta^{\star}(P,\bmu)$ is the implicit mapping induced by the FOC condition 
\begin{equation}
F(\btheta^{\star},\bmu):=(P-Q(\btheta^{\star}))^{\top}\bmu = \bm{0},
\end{equation}
$Q(\btheta)$ is the BT matrix induced by $\btheta$, $\bpi_{\rm ref}^{(t)}$ and $\bmu_t$ are defined in \eqref{eq:pi_ref_t}, \eqref{eq:mu_t}, $\alpha\in[0,1]$ controls how strongly the \emph{reference} model tracks the current deployment,
$\lambda\in[0,1]$ controls how strongly the \emph{sampling} is on-policy,
$\beta>0$ is the DPO inverse temperature,
and $P$ is the underlying pairwise preference matrix.
\end{definition}

Our next \Cref{prop:mrs_dpo_cycling_rps} is a counterpart for \Cref{prop:condorcet_instability_mixed}, illustrating that under a preference structure with Condorcet cycles, the MRS-DPO dynamics can exhibit the same nonconvergence and cycling patterns suffered by MRS-IPO.

\begin{proposition}[Cycling regime for MRS--DPO under a Condorcet cycle]\label{prop:mrs_dpo_cycling_rps}
Let $K=3$ and let $P=P^{\mathrm{RPS}}_a\in\cal P$ be the rock--paper--scissors
(Condorcet cycle) preference matrix
\[
P^{\mathrm{RPS}}_a
\;=\;
\begin{pmatrix}
\frac12 & \frac12+a & \frac12-a\\[2pt]
\frac12-a & \frac12 & \frac12+a\\[2pt]
\frac12+a & \frac12-a & \frac12
\end{pmatrix},
\qquad a\in(0,\tfrac12).
\]
Assume $\bpi_{\mathrm{ref}}=\bpi_0=\bmu_0\triangleq(\frac13,\frac13,\frac13)\in\Delta_3^\circ$ and
\[
\log\bpi_{\mathrm{ref}}^{(t)}=(1-\alpha)\log\bmu_0+\alpha\log\bpi_t,
\qquad
\bmu_t=(1-\lambda)\bmu_0+\lambda\bpi_t.
\]
Consider the MRS--DPO map defined in Definition~\ref{def:mixed_dpo_dyn}:
\[
\bpi_{t+1}=\softmax\!\Big(\log\bpi_{\mathrm{ref}}^{(t)}+\beta\,\btheta^\star(P,\bmu_t)\Big),
\qquad \beta>0.
\]
Then $\bmu_0$ is a fixed point. Moreover, if
\begin{equation}\label{eq:mrs_dpo_cycling_regime}
\alpha^2+\frac{16a^2}{3}\,\beta^2\lambda^2 \;\ge\; 1,
\end{equation}
then $\bmu_0$ is not locally attracting and the iterates exhibit persistent oscillations
(rotations) around $\bmu_0$; consequently the MRS--DPO last iterate does not converge to $\bmu_0$
for generic interior initializations.
\end{proposition}

\begin{proof}
We first shown $\bmu_0$ is a fixed point. Let $\bmu=\bmu_0$ and $\btheta=\bm 0$. Then $Q(\bm 0)_{ij}=\sigma(0)=\frac12$ for all $i,j$.
Since $P^{\mathrm{RPS}}_a-\frac12\bm 1\bm 1^\top$ is skew-symmetric with zero column sums,
we have $(P^{\mathrm{RPS}}_a-Q(\bm 0))^\top\bmu_0=\bm 0$. Hence $\btheta^\star(P,\bmu_0)$ may be chosen
as $\bm 0$. With $\bpi_{\mathrm{ref}}^{(t)}=\bmu_0$ at $\bpi_t=\bmu_0$, we obtain
\[
\bpi_{t+1}=\softmax(\log\bmu_0+\beta\bm 0)=\bmu_0,
\]
so $\bmu_0$ is a fixed point.

Next we derive the linearization of $\btheta^\star(P,\bmu)$ at $\bmu=\bmu_0$.
Define $F(\btheta,\bmu)\triangleq (P-Q(\btheta))^\top\bmu$. At $(\btheta,\bmu)=(\bm 0,\bmu_0)$ we have
$F(\bm 0,\bmu_0)=\bm 0$ as above. We compute the partial derivatives at $(\bm 0,\bmu_0)$.

First, $F$ is linear in $\bmu$, so
\[
\frac{\partial F}{\partial \bmu}(\bm 0,\bmu_0)=(P-Q(\bm 0))^\top = (P-\tfrac12\bm 1\bm 1^\top)^\top.
\]
Second, for $Q_{ij}(\btheta)=\sigma(\theta_i-\theta_j)$ we have
$\partial Q_{ij}/\partial \theta_k = \sigma'(\theta_i-\theta_j)(\delta_{ik}-\delta_{jk})$ and
$\sigma'(0)=\tfrac14$. Writing $F_m(\btheta,\bmu)=\sum_i (P_{im}-Q_{im}(\btheta))\mu_i$,
\[
\frac{\partial F_m}{\partial \theta_k}(\bm 0,\bmu_0)
= -\sum_{i=1}^3 \mu_i \frac{\partial Q_{im}}{\partial \theta_k}(\bm 0)
= -\frac14\sum_{i=1}^3 \mu_i(\delta_{ik}-\delta_{mk})
= -\frac14(\mu_k-\delta_{mk}).
\]
With $\mu_{0,k}=\tfrac13$, this yields the $3\times 3$ matrix
\[
\frac{\partial F}{\partial \btheta}(\bm 0,\bmu_0)
= \frac16 I - \frac1{12}(\bm 1\bm 1^\top - I),
\]
whose restriction to the tangent space $T$ is simply $\frac14 I_T$ (because $\bm 1\bm 1^\top$ vanishes on $T$).
Therefore, by the implicit function theorem (modulo the standard shift-invariance gauge, which is immaterial because
softmax is shift-invariant),
\[
D_{\bmu}\btheta^\star(P,\bmu_0)\big|_T
=
-\left(\frac{\partial F}{\partial \btheta}(\bm 0,\bmu_0)\big|_T\right)^{-1}
\left(\frac{\partial F}{\partial \bmu}(\bm 0,\bmu_0)\big|_T\right)
=
-4\,(P-\tfrac12\bm 1\bm 1^\top)^\top\big|_T.
\]
For $P\in\cal P$, the matrix $S\triangleq P-\tfrac12\bm 1\bm 1^\top$ is skew-symmetric, hence $S^\top=-S$, and thus
\begin{equation}\label{eq:dtheta_dmu}
D_{\bmu}\btheta^\star(P,\bmu_0)\big|_T
=
4\,S\big|_T.
\end{equation}

Finally we linearize the full MRS--DPO map at $\bmu_0$. Let $\bpi=\bmu_0+\bdelta$ with $\delta\in T$ small. Then
\[
\log\bpi_{\mathrm{ref}}^{(t)}=(1-\alpha)\log\bmu_0+\alpha
\log (\bmu_0+\bdelta ),\qquad
\bmu_t=\bmu_0+\lambda\delta
\]
Using elementwise Taylor expansion, $\log(\bmu_0+\bdelta)=\log\bmu_0+3\bdelta+o(\|\bdelta\|)$
(since each coordinate of $\bmu_0$ equals $1/3$).
Using \eqref{eq:dtheta_dmu}, $\btheta^\star(P,\bmu_t)=4\lambda S\bdelta+o(\|\bdelta\|)$.
Hence the softmax argument is
\[
\log\bpi_{\mathrm{ref}}^{(t)}+\beta\btheta^\star(P,\bmu_t)
=
\log\bmu_0 + \Bigl(3\alpha I + 4\beta\lambda S\Bigr)\bdelta + o(\|\bdelta\|).
\]
The Jacobian of $\softmax$ at $\log\bmu_0$ is $\mathrm{Diag}(\bmu_0)-\bmu_0\bmu_0^\top$; restricted to $T$ it equals $\frac13 I_T$.
Therefore,
\[
\bdelta^+ \;\triangleq\; \bpi_{t+1}-\bmu_0
=
\frac13\Bigl(3\alpha I + 4\beta\lambda S\Bigr)\bdelta + o(\|\bdelta\|)
=
\Bigl(\alpha I + \frac{4\beta\lambda}{3}S\Bigr)\bdelta + o(\|\bdelta\|).
\]
For the RPS matrix, $S=P-\frac12\bm 1\bm 1^\top$ equals $a A$ where
\[
A=\begin{pmatrix}
0&1&-1\\-1&0&1\\1&-1&0
\end{pmatrix}
\]
has eigenvalues $0,\pm i\sqrt3$. Thus on $T$ the eigenvalues of the linearized map are
\[
\lambda_\pm=\alpha \pm i\,\frac{4\sqrt3}{3}\,\beta\lambda\,a.
\]
Their modulus is
\[
|\lambda_\pm|=\sqrt{\alpha^2+\frac{16}{3}\beta^2\lambda^2a^2}.
\]
If \eqref{eq:mrs_dpo_cycling_regime} holds, then $|\lambda_\pm|\ge 1$ and the fixed point $\bmu_0$
is not locally attracting; since $\Im(\lambda_\pm)\neq 0$ whenever $a \beta\lambda >0$,
the iterates rotate (oscillate) rather than approach $\bmu_0$, yielding persistent cycling behavior.
Finally, if $\alpha=1$ and $\lambda>0$, then $|\lambda_\pm|>1$ for every $\beta>0$, so the regime holds.
\end{proof}

The next result is the counterpart of \Cref{thm:mix_contraction_generalP}, recovering a hyperparameter regime of MRS-DPO to ensure stability for general preference structure $P$.

\begin{theorem}[A sufficient stability regime for MRS--DPO]
\label{thm:mrs_dpo_stability_corrected}
Assume $P\in\calP$ and fix $\alpha\in[0,1)$ and $\lambda\in[0,1)$, let $\bpi_{\rm ref},\bpi_0\in\Delta_K^\circ$.
Define the on/off-policy sampling mixture
\[
\bmu(\bpi)\;:=\;(1-\lambda)\bpi_0+\lambda\bpi,
\qquad
\underline\mu \;:=\; (1-\lambda)\min_i \pi_{0,i}\;>\;0.
\]
For each $\bmu\in\Delta_K^\circ$, let $\btheta^\star(P,\bmu)\in\bbR^K$ denote the (gauge-fixed) DPO optimizer in the
$\theta$-parameterization, i.e.\ $\mathbf{1}^\top\btheta^\star=0$.

Assume the DPO optimizer is uniformly bounded on the sampling set induced by $\lambda$:
\begin{equation}
\label{eq:B_def_corrected}
B \;:=\;\sup_{\bpi\in\Delta_K^\circ}\ \Big(\max_i \theta_i^\star(P,\bmu(\bpi))-\min_i \theta_i^\star(P,\bmu(\bpi))\Big)\;<\;\infty,
\end{equation}
and let $s_B:=\sigma'(B)>0$ where $\sigma(t)=1/(1+e^{-t})$.

Consider the MRS--DPO update defined by
\begin{equation}
\label{eq:mrs_dpo_update_corrected}
\bpi_{t+1}
=
\softmax\!\Big(
(1-\alpha)\log\bpi_{\rm ref}+\alpha\log\bpi_t+\beta\,\btheta^\star(P,\bmu(\bpi_t))
\Big),
\qquad t\ge 1,
\end{equation}
with any initialization $\bpi_1\in\Delta_K^\circ$ and $\beta>0$. Let $\Pi:=I-\frac{1}{K}\mathbf{1}\mathbf{1}^\top$ and define the centered logit coordinate
\[
\x(\bpi):=\Pi\log\bpi\in\mathbf{1}^\perp.
\]
If
\begin{equation}
\label{eq:mrs_dpo_contraction_condition_corrected}
L \;:=\; \alpha \;+\; \frac{\beta\lambda}{\underline\mu^2\,s_B} \;<\; 1,
\end{equation}
then the MRS--DPO dynamics is globally stable in centered-logit distance:
there exists a unique fixed point $\bpi^\infty\in\Delta_K^\circ$ of \eqref{eq:mrs_dpo_update_corrected}, and for all $t\ge 1$,
\[
\bigl\|\x(\bpi_t)-\x(\bpi^\infty)\bigr\|_2
\;\le\;
L^{\,t-1}\,\bigl\|\x(\bpi_1)-\x(\bpi^\infty)\bigr\|_2.
\]
Consequently, $\bpi_t\to\bpi^\infty$ as $t\to\infty$ (and in particular the dynamics cannot cycle).
\end{theorem}

\begin{proof}
First we rewrite the update as a map on centered logits. Let $\z_t:=(1-\alpha)\log\bpi_{\rm ref}+\alpha\log\bpi_t+\beta\,\btheta^\star(P,\bmu(\bpi_t))$ so that
$\bpi_{t+1}=\softmax(\z_t)$.
Using $\log\softmax(\z)=\z-\log(\sum_i e^{z_i})\mathbf{1}$ and $\Pi\mathbf{1}=0$, we obtain the exact identity
\[
\x(\bpi_{t+1})=\Pi\log\bpi_{t+1}=\Pi \z_t.
\]
Moreover, since $\x(\bpi_t)=\Pi\log\bpi_t$, we have $\Pi(\alpha\log\bpi_t)=\alpha \x(\bpi_t)$, and therefore
\begin{equation}
\label{eq:x_update}
\x(\bpi_{t+1})
=
\alpha\,\x(\bpi_t)
+
(1-\alpha)\Pi\log\bpi_{\rm ref}
+
\beta\,\Pi\btheta^\star(P,\bmu(\bpi_t)).
\end{equation}
Because we fix the gauge $\mathbf{1}^\top\btheta^\star=0$, we have $\Pi\btheta^\star=\btheta^\star$ and we may drop $\Pi$ on that term.

Define the induced map on $\mathbf{1}^\perp$:
\[
F(\x)
:=
\alpha \x + (1-\alpha)\Pi\log\bpi_{\rm ref}
+\beta\,\btheta^\star\!\big(P,\bmu(\softmax(x))\big),
\]
so that $\x(\bpi_{t+1}) = F(\x(\bpi_t))$.

For $\bpi=\softmax(\x)$, the Jacobian is $J(\bpi)=\mathrm{Diag}(\bpi)-\bpi\bpi^\top$.
For any $\v\in\bbR^K$,
\[
\v^\top J(\bpi)\v
=
\sum_i \pi_i v_i^2-\Big(\sum_i \pi_i v_i\Big)^2
=\mathrm{Var}(v_I)
\le \frac{(\max_i v_i-\min_i v_i)^2}{4}
\le \frac{2\|\v\|_2^2}{4}
=\frac12\|\v\|_2^2,
\]
where $I\sim\bpi$ and we used $\max_i v_i-\min_i v_i\le \sqrt{2}\|\v\|_2$.
Thus $\|J(\bpi)\|_{2\to2}\le 1/2$ uniformly, and by the mean value theorem,
\begin{equation}
\label{eq:softmax_lip_2}
\|\softmax(\x)-\softmax(\y)\|_2 \le \frac12\|\x-\y\|_2.
\end{equation}

Now we derive the Lipschitz continuity of the DPO optimizer map $\bmu\mapsto\btheta^\star(P,\bmu)$. Fix $\bmu\in\Delta_K^\circ$ and let $Q(\btheta)$ be the Bradley--Terry matrix induced by $\btheta$,
$Q_{ij}(\btheta)=\sigma(\theta_i-\theta_j)$.
Under the dueling identity, the DPO population objective in $\btheta$ has gradient
\[
g(\btheta,\bmu):=\nabla_{\btheta}\mathcal{L}(\btheta;\bmu)
=\mathrm{Diag}(\bmu)\,(Q(\btheta)-P)\,\bmu,
\]
and Hessian on $\mathbf{1}^\perp$ given by the Laplacian form
\[
H(\btheta,\bmu)
=
\sum_{1\le i<j\le K}\mu_i\mu_j\,\sigma'(\theta_i-\theta_j)\,(\e_i-\e_j)(\e_i-\e_j)^\top.
\]
By the bounded-oscillation assumption \eqref{eq:B_def_corrected}, every difference $\theta_i^\star-\theta_j^\star$ lies in $[-B,B]$,
so $\sigma'(\theta_i^\star-\theta_j^\star)\ge s_B$.
Also, for $\bmu=\bmu(\bpi)$ we have $\mu_i\ge \underline\mu$.
Hence for any $\v\in\mathbf{1}^\perp$,
\[
\v^\top H(\btheta^\star,\bmu)\v
\ge
\underline\mu^2 s_B\sum_{i<j}(v_i-v_j)^2
=
\underline\mu^2 s_B\,K\,\|\v\|_2^2,
\]
where we used $\sum_{i<j}(v_i-v_j)^2=K\|\v\|_2^2$ on $\mathbf{1}^\perp$.
Therefore $g(\cdot,\bmu)$ is strongly monotone on $\mathbf{1}^\perp$ with modulus
\begin{equation}
\label{eq:strong_monotone}
m:=\underline\mu^2 s_B K.
\end{equation}

Next, fix $\btheta$ and write $M:=Q(\btheta)-P$.
Since $M_{ij}\in[-1,1]$, we have $\|M\|_2\le \|M\|_F\le K$.
For any $\bmu,\bnu\in\Delta_K$,
\[
g(\btheta,\bmu)-g(\btheta,\bnu)
=
\mathrm{Diag}(\bmu-\bnu)M\bmu + \mathrm{Diag}(\bnu)M(\bmu-\bnu),
\]
and thus
\[
\|g(\btheta,\bmu)-g(\btheta,\bnu)\|_2
\le
\|\bmu-\bnu\|_2\,\|M\bmu\|_2 + \|M\|_2\,\|\bmu-\bnu\|_2
\le
K\|\bmu-\bnu\|_2 + K\|\bmu-\bnu\|_2
=
2K\|\bmu-\bnu\|_2.
\]
Now let $\btheta_\mu^\star:=\btheta^\star(P,\bmu)$ and $\btheta_\nu^\star:=\btheta^\star(P,\bnu)$ (both in $\mathbf{1}^\perp$),
so $g(\btheta_\mu^\star,\bmu)=0$ and $g(\btheta_\nu^\star,\bnu)=0$.
By strong monotonicity \eqref{eq:strong_monotone},
\[
m\|\btheta_\mu^\star-\btheta_\nu^\star\|_2
\le
\|g(\btheta_\mu^\star,\bmu)-g(\btheta_\nu^\star,\bmu)\|_2
=
\|g(\btheta_\nu^\star,\bnu)-g(\btheta_\nu^\star,\bmu)\|_2
\le
2K\|\bmu-\bnu\|_2,
\]
and hence
\begin{equation}
\label{eq:theta_lip}
\|\btheta^\star(P,\bmu)-\btheta^\star(P,\bnu)\|_2
\le
\frac{2K}{m}\|\bmu-\bnu\|_2
=
\frac{2}{\underline\mu^2 s_B}\,\|\bmu-\bnu\|_2.
\end{equation}

Finally, we derive contraction of the centered-logit map. Let $\x,\y\in\mathbf{1}^\perp$ and set $\bpi=\softmax(\x)$, $\bpi'=\softmax(\y)$.
Since $\bmu(\bpi)=(1-\lambda)\bpi_0+\lambda\bpi$, we have
\[
\|\bmu(\bpi)-\bmu(\bpi')\|_2
=
\lambda\|\bpi-\bpi'\|_2
\le
\frac{\lambda}{2}\|\x-\y\|_2
\]
by \eqref{eq:softmax_lip_2}.
Combining with \eqref{eq:theta_lip} gives
\[
\|\btheta^\star(P,\bmu(\bpi))-\btheta^\star(P,\bmu(\bpi'))\|_2
\le
\frac{2}{\underline\mu^2 s_B}\cdot \frac{\lambda}{2}\|\x-\y\|_2
=
\frac{\lambda}{\underline\mu^2 s_B}\|\x-\y\|_2.
\]
Therefore,
\[
\|F(\x)-F(\y)\|_2
\le
\alpha\|\x-\y\|_2 + \beta\cdot \frac{\lambda}{\underline\mu^2 s_B}\|\x-\y\|_2
=
\Big(\alpha+\frac{\beta\lambda}{\underline\mu^2 s_B}\Big)\|\x-\y\|_2
= L\|\x-\y\|_2.
\]
Under the condition \eqref{eq:mrs_dpo_contraction_condition_corrected}, $L<1$, so $F$ is a contraction on the complete space
$(\mathbf{1}^\perp,\|\cdot\|_2)$.
By Banach's fixed point theorem, $F$ admits a unique fixed point $\x^\infty$ and $\x(\bpi_t)\to \x^\infty$ geometrically.
Let $\bpi^\infty:=\softmax(\x^\infty)$; then $\bpi^\infty$ is the unique fixed point of \eqref{eq:mrs_dpo_update_corrected} and
$\bpi_t\to\bpi^\infty$.
\end{proof}

Finally, we present the long-term characterization of MRS-DPO dynamics under SST preference structures in the following Theorem, which corresponds to \Cref{thm:true_vs_eps_collapse} and reveals a similar concern of entropy collapse.

\begin{theorem}[Entropy/peak-collapse of MRS--DPO under SST]\label{thm:dpo_collapse_sst}
Assume that $P\in\cal P$ is SST. Let $\bpi_0\in\Delta_K^\circ$ and $\bpi_{\rm ref}\in\Delta_K^\circ$. Consider the MRS--DPO dynamics
of Definition~\ref{def:mixed_dpo_dyn},
\[
\bpi_{t+1}
=
\softmax\!\Big(
\log \bpi_{\rm ref}^{(t)}
+\beta\,\btheta^{\star}(P,\bmu_t)
\Big),\qquad t\ge 1,
\]
where $\btheta^\star(P,\bmu)$ is defined implicitly by the DPO FOC
$F(\btheta^\star,\bmu)=(P-Q(\btheta^\star))^\top\bmu=\bm 0$, and $Q(\btheta)$ is the BT matrix induced by $\btheta$.

Define the (SST) gap constants
\begin{equation}\label{eq:delta_def_dpo}
\underline\delta
\;\coloneqq\;
\min_{i\in[K-1]}\ \min_{j\in[K]}\bigl(P_{ij}-P_{i+1,j}\bigr)\;>\;0,
\qquad
\delta
\;\coloneqq\;
\min_{i\in[K-1]}
\Big\{(1-\lambda)\sum_{j=1}^K \pi_{0,j}\,(P_{ij}-P_{i+1,j})\Big\}\;>\;0,
\end{equation}
(where $\delta>0$ follows from $\bpi_0\in\Delta_K^\circ$ and SST).

Then the MRS--DPO dynamics exhibits the following two forms of collapse.

\begin{itemize}
\item \textbf{($\varepsilon$-collapse for $\alpha,\lambda\in[0,1)$ when $\beta/(1-\alpha)$ is large).}
Let
\[
\kappa \;\coloneqq\; \min_{i\in[K-1]}\log\frac{\pi_{{\rm ref},i}}{\pi_{{\rm ref},i+1}}
\quad (\text{finite since }\bpi_{\rm ref}\in\Delta_K^\circ),
\quad
c \;\coloneqq\; \frac12\Big(\kappa+\frac{4\beta\delta}{1-\alpha}\Big),
\quad
\varepsilon \;\coloneqq\; \bigl(e^{c}-1\bigr)^{-1}.
\]
If $\alpha,\lambda\in[0,1)$ and $c>0$ (in particular, if $\beta\delta/(1-\alpha)$ is sufficiently large),
then there exists a finite $T$ such that for all $t\ge T$,
\begin{equation}\label{eq:dpo_eps_collapse}
\bpi_{t,1}\ \ge\ 1-\varepsilon.
\end{equation}
In particular, if $\bpi_{\rm ref}$ is \emph{aligned} with the SST order (i.e.\ $\pi_{{\rm ref},1}\ge\cdots\ge\pi_{{\rm ref},K}$, so that $\kappa\ge 0$),
then one may take $c=\frac{2\beta\delta}{1-\alpha}$ and hence
$\varepsilon=(\exp(\frac{2\beta\delta}{1-\alpha})-1)^{-1}$.

\item \textbf{(True entropy collapse when $\alpha=1$).}
If $\alpha=1$, then $\bpi_t\to \e_1$ and $H(\bpi_t)\to 0$ as $t\to+\infty$.
Moreover, there exists a finite $T$ such that for all $t\ge T$,
\begin{equation}\label{eq:dpo_entropy_monotone}
H(\bpi_{t+1})<H(\bpi_t),
\end{equation}
where $H(\bpi)\coloneqq-\sum_{i=1}^K \pi_i\log\pi_i$ is the Shannon entropy.
\end{itemize}
\end{theorem}

\begin{proof}
Throughout the proof, fix $t$ and abbreviate
\[
\bmu\coloneqq\bmu_t,\qquad
\btheta^\star\coloneqq\btheta^\star(P,\bmu),\qquad
Q\coloneqq Q(\btheta^\star),
\qquad
\bpi^+\coloneqq\bpi_{t+1},\qquad
\bpi\coloneqq\bpi_t,\qquad
\bpi_{\rm ref}\coloneqq\bpi_{\rm ref}^{(t)}.
\]

First we utilize the row-form of the DPO FOC and derive a gap lower bound for $\btheta^\star$. Since $P,Q\in\cal P$ we have $(P-Q)^\top=-(P-Q)$ (skew-symmetry), and the FOC
$(P-Q)^\top\bmu=\bm 0$ is equivalent to
\begin{equation}\label{eq:row_FOC}
(P-Q)\bmu=\bm 0
\quad\Longleftrightarrow\quad
\sum_{j=1}^K \mu_j\bigl(P_{ij}-Q_{ij}\bigr)=0,\ \ \forall i\in[K].
\end{equation}
Thus for each $i$,
\begin{equation}\label{eq:row_match}
\sum_{j=1}^K \mu_j P_{ij}
=
\sum_{j=1}^K \mu_j Q_{ij}
=
\sum_{j=1}^K \mu_j\,\sigma(\theta_i^\star-\theta_j^\star).
\end{equation}

Fix any $i\in[K-1]$ and subtract \eqref{eq:row_match} for $i$ and $i+1$:
\begin{equation}\label{eq:gap_identity}
\sum_{j=1}^K \mu_j\bigl(P_{ij}-P_{i+1,j}\bigr)
=
\sum_{j=1}^K \mu_j\Big(\sigma(\theta_i^\star-\theta_j^\star)-\sigma(\theta_{i+1}^\star-\theta_j^\star)\Big).
\end{equation}
Because $\sigma$ is globally $1/4$-Lipschitz (since $\sup_x \sigma'(x)=1/4$), for each $j$,
\[
\sigma(\theta_i^\star-\theta_j^\star)-\sigma(\theta_{i+1}^\star-\theta_j^\star)
\le \frac14(\theta_i^\star-\theta_{i+1}^\star).
\]
Plugging this into \eqref{eq:gap_identity} and using $\sum_j\mu_j=1$ yields the key lower bound
\begin{equation}\label{eq:theta_gap_lower}
\theta_i^\star-\theta_{i+1}^\star
\ \ge\
4\sum_{j=1}^K \mu_j\bigl(P_{ij}-P_{i+1,j}\bigr).
\end{equation}
Under SST , the right-hand side is strictly positive; in particular,
\[
\theta_i^\star-\theta_{i+1}^\star \ \ge\ 4\underline\delta
\qquad\text{for all }i\in[K-1] \text{ and all }\bmu\in\Delta_K.
\]
Moreover, since $\bmu_t=(1-\lambda)\bpi_0+\lambda\bpi_t$ and $(P_{ij}-P_{i+1,j})\ge 0$ for all $j$,
\[
\sum_{j=1}^K \mu_{t,j}(P_{ij}-P_{i+1,j})
\ \ge\
(1-\lambda)\sum_{j=1}^K \pi_{0,j}(P_{ij}-P_{i+1,j})
\ \ge\ \delta,
\]
so for all $t$ and all $i\in[K-1]$,
\begin{equation}\label{eq:theta_gap_delta}
\theta_i^\star-\theta_{i+1}^\star \ \ge\ 4\delta.
\end{equation}

Next, we derive the log-ratio recursion formula for the policy iterates. By the softmax and multiplicative form of the update
\[
\pi^+_i
=
\frac{\pi_{{\rm ref},i}\,e^{\beta\theta_i^\star}}{\sum_{k=1}^K \pi_{{\rm ref},k}\,e^{\beta\theta_k^\star}},
\]
hence for each $i\in[K-1]$,
\begin{equation}\label{eq:ratio_update}
\log\frac{\pi^+_i}{\pi^+_{i+1}}
=
\log\frac{\pi_{{\rm ref},i}}{\pi_{{\rm ref},i+1}}
+\beta(\theta_i^\star-\theta_{i+1}^\star).
\end{equation}

Assume $\alpha,\lambda\in[0,1)$, Using $\log\bpi_{\rm ref}^{(t)}=(1-\alpha)\log\bpi_{\rm ref}+\alpha\log\bpi_t$ Subtracting for coordinate $i+1$ from $i$ gives
\begin{equation}\label{eq:ref_ratio_lower}
\log\frac{\pi_{{\rm ref},i}}{\pi_{{\rm ref},i+1}}
\ =\
\alpha\log\frac{\pi_i}{\pi_{i+1}}
+(1-\alpha)\log\frac{\pi_{{\rm ref},i}}{\pi_{{\rm ref},i+1}}
\ \ge\
\alpha\log\frac{\pi_i}{\pi_{i+1}}+(1-\alpha)\kappa.
\end{equation}
Combining \eqref{eq:ratio_update}, \eqref{eq:ref_ratio_lower}, and \eqref{eq:theta_gap_delta} yields the uniform recursion
\begin{equation}\label{eq:affine_recursion}
\log\frac{\pi^+_i}{\pi^+_{i+1}}
\ \ge\
\alpha\log\frac{\pi_i}{\pi_{i+1}} + (1-\alpha)\kappa + 4\beta\delta,
\qquad \forall i\in[K-1].
\end{equation}
Let the log ratio $r_{t,i}\coloneqq \log(\pi_{t,i}/\pi_{t,i+1})$. Then \eqref{eq:affine_recursion} implies
\[
r_{t+1,i} \ge \alpha r_{t,i} + (1-\alpha)\kappa + 4\beta\delta.
\]
Solving this affine recursion gives
\begin{equation}\label{eq:Lt_lower}
r_{t,i}
\ \ge\
\alpha^{t-1}r_{1,i}
+(1-\alpha^{t-1})\Big(\kappa+\frac{4\beta\delta}{1-\alpha}\Big).
\end{equation}
If $c=\frac12(\kappa+\frac{4\beta\delta}{1-\alpha})>0$, then since $\alpha^{t-1}\to 0$ there exists
a finite $T$ such that for all $t\ge T$ and all $i\in[K-1]$,
\begin{equation}\label{eq:ratio_uniform_c}
\log\frac{\pi_{t,i}}{\pi_{t,i+1}} \ \ge\ c.
\end{equation}
From \eqref{eq:ratio_uniform_c}, we have $\pi_{t,i+1}\le \pi_{t,i}e^{-c}$ and by induction
$\pi_{t,k}\le \pi_{t,1}e^{-(k-1)c}$ for all $k\ge 2$. Hence
\[
1-\pi_{t,1}=\sum_{k=2}^K \pi_{t,k}
\le \pi_{t,1}\sum_{m=1}^{K-1} e^{-mc}
\le \frac{\pi_{t,1}}{e^c-1},
\]
which implies $1\le \pi_{t,1}(1+\frac{1}{e^c-1})=\pi_{t,1}\frac{e^c}{e^c-1}$ and thus
\[
\pi_{t,1}\ \ge\ \frac{e^c-1}{e^c}\ =\ 1-e^{-c}\ \ge\ 1-(e^c-1)^{-1}\ =\ 1-\varepsilon.
\]
This proves \eqref{eq:dpo_eps_collapse}.

In addition, if $\bpi_{\rm ref}$ is aligned with the order $1\succ\cdots\succ K$ (nonincreasing), then $\kappa\ge 0$ and the stated simplification follows.

Finally we show the true collapse when $\alpha=1$ and the induced entropy monotonicity. Since $\alpha=1$, $\bpi_{\rm ref}^{(t)}=\bpi_t$ and \eqref{eq:ratio_update} becomes
\begin{equation}\label{eq:alpha1_ratio_update}
\log\frac{\pi_{t+1,i}}{\pi_{t+1,i+1}}
=
\log\frac{\pi_{t,i}}{\pi_{t,i+1}}+\beta(\theta_i^\star-\theta_{i+1}^\star).
\end{equation}
By \eqref{eq:theta_gap_lower} and SST , we have
$\theta_i^\star-\theta_{i+1}^\star\ge 4\underline\delta$ for all $i$, hence from \eqref{eq:alpha1_ratio_update}
\[
\log\frac{\pi_{t+1,i}}{\pi_{t+1,i+1}} \ \ge\ \log\frac{\pi_{t,i}}{\pi_{t,i+1}} + 4\beta\underline\delta,
\]
so each adjacent log-ratio diverges to $+\infty$ at least linearly in $t$. The same geometric-series argument as above
then yields $\pi_{t,1}\to 1$, i.e.\ $\bpi_t\to \e_1$.

Since $\bpi_t\to \e_1$, we also have $H(\bpi_t)\to 0$.
Finally, $H$ is continuous on $\Delta_K^\circ$ and strictly Schur-concave, and the update with $\alpha=1$ is a strict
reweighting toward action $1$ with uniformly separated weights
($\theta_1^\star>\theta_2^\star>\cdots>\theta_K^\star$ and gaps $\ge 4\underline\delta$). Therefore, there exists a neighborhood
of $\e_1$ in which the MRS--DPO step strictly decreases entropy. Since $\bpi_t\to \e_1$, the trajectory eventually enters this
neighborhood and remains there, which proves the existence of a finite $T$ such that \eqref{eq:dpo_entropy_monotone} holds for all $t\ge T$.
\end{proof}

\section{Experimental Details and Additional Results}
\label{app:experiments}

\subsection{Dataset description}
\label{app:dataset}

All experiments were conducted using the \textsc{NVIDIA HelpSteer} dataset, which contains real human feedback collected for LLM alignment.
The raw dataset consists of $35{,}331$ response-level records, each associated with a textual prompt and a model-generated response.

\paragraph{Available annotations.}
Each response is independently evaluated along multiple qualitative attributes.
Specifically, the available fields include:
\begin{center}
\texttt{\{helpfulness, correctness, coherence, complexity, verbosity\}}.
\end{center}
These scalar scores are used as the basis for constructing pairwise preference matrices.

An example entry from the dataset is shown below:
\begin{quote}
\label{sample}
\small
\texttt{
\{
  ``prompt'': ``Explain why the sky is blue.'',
  ``response'': ``The sky appears blue due to Rayleigh scattering...'',
  ``helpfulness'': 3,
  ``correctness'': 4,
  ``coherence'': 4,
  ``complexity'': 2,
  ``verbosity'': 2
\}
}
\end{quote}

\paragraph{Data filtering and preprocessing.}
Our analysis focuses on preference dynamics over fixed-size sets of competing
responses.
To enable a uniform experimental protocol across prompts with varying numbers
of responses, we restrict attention to prompts with at least $K=4$ associated
responses and construct preference instances by selecting $K=4$ responses per
prompt uniformly at random.
This choice balances expressiveness and data coverage: most prompts in the dataset correspond to at most $4$ responses, and preference sets of size $4$ are sufficiently rich to admit both ST and cyclic structures.  

\paragraph{Resulting structure.}
After preprocessing, each retained prompt induces a local preference learning instance consisting of a fixed-size set of $K=4$ responses, together with their corresponding multi-dimensional evaluation scores.
This collection of $7{,}497$ local instances forms the basis from which all preference matrices used in the experiments are derived.


\subsection{Construction of preference matrices}
\label{app:construction}

For learning instance after the preprocessing, each prompt associates with $K=4$ responses, we construct a local probabilistic preference matrix
\[
P \in [0,1]^{K \times K},
\]
where $P_{ij} \in [0,1]$ represents the probability that response $i$ is preferred over response $j$, and $P_{ij} + P_{ji} = 1$ for any $i,j \in [K]$. 


\paragraph{Randomized attribute selection.}
Let $\mathcal{D}$ denote the set of available evaluation attributes.
For each unordered pair of responses $(i,j)$, we independently sample an attribute
\[
d \sim \mathrm{Unif}(\mathcal{D}),
\]
and denote by $s_i^{(d)}$ and $s_j^{(d)}$ the corresponding scalar scores.

\paragraph{Matrix Construction.}
For each $(i,j)$, given the selected attribute $d$ from the last step, we generate the probability related to $(i,j)$ via sigmoid function:
\[
P_{ij}
= \mathrm{sigmoid}(s^{(d)}_i - s^{(d)}_j)=
\frac{\exp\!\big(s_i^{(d)}\big)}
{\exp\!\big(s_i^{(d)}\big) + \exp\!\big(s_j^{(d)}\big)},
\qquad
P_{ji} = 1 - P_{ij},
\]
with $P_{ii} = \tfrac{1}{2}$ by convention.
By construction, the resulting matrix satisfies
$P_{ij} + P_{ji} = 1$for all $i \neq j$ and thus is a valid preference matrix.

While scores on the same attribute yield a ranking over responses, independent randomization of attributes across response pairs can introduce conflicting pairwise comparisons across different dimensions.
As a result, the aggregated preference matrix need not correspond to any single latent scoring function, nor even be transitive. 
We classify them into two categories: ST matrices defined in \Cref{def:st}, and cyclic matrices that are not transitive.  

\paragraph{Preference Structure Verification.}
To verify ST matrices in the experiments, for each matrix $P$, we first extract a candidate total order
by sorting responses according to their aggregate preference scores $\sum_j P_{ij}$.
We then verify whether $P$ satisfies the strongly transitive condition in
\Cref{def:st} with respect to this order.
Matrices that fail this verification are not classified as strongly transitive.

We test a matrix and verify it is cyclic if its induced dominance graph contains a
directed cycle, i.e., there exist indices $i_1,\dots,i_m$ with $m\ge3$ such that
$P_{i_\ell i_{\ell+1}}>\tfrac12$ for all $\ell$ (indices modulo $m$).

Matrices that are neither ST nor cyclic are excluded from the analysis.

\paragraph{Statistics.}
After filtering, the dataset contains $4,924$ ST preference matrices and $118$ cyclic preference matrices.

\subsection{Experimental setup in \Cref{sec:experiments}}
\label{app:protocol}

All experiments in \Cref{sec:experiments} are conducted using preference
matrices constructed from the dataset, without additional sampling noise.
These matrices represent probabilistic pairwise comparisons aggregated from
human evaluation scores, as described in
\Cref{app:construction}.

Unless otherwise stated, for each preference instance and parameter configuration
$(\alpha,\beta\lambda)$, we fix $\lambda = 0.5$ and set $\beta = v/0.5$ for each value $v$ of $\beta\lambda$. We initialize the dynamics with
$\pi_{\mathrm{ref}} = \pi_0$ equal to the uniform distribution over $K$ responses,
set the initial policy $\pi_1 = \pi_0$, and run the MRS-IPO dynamics for
$T = 3000$ iterations.
This horizon is sufficient to observe convergence, collapse, or persistent
non-convergent behavior across all parameter regimes considered.



\subsubsection{Preferences matrices for illustrative examples in \Cref{sec:exp-illustrative} and \Cref{app:single_param}} 
\label{app:form-illustrative-examples}

We use the same preference matrices for the illustrative examples, both in the noiseless setting in \Cref{sec:exp-illustrative} and the finite sample setting in \Cref{app:protocol-noise}. 

The strongly transitive preference matrix used throughout
is
given by
\[
P_{\mathrm{ST}} =
\begin{pmatrix}
0.5 & 0.731 & 0.269 & 0.5 \\
0.269 & 0.5 & 0.119 & 0.269 \\
0.731 & 0.881 & 0.5 & 0.731 \\
0.5 & 0.731 & 0.269 & 0.5
\end{pmatrix}.
\]
This matrix admits a total order consistent with Definition~\ref{def:st}, and all
pairwise preferences satisfy the strong transitivity inequalities with respect
to that order.

The cyclic preference matrix is given by
\[
P_{\mathrm{cyc}} =
\begin{pmatrix}
0.5 & 0.731 & 0.269 & 0.5 \\
0.269 & 0.5 & 0.731 & 0.881 \\
0.731 & 269 & 0.5 & 0.731 \\
0.5 & 0.119 & 0.269 & 0.5
\end{pmatrix},
\]
with cycle $3 \succ 1\succ 2\succ3$. The induced dominance graph of this matrix contains directed cycles.

\subsubsection{Statistical metrics for aggregated experimental results in \Cref{sec:aggregated-results} and \Cref{app:complementary-aggregated results}} 
\label{aggregate_ipo}

In \Cref{sec:aggregated-results} and \Cref{app:complementary-aggregated results}, we consider aggregated phenomena across all instances constructed from the HelpSteer dataset.

Specifically, for ST preferences, policy collapse is quantified using
the Shannon entropy of the terminal policy:
\[
H(\bpi_t) \;=\; -\sum_{i=1}^K \pi_{t,i}\log \pi_{t,i}.
\]
Low entropy indicates more concentration on a small subset of responses, implying more serious policy collapse.

For cyclic preferences,
we measure the oscillatory behavior of the dynamics using the time-averaged variance:
\[
\mathrm{CS}(\bpi)
\;=\;
\frac{1}{K}
\sum_{i=1}^K
\operatorname{Var}_{t \ge T_0}\left[\pi_{t,i}\right],
\]
computed using policies after a burn-in period $T_0$.
This metric converges to zero if and only if the dynamics converge to a fixed point, and
increases with the amplitude of persistent oscillations.

\subsection{Finite sample preference noise under single parameter variations}
\label{app:protocol-noise}

We consider preference matrices constructed from the dataset as true underlying preference matrices.  
To account for the finite sample variability inherent in real-world preference
data, we introduce controlled stochastic perturbations to these matrices before
running the dynamics and conduct additional experiments examing the behavior of MRS-IPO dynamics under finite sample preference estimation.
The goal of these experiments is not to study noise as a perturbation per se,
but to assess whether the qualitative dynamical behaviors identified in \Cref{sec:self_ref_dynamics} 
persists when preferences are instantiated from finite data.

Again, we consider the ST matrix and cyclic matrix in \Cref{app:form-illustrative-examples}. For each preference matrix
$P \in [0,1]^{K \times K}$, we generate an empirical realization
$\widehat P$ by independently sampling a fixed number $n$ of pairwise comparisons
for each unordered response pair.
For each pair $(i,j)$, the empirical preference value $\widehat P_{ij}$ is obtained
by averaging the outcomes of $n$ independent Bernoulli trials with success
probability $P_{ij}$, yielding
\[
\widehat P_{ij} \in \Big\{0, \tfrac{1}{n}, \tfrac{2}{n}, \dots, 1\Big\},
\qquad
\widehat P_{ji} = 1 - \widehat P_{ij}.
\]

Importantly, this finite sample preference realization $\widehat P$ is sampled
\emph{once per experimental run} and then held fixed throughout the entire
execution of the MRS-IPO dynamics.
That is, preference noise is introduced at the level of the preference instance,
rather than being resampled over time or depending on the evolving policy
$\pi_t$.
This design isolates the effect of static estimation noise in the preference
matrix and avoids confounding preference dynamics with on-policy resampling
effects.

For each parameter of $\alpha,\beta, \lambda$, we change its value with two other parameters fixed to observe how the dynamics change as one parameter varies. 
Unless otherwise stated, all noisy experiments use the same algorithmic
initialization as in the noiseless setting.
Specifically, we set $\pi_{\mathrm{ref}} = \pi_0$ to be the uniform distribution
over the $K$ responses, initialize the policy with $\pi_1 = \pi_0$, and run the
dynamics for $T = 3000$ iterations.
Throughout these experiments, we set the number of finite sample realizations per
preference instance to $50$, which provides a stable estimate of the mean
trajectory and variability without obscuring instance level effects and fix $n = 5$, which introduces moderate
variability without overwhelming the underlying preference signal.

For parameter regimes in which finite sample realizations remain approximately
phase aligned (e.g., convergent dynamics under strongly transitive preferences or
weakly oscillatory regimes under cyclic preferences), we summarize behavior by
reporting the mean trajectory across realizations, with shaded regions indicating
$\pm$ one standard error (Figures~\ref{fig:app-st-alpha}--\ref{fig:app-st-lambda} and Figures~\ref{fig:app-cyc-alpha}--\ref{fig:app-cyc-lambda}).

In contrast, in regimes where the dynamics depend strongly on the realized
samples, such as collapse under strongly transitive preferences or strong
oscillations under cyclic preferences, finite sample realizations can differ
substantially in onset time, phase, or amplitude.
In these cases, averaging trajectories can cancel the phenomena and obscure the
underlying dynamical structure.
Accordingly, we present
representative individual
finite sample realizations to illustrate the diversity of behaviors that arise in
practice (Figure~\ref{fig:collapse},~\ref{fig:cyc-repr}).

Overall, these experiments indicate that the qualitative dynamical regimes
predicted by the theory, convergence, collapse, and persistent cycling, remain
observable even when the idealized structural assumptions (e.g., strict strong
transitivity or perfectly cyclic preferences) are mildly violated due to finite
sampling noise.

\subsubsection{Experimental results}
\label{app:single_param}

\paragraph{Cyclic preferences.}
Figures~\ref{fig:app-cyc-alpha}--\ref{fig:app-cyc-lambda} present the corresponding
experiments for cyclic preference matrices.
In parameter regimes where the dynamics weakly oscillate, the oscillatory phases remain approximately aligned across realizations.

In these cases, averaged trajectories capture the qualitative behavior of the
system and reveal how varying $\alpha$, $\beta$, or $\lambda$ modulates the
magnitude of oscillations.

\begin{figure}[!htbp] 
    \centering 
    \includegraphics[width=\linewidth]{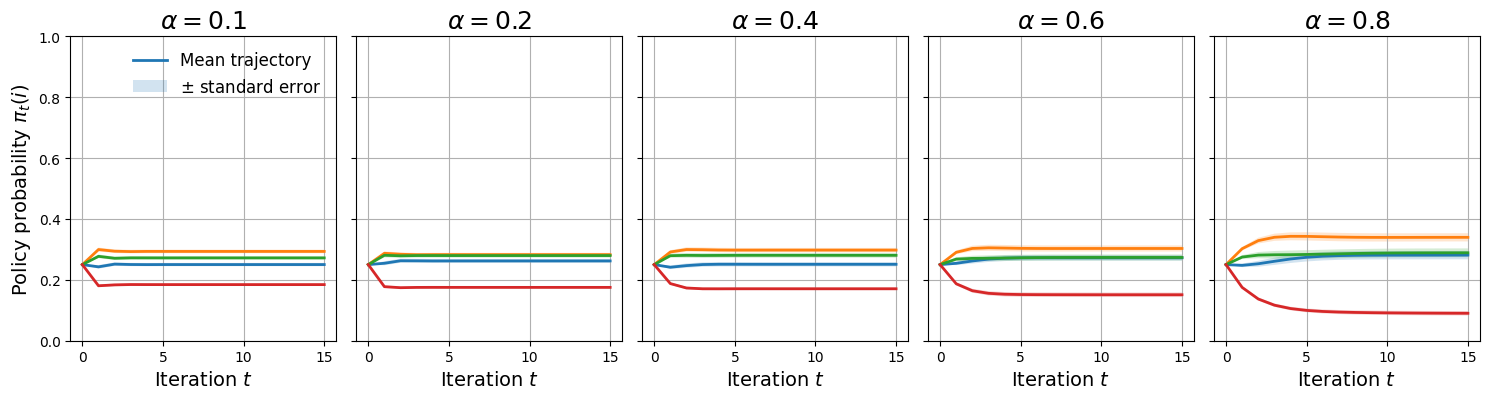}
    \caption{Effect of varying $\alpha$ under cyclic preferences, with $\beta = 2.0$ and $\lambda = 0.8$.} 
    \label{fig:app-cyc-alpha}
\end{figure}

\begin{figure}[!htbp] 
    \centering 
    \includegraphics[width=\linewidth]{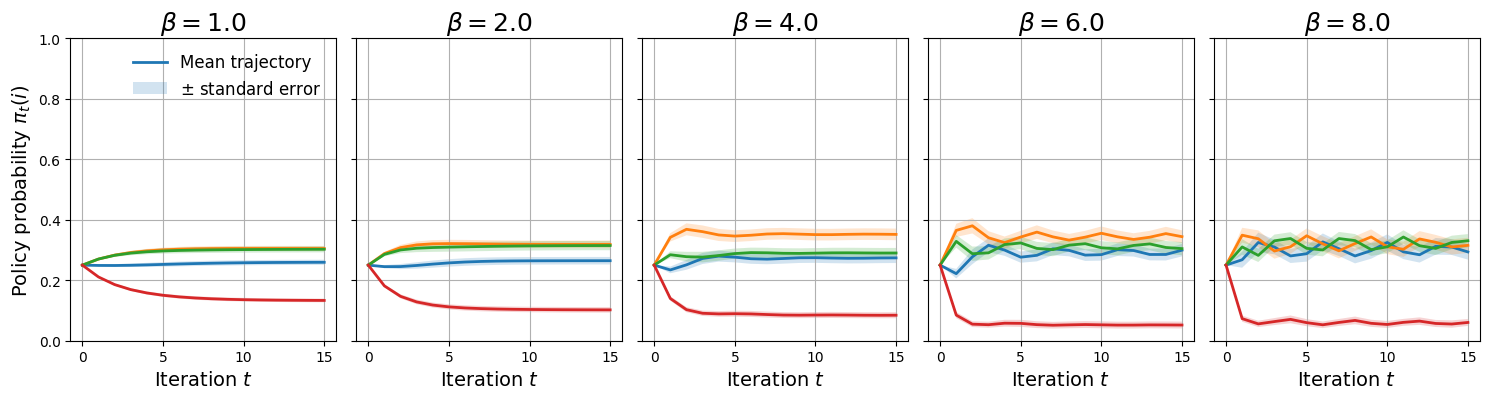}
    \caption{Effect of varying $\beta$ under cyclic preferences, with $\alpha = 0.8$ and $\lambda = 0.8$.} 
    \label{fig:app-cyc-beta}
\end{figure}

\begin{figure}[!htbp] 
    \centering 
    \includegraphics[width=\linewidth]{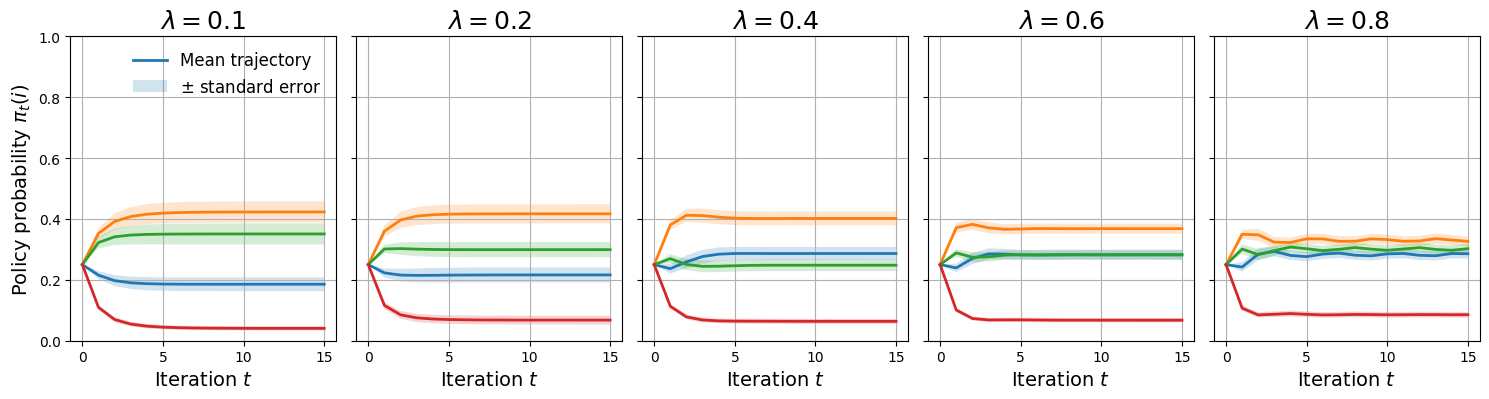}
    \caption{Effect of varying $\lambda$ under cyclic preferences, with $\alpha = 0.6$ and $\beta = 5.0$.} 
    \label{fig:app-cyc-lambda}
\end{figure}

However, in regimes where the dynamics oscillate strongly and persistently, the amplitudes, phases and periods vary, substantially across realizations.
Averaging trajectories in these regimes obscures the underlying oscillatory
structure due to phase cancellation.
Accordingly, we complement the averaged results with representative individual
finite sample realizations.
\Cref{fig:cyc-repr} shows two such realizations generated from the same cyclic preference matrix under identical algorithmic
parameters, illustrating the diversity of cyclic behaviors that can arise from
finite sample variability.

\begin{figure}[!htbp] 
    \centering 
    \includegraphics[width=\linewidth]{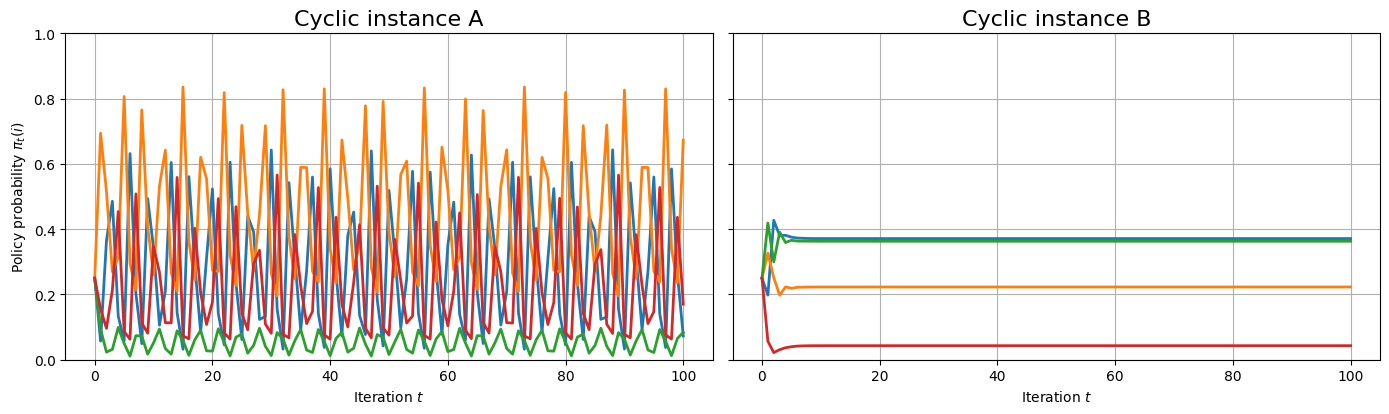}
    \caption{Two representative finite sample realizations generated from the same cyclic preference matrix. $\alpha=0.9,\ \beta=10.0,\ \lambda=0.8$}. 
\label{fig:cyc-repr}
\end{figure}

\paragraph{ST preferences.}
Figures~\ref{fig:app-st-alpha}--\ref{fig:app-st-lambda} report single parameter
variation experiments under ST preferences in parameter regimes
where the dynamics converge to a relatively diverse entropy.
In these regimes, finite-sample realizations exhibit highly aligned trajectories,
and averaging across instances yields a faithful representation of the dynamics.
Varying $\alpha$, $\beta$, or $\lambda$ affects the speed of convergence and the
concentration of the limiting policy in a manner consistent with the
aggregated results analysis.
Finite-sample estimation introduces only mild variability during the transient
phase and does not qualitatively alter the convergent behavior.

\begin{figure}[!htbp] 
    \centering 
    \includegraphics[width=\linewidth]{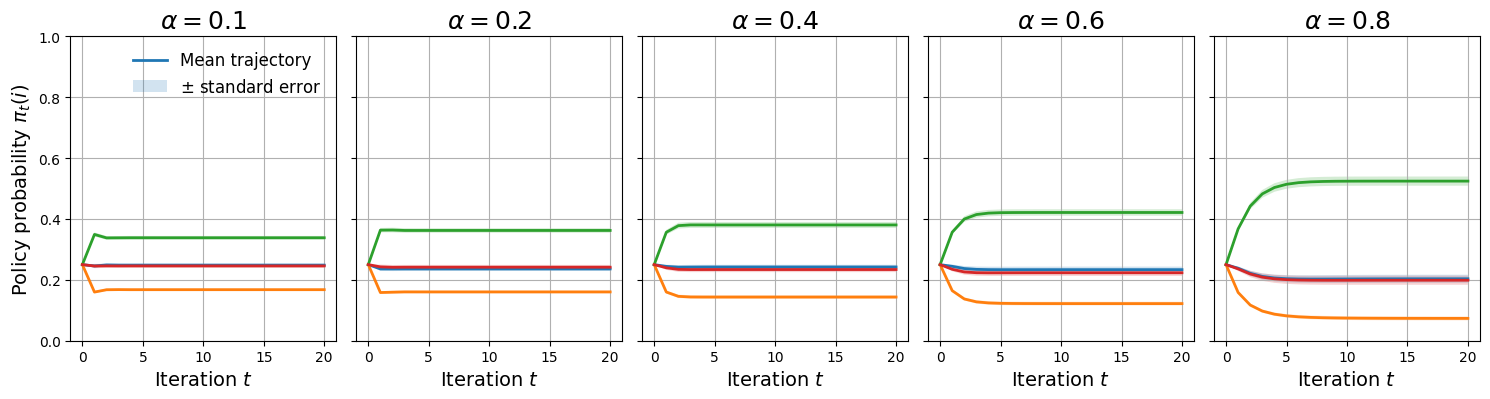}
    \caption{Effect of varying $\alpha$ under ST preferences, with $\beta = 2.0$ and $\lambda = 0.8$.} 
    \label{fig:app-st-alpha}
\end{figure}

\begin{figure}[!htbp] 
    \centering 
    \includegraphics[width=\linewidth]{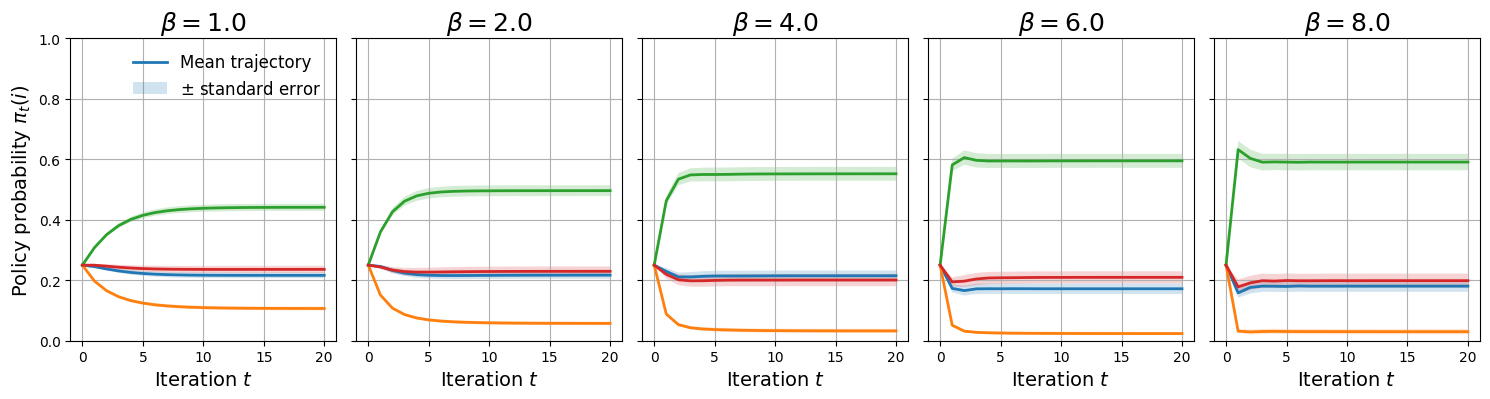}
    \caption{Effect of varying $\beta$ under ST preferences, with $\alpha = 0.8$ and $\lambda = 0.8$.} 
    \label{fig:app-st-beta}
\end{figure}

\begin{figure}[!htbp] 
    \centering 
    \includegraphics[width=\linewidth]{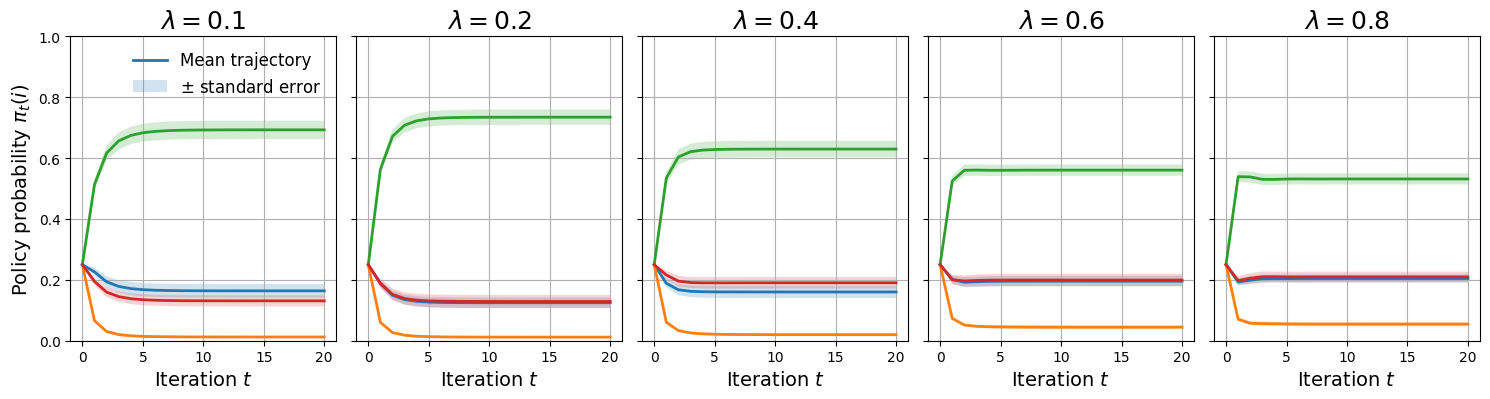}
    \caption{Effect of varying $\lambda$ under ST preferences, with $\alpha = 0.6$ and $\beta = 5.0$.} 
    \label{fig:app-st-lambda}
\end{figure}

In contrast, when parameters keep increasing 
so that the dynamics converge to highly concentrated, low entropy
policies,
the onset times and transient paths vary across realizations.
While all realizations ultimately converge to degenerate policies concentrated
on top-ranked actions, the lack of phase alignment across realizations makes
averaged trajectories difficult to interpret.
Therefore, for these parameter regimes, we illustrate representative individual
finite-sample realizations rather than reporting aggregated averages.
\Cref{fig:collapse} shows two such realizations generated from the same preference matrix, highlighting the diversity of collapse
dynamics induced by finite-sample variability.

\begin{figure}[!htbp] 
    \centering 
    \includegraphics[width=\linewidth]{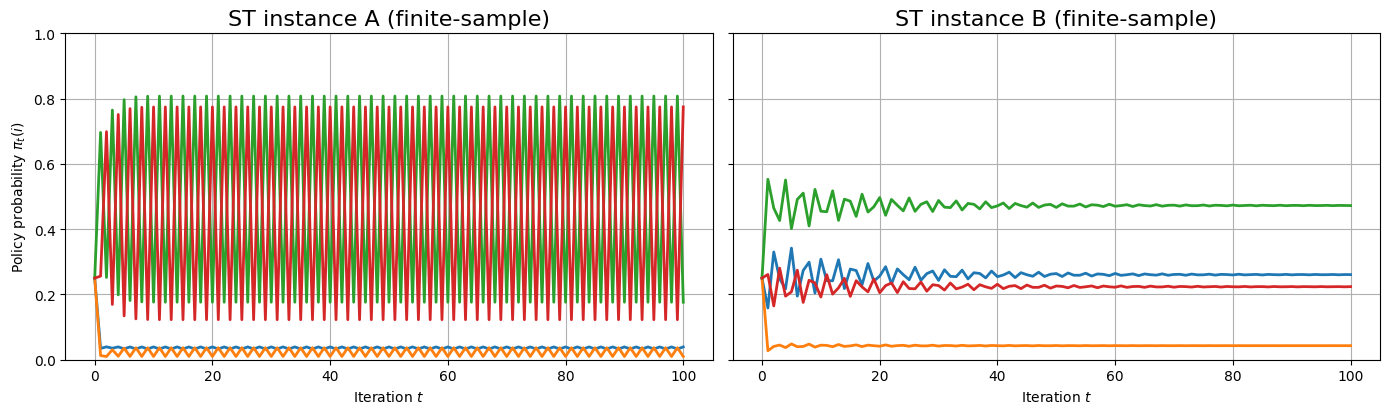}
    \caption{Two representative finite-sample realizations generated from the same
ST population preference matrix. $\alpha=0.5,\ \beta=10.0,\ \lambda=0.8$} . 
    \label{fig:collapse}
\end{figure}

\subsection{Complementary aggregated results}
\label{app:complementary-aggregated results}

This section reports aggregated statistics with $\beta$ and $\lambda$ treated separately and more possible parameter values $(\alpha,\beta,\lambda)$, as a complement to results in \Cref{sec:aggregated-results}. The goal of these experiments is not to introduce new theoretical regimes, but to
verify that the aggregated results predictions derived under the $\beta\lambda$
abstraction remain informative when the two parameters are separated.
This separation allows us to examine how $\beta$ (inverse temperature) and
$\lambda$ (on-policy sampling) influence convergence speed, collapse timing, and
oscillation amplitude in practice, while preserving the qualitative structure
predicted by the theory.

All statistics are computed over large collections of empirically constructed
preference matrices derived from the HelpSteer dataset, consisting of
$4924$ strongly transitive matrices and $118$ cyclic matrices.
For each parameter configuration, the MRS-IPO dynamics are run independently on
all preference matrices in the corresponding class, and we report the mean and
standard deviation of the relevant statistic across matrices.

\paragraph{Convergence/Oscillation in cyclic preferences. }

Figure~\ref{fig:cycle} reports the mean and standard deviation of cycle strength
across all cyclic preference matrices for independently varied
$(\alpha,\beta,\lambda)$.
When $\beta$ and $\lambda$ are small, the cycle strength remains close to zero for
most values of $\alpha$, indicating convergence or only weak oscillations.
As $\beta$ or $\lambda$ increases, persistent oscillatory behavior emerges, and
the cycle strength rises.
Larger values of $\alpha$ further amplify oscillations by reinforcing
self-referential dynamics.

The large standard deviations observed in strong amplification regimes reflect
substantial heterogeneity in oscillation amplitude, phase, and period across
cyclic preference matrices.

\begin{figure}[!htbp]
    \centering
    \includegraphics[width=0.8\linewidth]{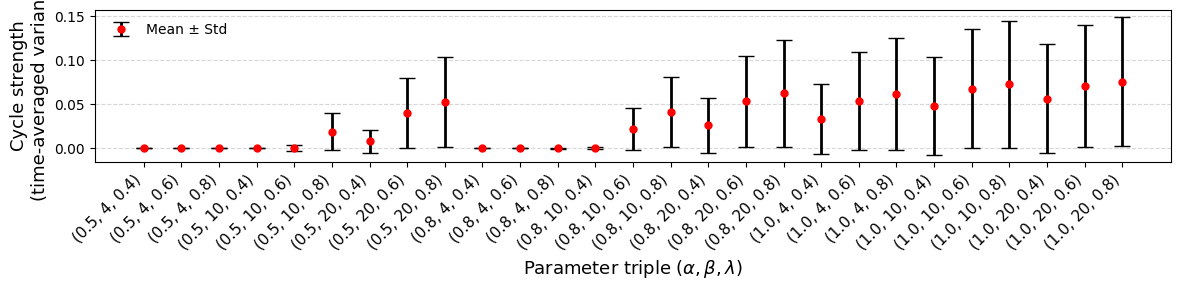}
    \caption{Mean $\pm$ standard deviation of cycle strength (time-averaged
    variance of $\{\pi_t\}$) across all $118$ cyclic preference matrices as
    $(\alpha,\beta,\lambda)$ varies.}
    \label{fig:cycle}
\end{figure}

\paragraph{Model collapse in ST preferences.} For ST preferences, \Cref{fig:entropy} reports the mean and standard deviation of
$H(\pi_T)$ across all strongly transitive preference matrices for a grid of
parameter triples $(\alpha,\beta,\lambda)$.
For small values of $\beta$ and $\lambda$, the final entropy remains high across
all $\alpha$, indicating convergence to non-collapsed policies.
Increasing $\beta$ leads to systematically lower entropy,
reflecting stronger preference amplification and faster concentration of mass on
top-ranked actions. However, increasing $\lambda$ leads to a slightly higher entropy.
For sufficiently large $\beta$ and $\lambda$, the entropy approaches to (near) zero,
corresponding to near deterministic collapse.

The standard deviation of entropy is small in weak amplification regimes,
indicating consistent convergence behavior across preference matrices.
In contrast, near the onset of collapse the variance increases substantially,
indicating that the degree of  model collapse relies crucially on the specific preference matrices.

\begin{figure}[!htbp]
    \centering
    \includegraphics[width=0.8\linewidth]{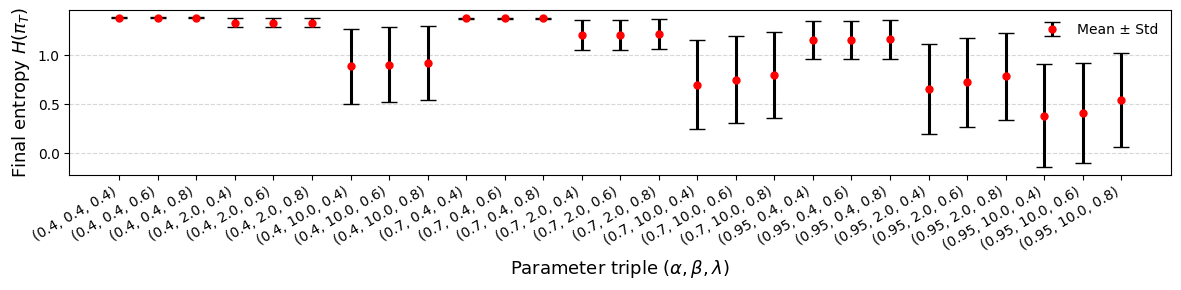}
    \caption{Mean $\pm$ standard deviation of the final policy entropy
    $H(\pi_T)$ across all $4924$ strongly transitive preference matrices as
    $(\alpha,\beta,\lambda)$ varies.}
    \label{fig:entropy}
\end{figure}

\paragraph{Summary.}
Across thousands of empirically constructed preference instances, these
statistics demonstrate that independently varying $\alpha$, $\beta$, and
$\lambda$ produces systematic and interpretable changes in convergence,
collapse, and oscillatory behavior.

\section{Experimental Results for MRS-DPO Dynamics}
\label{app:dpo_experiments}

In this section, we provide additional simulations and evaluations for the Mixed Reference/Sampling DPO (MRS-DPO) dynamics defined in \Cref{def:mixed_dpo_dyn}. 

\paragraph{Experimental setup and protocol.}
We utilize the same dataset generation pipeline and preference matrices construction protocol as in MRS-IPO, described in \Cref{app:dataset,app:construction}. To simulate the MRS-DPO update rule, we compute the implicit optimizer $\btheta^{\star}$ using numerical method (gradient-based optimization) solving the inner concave optimization problem induced by the DPO objective under pairwise logistic preferences.

\paragraph{Overview of experiments.}
To validate the theoretical properties of MRS-DPO, we structure our analysis into two main components:
\begin{itemize}
    \item \textbf{Parameter Sensitivity Analysis (Trajectories):} We first examine how the three hyperparameters---$\alpha$, $\beta$, and $\lambda$---affect the optimization trajectory. For both ST and cyclic preference structures, we conduct two parallel experiments:
    \begin{enumerate}
        \item \textit{Deterministic Dynamics on Specific $P$:} We use the specific illustrative matrices ($P_{ST}$ and $P_{cyc}$) defined in \Cref{app:form-illustrative-examples} to observe asymptotic behavior in a noiseless setting with fixed preference probabilities.
        \item \textit{Robustness to Finite Sample Noise:} We apply the finite sample noise injection method detailed in \Cref{app:protocol-noise} to verify that the observed phenomena are robust to estimation errors.
    \end{enumerate}
    \item \textbf{Aggregate Experimental Results:} Finally, mirroring the analysis in \Cref{aggregate_ipo}, we quantify the global convergence behavior by measuring the \textit{Policy Entropy} (for ST preferences) and \textit{Cycle Strength} (for Cyclic preferences) across a broader range of conditions.
\end{itemize}

\subsection{Simulation of MRS-DPO under single parameter variations}
\label{app:dpo_single_param}

\paragraph{Cyclic preferences.}

\Cref{fig:dpo-cyc-alpha,fig:dpo-cyc-beta,fig:dpo-cyc-lambda} present the experimental results for MRS-DPO dynamics a representative cyclic preference matrix $P_{cyc}$. In this idealized regime, the dynamics exhibit clear behavioral transitions governed by the stability conditions. Specifically, \Cref{fig:dpo-cyc-alpha} shows that as $\alpha$ increases, the system loses stability, shifting from convergence to the uniform distribution to a regime of persistent oscillations.
Similarly, \Cref{fig:dpo-cyc-beta} demonstrates that increasing $\beta$ strengthens the preference update, which destabilizes the uniform fixed point and drives the dynamics to oscillate strongly.
\Cref{fig:dpo-cyc-lambda} further confirms this trend: while low values of $\lambda$ allow sufficient exploration for convergence, increasing $\lambda$ (greater strength of on-policy sampling) promotes instability, resulting in oscillations of increasing magnitude.

\begin{figure}[!htbp] 
    \centering 
    \includegraphics[width=\linewidth]{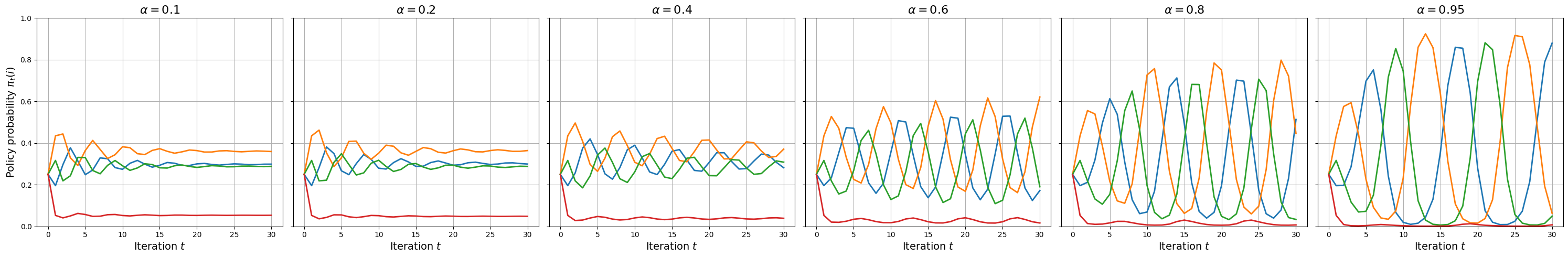}
    \caption{MRS-DPO: Effect of varying $\alpha$ under cyclic preferences ($P_{cyc}$), with $\beta = 2.0$ and $\lambda = 0.8$.} 
    \label{fig:dpo-cyc-alpha}
\end{figure}

\begin{figure}[!htbp] 
    \centering 
    \includegraphics[width=\linewidth]{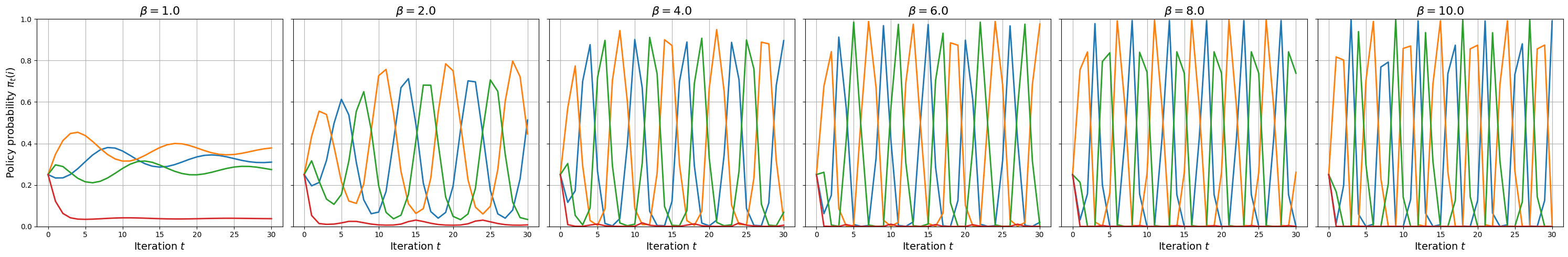}
    \caption{MRS-DPO: Effect of varying $\beta$ under cyclic preferences ($P_{cyc}$), with $\alpha = 0.8$ and $\lambda = 0.8$.} 
    \label{fig:dpo-cyc-beta}
\end{figure}

\begin{figure}[!htbp] 
    \centering 
    \includegraphics[width=\linewidth]{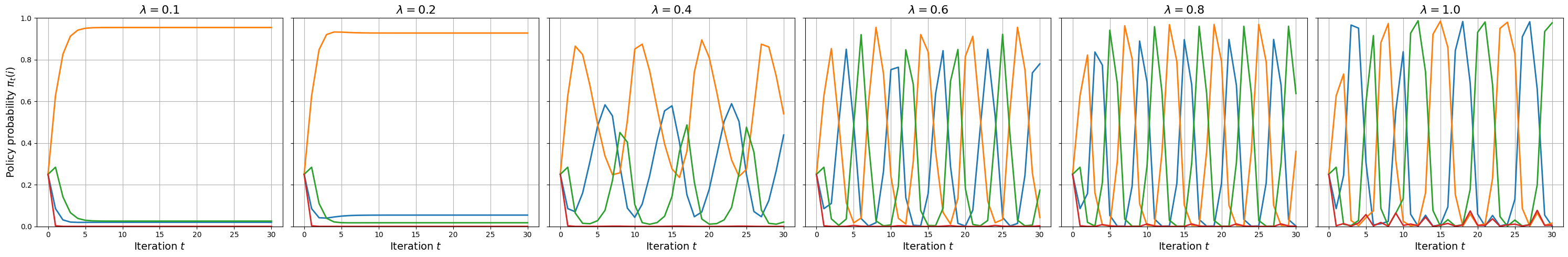}
    \caption{MRS-DPO: Effect of varying $\lambda$ under cyclic preferences ($P_{cyc}$), with $\alpha = 0.6$ and $\beta = 5.0$.} 
    \label{fig:dpo-cyc-lambda}
\end{figure}

Next, \Cref{fig:dpo-cyc-noisy-alpha,fig:dpo-cyc-noisy-beta,fig:dpo-cyc-noisy-lambda} present the corresponding experiments under finite sample noise. Here, the general trend mirrors the noiseless case: increasing $\alpha$, $\beta$, or $\lambda$ exacerbates instability.
However, the estimation noise introduces significant stochasticity, causing distinct realizations to oscillate with random phase shifts.
Consequently, averaging across these misaligned trajectories results in a phase cancellation effect, where the oscillations partially offset each other, artificially dampening the visualized amplitude.
Despite this cancellation artifact, the qualitative trend remains discernible: in low-parameter regimes, the mean trajectories are stable and constrained; conversely, as $\alpha$, $\beta$, or $\lambda$ increase, the aggregated dynamics exhibit non-vanishing variance, confirming that strong self-reinforcement undermines long-term policy stability.

\begin{figure}[!htbp] 
    \centering 
    \includegraphics[width=\linewidth]{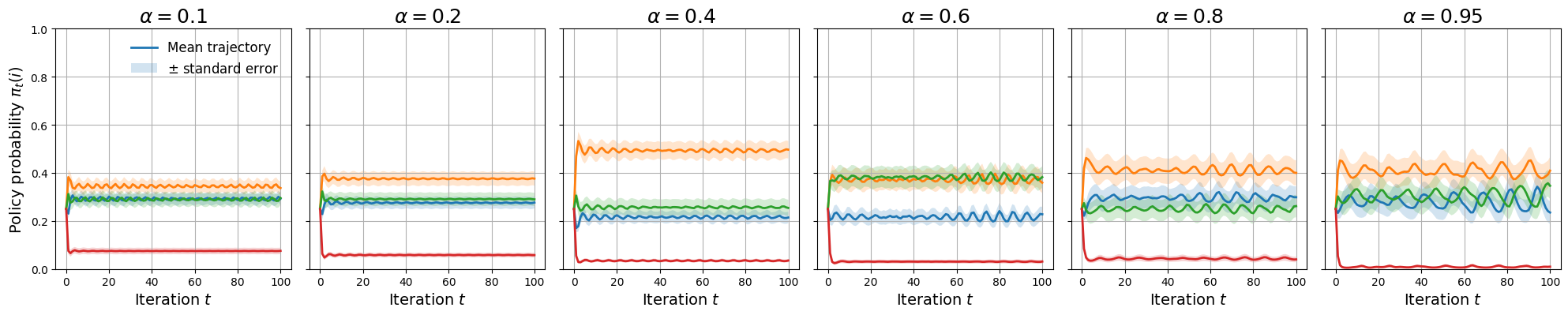}
    \caption{MRS-DPO (Cyclic Preferences with Finite Sample Noise): Effect of varying $\alpha$ with $\beta = 2.0$ and $\lambda = 0.5$.} 
    \label{fig:dpo-cyc-noisy-alpha}
\end{figure}

\begin{figure}[!htbp] 
    \centering 
    \includegraphics[width=\linewidth]{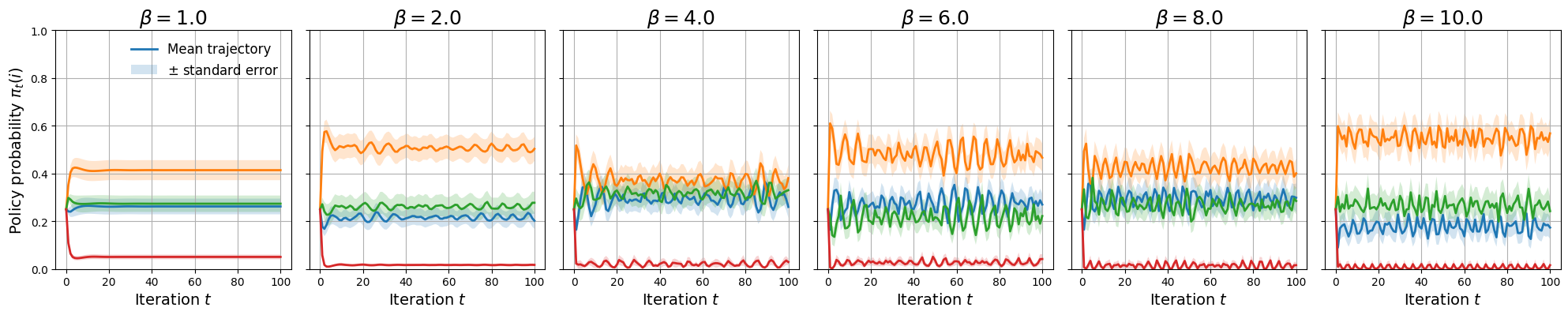}
    \caption{MRS-DPO (Cyclic Preferences with Finite Sample Noise): Effect of varying $\beta$ with $\alpha = 0.8$ and $\lambda = 0.5$.} 
    \label{fig:dpo-cyc-noisy-beta}
\end{figure}

\begin{figure}[!htbp] 
    \centering 
    \includegraphics[width=\linewidth]{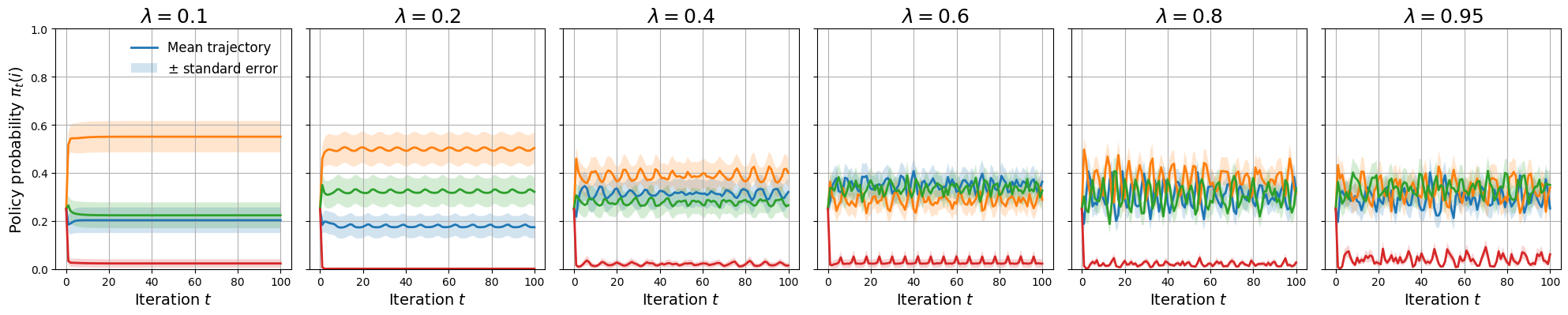}
    \caption{MRS-DPO (Cyclic Preferences with Finite Sample Noise): Effect of varying $\lambda$ with $\alpha = 0.8$ and $\beta = 5.0$.} 
    \label{fig:dpo-cyc-noisy-lambda}
\end{figure}

\paragraph{ST preferences.}
We then examine the dynamics under a representative strongly transitive preference matrix $P_{ST}$. In this idealized regime, we observe clear trends in how the parameters influence the convergence rate.
Specifically, \Cref{fig:dpo-st-noiseless-alpha} shows that increasing $\alpha$ significantly accelerates the convergence toward the deterministic policy.
Similarly, \Cref{fig:dpo-st-noiseless-beta} demonstrates that increasing the inverse temperature $\beta$ leads to a sharper concentration of probability mass on the optimal response.
In contrast, \Cref{fig:dpo-st-noiseless-lambda} indicates that the collapse trend is relatively insensitive to the sampling parameter $\lambda$ in this noiseless setting.

\begin{figure}[!htbp] 
    \centering 
    \includegraphics[width=\linewidth]{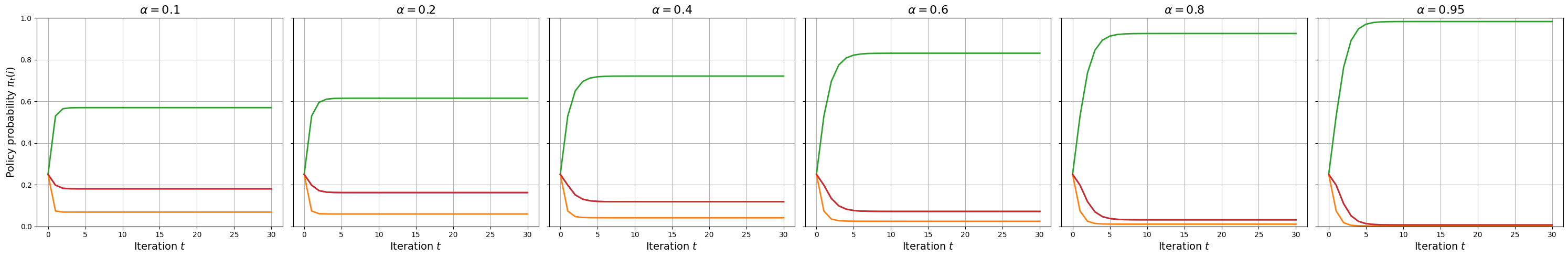}
    \caption{MRS-DPO: Effect of varying $\alpha$ under Strongly Transitive preferences ($P_{ST}$), with $\beta = 1.0$ and $\lambda = 0.6$.} 
    \label{fig:dpo-st-noiseless-alpha}
\end{figure}

\begin{figure}[!htbp] 
    \centering 
    \includegraphics[width=\linewidth]{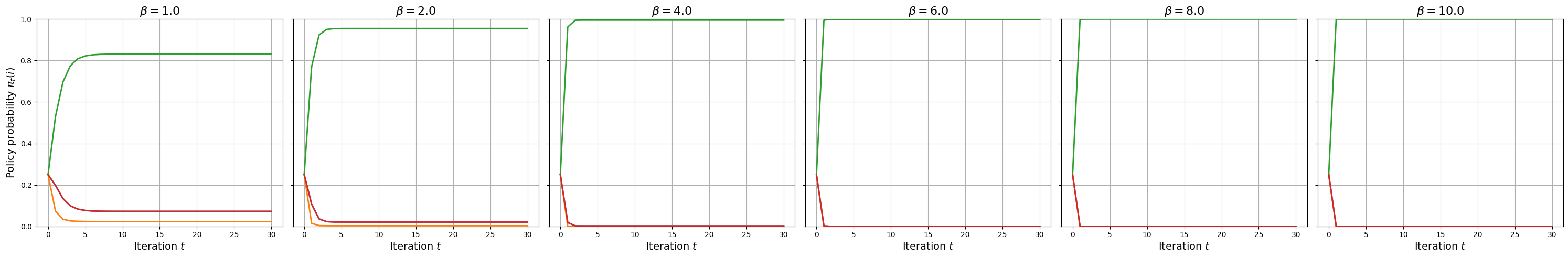}
    \caption{MRS-DPO: Effect of varying $\alpha$ under Strongly Transitive preferences ($P_{ST}$), with $\alpha = 0.6$ and $\lambda = 0.5$.} 
    \label{fig:dpo-st-noiseless-beta}
\end{figure}

\begin{figure}[!htbp] 
    \centering 
    \includegraphics[width=\linewidth]{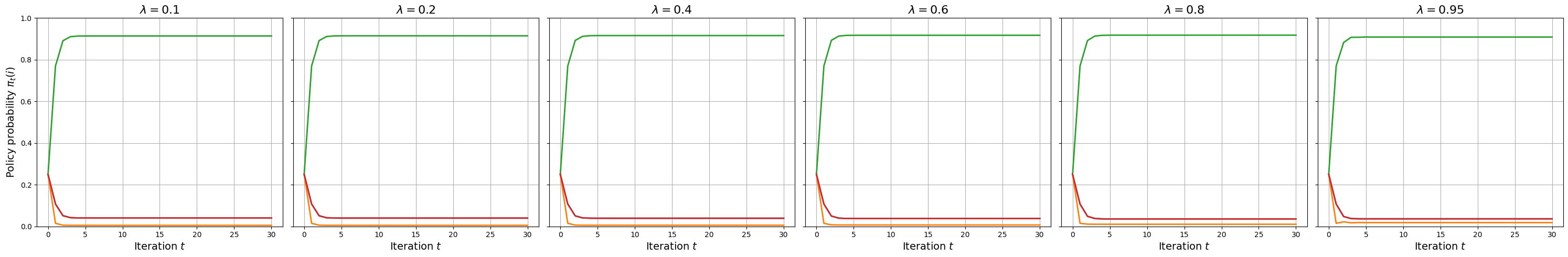}
    \caption{MRS-DPO: Effect of varying $\alpha$ under Strongly Transitive preferences ($P_{ST}$), with $\alpha = 0.6$ and $\beta = 2.0$.} 
    \label{fig:dpo-st-noiseless-lambda}
\end{figure}

Finally, \Cref{fig:dpo-st-noisy-alpha}--\Cref{fig:dpo-st-noisy-lambda} display the results under finite sample noise. In the regime of small parameters, the trends remain stable and consistent with the noiseless case: increasing $\alpha$, $\beta$, or $\lambda$ accelerates the concentration of probability mass, making the policy collapse more pronounced.
However, when these parameters are large, the trajectories exhibit significant variance and transient oscillations.
This instability arises from two distinct factors.
First, different realizations generated from the same preference matrix can exhibit completely different collapse behaviors.
Second, the theoretical stability conditions for DPO are considerably more stringent than for IPO.
As indicated by the analysis in \Cref{thm:mrs_dpo_stability_corrected}, the scaling constant governing the effective update size $\beta\lambda$ is usually large, implying that even moderate values of $\beta$ and $\lambda$ may violate the conditions required for monotonic convergence. Consequently, DPO dynamics are highly sensitive to sampling noise, resulting in the oscillatory or erratic behavior observed in the high-parameter regimes.

\begin{figure}[!htbp] 
    \centering 
    \includegraphics[width=\linewidth]{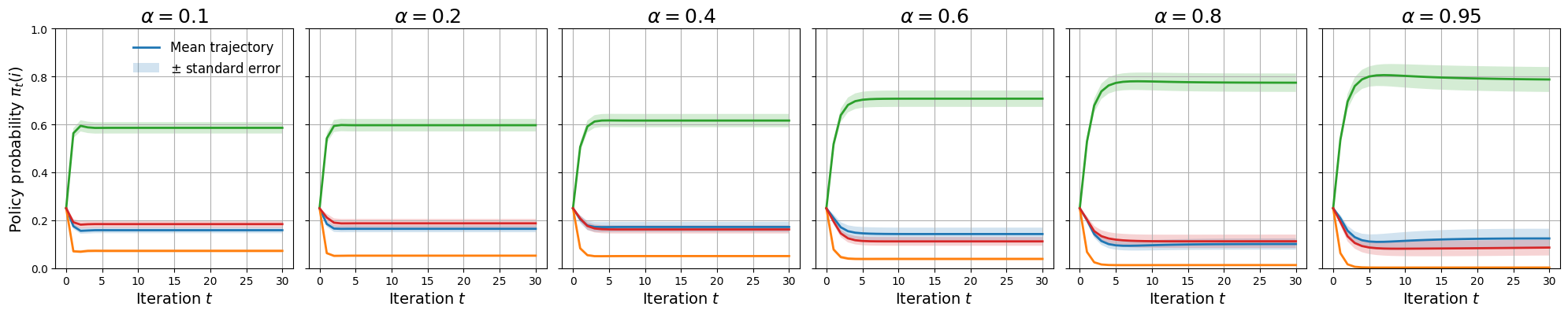}
    \caption{MRS-DPO (ST Preferences with Finite Sample Noise): Effect of varying $\alpha$ with $\beta = 1.0$ and $\lambda = 0.6$} 
    \label{fig:dpo-st-noisy-alpha}
\end{figure}

\begin{figure}[!htbp] 
    \centering 
    \includegraphics[width=\linewidth]{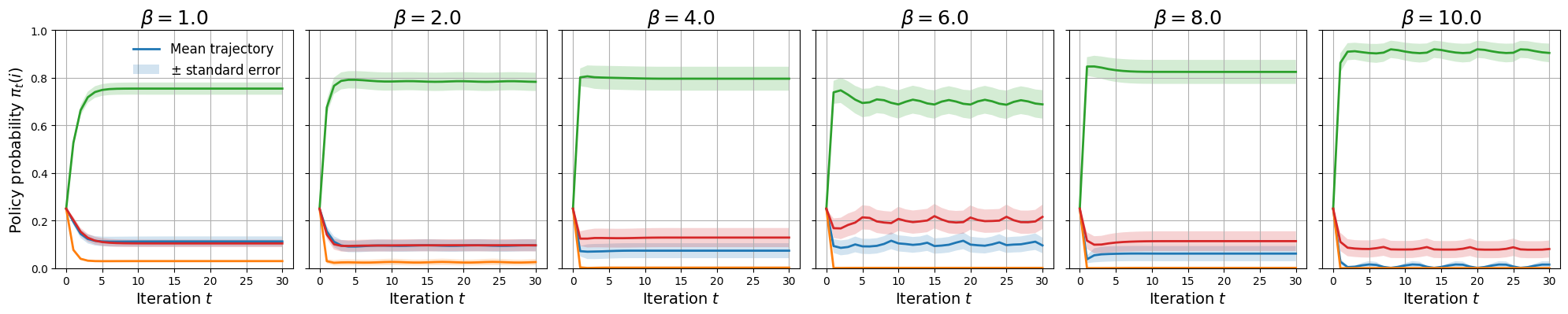}
    \caption{MRS-DPO (ST Preferences with Finite Sample Noise): Effect of varying $\beta$ with $\alpha = 0.6$ and $\lambda = 0.5$} 
    \label{fig:dpo-st-noisy-beta}
\end{figure}

\begin{figure}[!htbp] 
    \centering 
    \includegraphics[width=\linewidth]{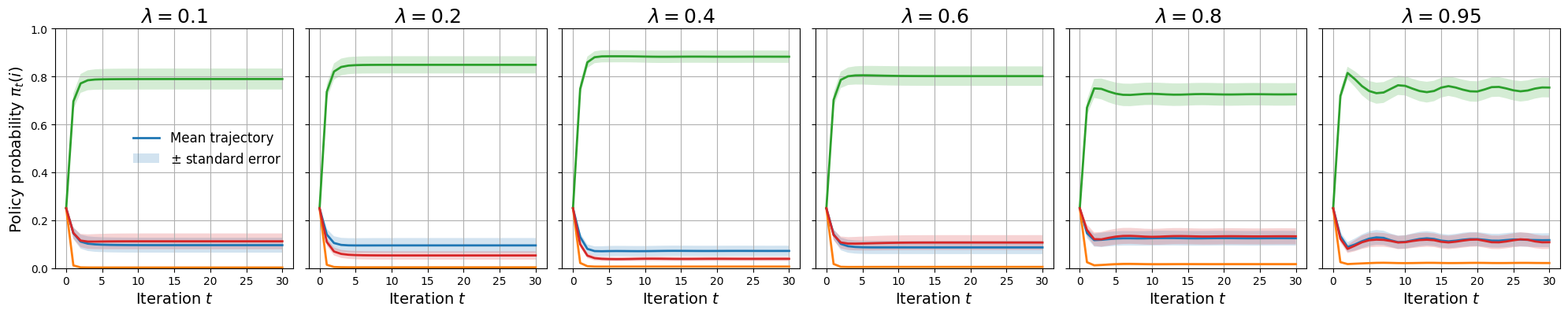}
    \caption{MRS-DPO (ST Preferences with Finite Sample Noise): Effect of varying $\lambda$ with $\alpha = 0.6$ and $\beta = 2.0$.} 
    \label{fig:dpo-st-noisy-lambda}
\end{figure}

\subsection{Complementary aggregated results for MRS-DPO}
\label{app:dpo_complementary_aggregated}

This section reports aggregated statistics for MRS-DPO dynamics with a broader grid of parameter values $(\alpha,\beta,\lambda)$, as a complement to the results in \Cref{app:dpo_experiments}.
The goal of these experiments is to verify that the aggregated predictions derived under the theoretical abstraction remain informative when the parameters are separated in the DPO objective.

All statistics are computed over the same large collections of empirically constructed preference matrices derived from the HelpSteer dataset, consisting of $4924$ strongly transitive matrices and $118$ cyclic matrices.
For each parameter configuration, the MRS-DPO dynamics are run independently on all preference matrices in the corresponding class, and we report the mean and standard deviation of the relevant statistic across matrices.

\paragraph{Convergence/Oscillation in cyclic preferences.}
\Cref{fig:dpo-aggregated-cycle} reports the mean and standard deviation of cycle strength across all cyclic preference matrices for independently varied $(\alpha,\beta,\lambda)$.
We observe a clear trend regarding the update and sampling parameters: increasing $\beta$ or $\lambda$ leads to a systematic increase in cycle strength, confirming that aggressive updates and skewed on-policy sampling are the primary drivers of instability in DPO.
However, in contrast to the sensitivity observed with $\beta$ and $\lambda$, the cycle strength appears relatively insensitive to the reference mixing parameter $\alpha$.
While the oscillation magnitude scales sharply with $\beta$ and $\lambda$, varying $\alpha$ results in only marginal changes to the aggregate cycle strength.
This suggests that for MRS-DPO, the global stability characteristics are dominated by the interaction between the inverse temperature $\beta$ and the on-policy sampling $\lambda$, with the reference anchor playing a less significant role in dampening or amplifying the limit cycles compared to the IPO case.

\begin{figure}[!htbp]
    \centering
    \includegraphics[width=0.8\linewidth]{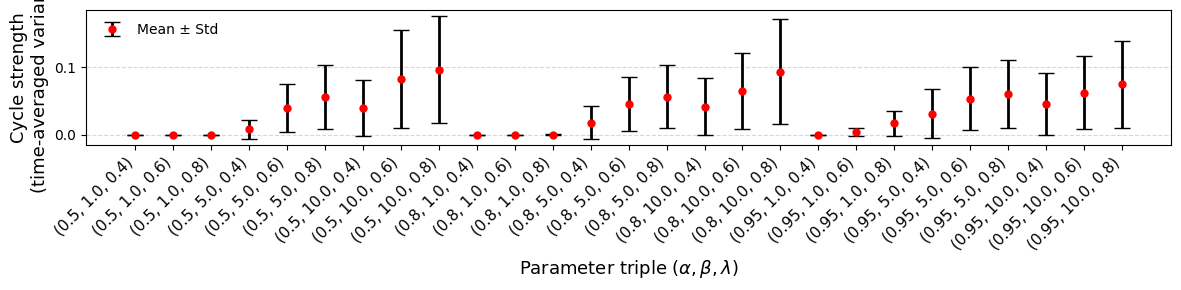}
    \caption{MRS-DPO: Mean $\pm$ standard deviation of cycle strength (time-averaged variance of $\{\bpi_t\}$) across all $118$ cyclic preference matrices as $(\alpha,\beta,\lambda)$ varies.}
    \label{fig:dpo-aggregated-cycle}
\end{figure}

\paragraph{Model collapse in ST preferences.}
For ST preferences, \Cref{fig:dpo-aggregated-entropy} reports the mean and standard deviation of the final policy entropy $H(\bpi_T)$ across all strongly transitive preference matrices.
Regarding parameter sensitivity, for small values of $\beta$ and $\lambda$, the final entropy remains very high across all $\alpha$.
As the inverse temperature $\beta$ increases, the entropy decreases significantly, identifying $\beta$ as the primary driver of probability concentration in DPO.
Increasing the reference weight $\alpha$ also contributes to entropy reduction, though the effect is weaker compared to $\beta$.
However, increasing the on-policy sampling $\lambda$ leads to a slight increase in the final entropy in this aggregate setting.

\begin{figure}[!htbp]
    \centering
    \includegraphics[width=0.8\linewidth]{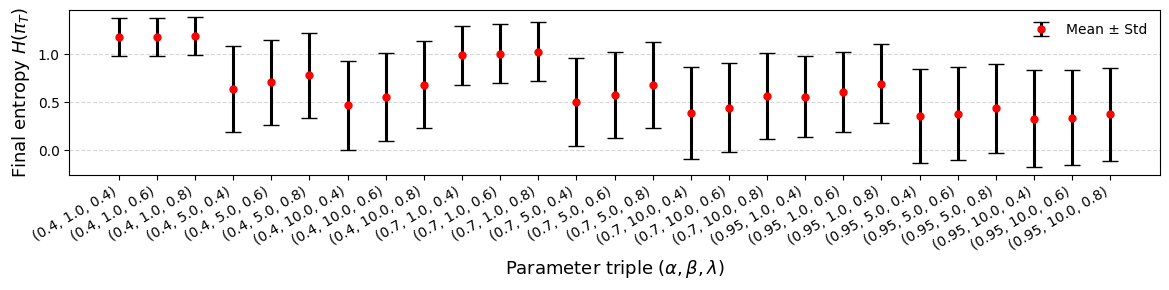}
    \caption{MRS-DPO: Mean $\pm$ standard deviation of the final policy entropy $H(\bpi_T)$ across all $4924$ strongly transitive preference matrices as $(\alpha,\beta,\lambda)$ varies.}
    \label{fig:dpo-aggregated-entropy}
\end{figure}

\newpage
\bibliographystyle{plainnat}
\bibliography{refs}

@article{shi2024crucial,
  title={The crucial role of samplers in online direct preference optimization},
  author={Shi, Ruizhe and Zhou, Runlong and Du, Simon S},
  journal={arXiv preprint arXiv:2409.19605},
  year={2024}
}

@article{xiao2025theoretical,
  title={Theoretical Tensions in {RLHF}: Reconciling Empirical Success with Inconsistencies in Social Choice Theory},
  author={Xiao, Jiancong and Shi, Zhekun and Liu, Kaizhao and Long, Qi and Su, Weijie J},
  journal={arXiv preprint arXiv:2506.12350},
  year={2025}
}

@article{feng2025pilaf,
  title={Pilaf: Optimal human preference sampling for reward modeling},
  author={Feng, Yunzhen and Kwiatkowski, Ariel and Zheng, Kunhao and Kempe, Julia and Duan, Yaqi},
  journal={arXiv preprint arXiv:2502.04270},
  year={2025}
}

@article{rafailov2023direct,
  title={Direct preference optimization: Your language model is secretly a reward model},
  author={Rafailov, Rafael and Sharma, Archit and Mitchell, Eric and Manning, Christopher D and Ermon, Stefano and Finn, Chelsea},
  journal={Advances in Neural Information Processing Systems},
  volume={36},
  pages={53728--53741},
  year={2023}
}

@article{liu2025statistical,
  title={Statistical impossibility and possibility of aligning {LLM}s with human preferences: From {C}ondorcet paradox to {N}ash equilibrium},
  author={Liu, Kaizhao and Long, Qi and Shi, Zhekun and Su, Weijie J and Xiao, Jiancong},
  journal={arXiv preprint arXiv:2503.10990},
  year={2025}
}

@inproceedings{azar2024general,
  title={A general theoretical paradigm to understand learning from human preferences},
  author={Azar, Mohammad Gheshlaghi and Guo, Zhaohan Daniel and Piot, Bilal and Munos, Remi and Rowland, Mark and Valko, Michal and Calandriello, Daniele},
  booktitle={International Conference on Artificial Intelligence and Statistics},
  pages={4447--4455},
  year={2024},
  organization={PMLR}
}

@inproceedings{noothigattu2020axioms,
 author = {Noothigattu, Ritesh and Peters, Dominik and Procaccia, Ariel D},
 booktitle = {Advances in Neural Information Processing Systems},
 editor = {H. Larochelle and M. Ranzato and R. Hadsell and M.F. Balcan and H. Lin},
 pages = {17745--17754},
 publisher = {Curran Associates, Inc.},
 title = {Axioms for Learning from Pairwise Comparisons},
 url = {https://proceedings.neurips.cc/paper_files/paper/2020/file/cdaa9b682e10c291d3bbadca4c96f5de-Paper.pdf},
 volume = {33},
 year = {2020}
}

@article{liu2023statistical,
  title={Statistical rejection sampling improves preference optimization},
  author={Liu, Tianqi and Zhao, Yao and Joshi, Rishabh and Khalman, Misha and Saleh, Mohammad and Liu, Peter J and Liu, Jialu},
  journal={arXiv preprint arXiv:2309.06657},
  year={2023}
}

@article{kim2025understanding,
  title={Understanding the Impact of Sampling Quality in Direct Preference Optimization},
  author={Kim, Kyung Rok and Bai, Yumo and Wang, Chonghuan and Chen, Guanting},
  journal={arXiv preprint arXiv:2506.04272},
  year={2025}
}

@article{ye2024online,
  title={Online iterative reinforcement learning from human feedback with general preference model},
  author={Ye, Chenlu and Xiong, Wei and Zhang, Yuheng and Dong, Hanze and Jiang, Nan and Zhang, Tong},
  journal={Advances in Neural Information Processing Systems},
  volume={37},
  pages={81773--81807},
  year={2024}
}

@article{calandriello2024human,
  title={Human alignment of large language models through online preference optimisation},
  author={Calandriello, Daniele and Guo, Daniel and Munos, Remi and Rowland, Mark and Tang, Yunhao and Pires, Bernardo Avila and Richemond, Pierre Harvey and Lan, Charline Le and Valko, Michal and Liu, Tianqi and others},
  journal={arXiv preprint arXiv:2403.08635},
  year={2024}
}

@article{dong2024rlhf,
  title={{RLHF} workflow: From reward modeling to online {RLHF}},
  author={Dong, Hanze and Xiong, Wei and Pang, Bo and Wang, Haoxiang and Zhao, Han and Zhou, Yingbo and Jiang, Nan and Sahoo, Doyen and Xiong, Caiming and Zhang, Tong},
  journal={arXiv preprint arXiv:2405.07863},
  year={2024}
}

@article{dong2023raft,
  title={Raft: Reward ranked finetuning for generative foundation model alignment},
  author={Dong, Hanze and Xiong, Wei and Goyal, Deepanshu and Zhang, Yihan and Chow, Winnie and Pan, Rui and Diao, Shizhe and Zhang, Jipeng and Shum, Kashun and Zhang, Tong},
  journal={arXiv preprint arXiv:2304.06767},
  year={2023}
}

@article{xiong2023iterative,
  title={Iterative preference learning from human feedback: Bridging theory and practice for {RLHF} under {KL}-constraint},
  author={Xiong, Wei and Dong, Hanze and Ye, Chenlu and Wang, Ziqi and Zhong, Han and Ji, Heng and Jiang, Nan and Zhang, Tong},
  journal={arXiv preprint arXiv:2312.11456},
  year={2023}
}

@article{tajwar2024preference,
  title={Preference fine-tuning of {LLM}s should leverage suboptimal, on-policy data},
  author={Tajwar, Fahim and Singh, Anikait and Sharma, Archit and Rafailov, Rafael and Schneider, Jeff and Xie, Tengyang and Ermon, Stefano and Finn, Chelsea and Kumar, Aviral},
  journal={arXiv preprint arXiv:2404.14367},
  year={2024}
}

@article{guo2024direct,
  title={Direct language model alignment from online {AI} feedback},
  author={Guo, Shangmin and Zhang, Biao and Liu, Tianlin and Liu, Tianqi and Khalman, Misha and Llinares, Felipe and Rame, Alexandre and Mesnard, Thomas and Zhao, Yao and Piot, Bilal and others},
  journal={arXiv preprint arXiv:2402.04792},
  year={2024}
}

@article{gorbatovski2024learn,
  title={Learn your reference model for real good alignment},
  author={Gorbatovski, Alexey and Shaposhnikov, Boris and Malakhov, Alexey and Surnachev, Nikita and Aksenov, Yaroslav and Maksimov, Ian and Balagansky, Nikita and Gavrilov, Daniil},
  journal={arXiv preprint arXiv:2404.09656},
  year={2024}
}

@article{bradley1952rank,
  title={Rank analysis of incomplete block designs: {I}. {T}he method of paired comparisons},
  author={Bradley, Ralph Allan and Terry, Milton E},
  journal={Biometrika},
  volume={39},
  number={3/4},
  pages={324--345},
  year={1952},
  publisher={JSTOR}
}

@misc{trl_dpotrainer_docs,
  author = {{Hugging Face}},
  title = {{TRL} Documentation: {DPO} Trainer},
  year = {2026},
  howpublished = {\url{https://huggingface.co/docs/trl/v0.27.0/dpo_trainer}},
  note = {Version v0.27.0, accessed 2026-01-23}
}

@inproceedings{wang2024helpsteer,
  title={Helpsteer: Multi-attribute helpfulness dataset for {SteerLM}},
  author={Wang, Zhilin and Dong, Yi and Zeng, Jiaqi and Adams, Virginia and Sreedhar, Makesh Narsimhan and Egert, Daniel and Delalleau, Olivier and Scowcroft, Jane and Kant, Neel and Swope, Aidan and others},
  booktitle={Proceedings of the 2024 Conference of the North American Chapter of the Association for Computational Linguistics: Human Language Technologies (Volume 1: Long Papers)},
  pages={3371--3384},
  year={2024}
}

@article{xu2024bpo,
  title={Bpo: Staying close to the behavior {LLM} creates better online {LLM} alignment},
  author={Xu, Wenda and Li, Jiachen and Wang, William Yang and Li, Lei},
  journal={arXiv preprint arXiv:2406.12168},
  year={2024}
}

@inproceedings{ge2024axioms,
 author = {Ge, Luise and Halpern, Daniel and Micha, Evi and Procaccia, Ariel D. and Shapira, Itai and Vorobeychik, Yevgeniy and Wu, Junlin},
 booktitle = {Advances in Neural Information Processing Systems},
 doi = {10.52202/079017-2557},
 editor = {A. Globerson and L. Mackey and D. Belgrave and A. Fan and U. Paquet and J. Tomczak and C. Zhang},
 pages = {80439--80465},
 publisher = {Curran Associates, Inc.},
 title = {Axioms for {AI} Alignment from Human Feedback},
 url = {https://proceedings.neurips.cc/paper_files/paper/2024/file/9328208f88ec69420031647e6ff97727-Paper-Conference.pdf},
 volume = {37},
 year = {2024}
}

@misc{openai2024gpt4ocard,
      title={{GPT-4o} System Card}, 
      author={OpenAI and Aaron Hurst and Adam Lerer and Adam P. Goucher and Adam Perelman and Aditya Ramesh and Aidan Clark and AJ Ostrow and Akila Welihinda and Alan Hayes and Alec Radford and Aleksander Madry and Alex Baker-Whitcomb and Alex Beutel and Alex Borzunov and Alex Carney and Alex Chow and Alex Kirillov and Alex Nichol and Alex Paino and Alex Renzin and Alex Tachard Passos and Alexander Kirillov and Alexi Christakis and Alexis Conneau and Ali Kamali and Allan Jabri and Allison Moyer and Allison Tam and Amadou Crookes and Amin Tootoochian and Amin Tootoonchian and Ananya Kumar and Andrea Vallone and Andrej Karpathy and Andrew Braunstein and Andrew Cann and Andrew Codispoti and Andrew Galu and Andrew Kondrich and Andrew Tulloch and Andrey Mishchenko and Angela Baek and Angela Jiang and Antoine Pelisse and Antonia Woodford and Anuj Gosalia and Arka Dhar and Ashley Pantuliano and Avi Nayak and Avital Oliver and Barret Zoph and Behrooz Ghorbani and Ben Leimberger and Ben Rossen and Ben Sokolowsky and Ben Wang and Benjamin Zweig and Beth Hoover and Blake Samic and Bob McGrew and Bobby Spero and Bogo Giertler and Bowen Cheng and Brad Lightcap and Brandon Walkin and Brendan Quinn and Brian Guarraci and Brian Hsu and Bright Kellogg and Brydon Eastman and Camillo Lugaresi and Carroll Wainwright and Cary Bassin and Cary Hudson and Casey Chu and Chad Nelson and Chak Li and Chan Jun Shern and Channing Conger and Charlotte Barette and Chelsea Voss and Chen Ding and Cheng Lu and Chong Zhang and Chris Beaumont and Chris Hallacy and Chris Koch and Christian Gibson and Christina Kim and Christine Choi and Christine McLeavey and Christopher Hesse and Claudia Fischer and Clemens Winter and Coley Czarnecki and Colin Jarvis and Colin Wei and Constantin Koumouzelis and Dane Sherburn and Daniel Kappler and Daniel Levin and Daniel Levy and David Carr and David Farhi and David Mely and David Robinson and David Sasaki and Denny Jin and Dev Valladares and Dimitris Tsipras and Doug Li and Duc Phong Nguyen and Duncan Findlay and Edede Oiwoh and Edmund Wong and Ehsan Asdar and Elizabeth Proehl and Elizabeth Yang and Eric Antonow and Eric Kramer and Eric Peterson and Eric Sigler and Eric Wallace and Eugene Brevdo and Evan Mays and Farzad Khorasani and Felipe Petroski Such and Filippo Raso and Francis Zhang and Fred von Lohmann and Freddie Sulit and Gabriel Goh and Gene Oden and Geoff Salmon and Giulio Starace and Greg Brockman and Hadi Salman and Haiming Bao and Haitang Hu and Hannah Wong and Haoyu Wang and Heather Schmidt and Heather Whitney and Heewoo Jun and Hendrik Kirchner and Henrique Ponde de Oliveira Pinto and Hongyu Ren and Huiwen Chang and Hyung Won Chung and Ian Kivlichan and Ian O'Connell and Ian O'Connell and Ian Osband and Ian Silber and Ian Sohl and Ibrahim Okuyucu and Ikai Lan and Ilya Kostrikov and Ilya Sutskever and Ingmar Kanitscheider and Ishaan Gulrajani and Jacob Coxon and Jacob Menick and Jakub Pachocki and James Aung and James Betker and James Crooks and James Lennon and Jamie Kiros and Jan Leike and Jane Park and Jason Kwon and Jason Phang and Jason Teplitz and Jason Wei and Jason Wolfe and Jay Chen and Jeff Harris and Jenia Varavva and Jessica Gan Lee and Jessica Shieh and Ji Lin and Jiahui Yu and Jiayi Weng and Jie Tang and Jieqi Yu and Joanne Jang and Joaquin Quinonero Candela and Joe Beutler and Joe Landers and Joel Parish and Johannes Heidecke and John Schulman and Jonathan Lachman and Jonathan McKay and Jonathan Uesato and Jonathan Ward and Jong Wook Kim and Joost Huizinga and Jordan Sitkin and Jos Kraaijeveld and Josh Gross and Josh Kaplan and Josh Snyder and Joshua Achiam and Joy Jiao and Joyce Lee and Juntang Zhuang and Justyn Harriman and Kai Fricke and Kai Hayashi and Karan Singhal and Katy Shi and Kavin Karthik and Kayla Wood and Kendra Rimbach and Kenny Hsu and Kenny Nguyen and Keren Gu-Lemberg and Kevin Button and Kevin Liu and Kiel Howe and Krithika Muthukumar and Kyle Luther and Lama Ahmad and Larry Kai and Lauren Itow and Lauren Workman and Leher Pathak and Leo Chen and Li Jing and Lia Guy and Liam Fedus and Liang Zhou and Lien Mamitsuka and Lilian Weng and Lindsay McCallum and Lindsey Held and Long Ouyang and Louis Feuvrier and Lu Zhang and Lukas Kondraciuk and Lukasz Kaiser and Luke Hewitt and Luke Metz and Lyric Doshi and Mada Aflak and Maddie Simens and Madelaine Boyd and Madeleine Thompson and Marat Dukhan and Mark Chen and Mark Gray and Mark Hudnall and Marvin Zhang and Marwan Aljubeh and Mateusz Litwin and Matthew Zeng and Max Johnson and Maya Shetty and Mayank Gupta and Meghan Shah and Mehmet Yatbaz and Meng Jia Yang and Mengchao Zhong and Mia Glaese and Mianna Chen and Michael Janner and Michael Lampe and Michael Petrov and Michael Wu and Michele Wang and Michelle Fradin and Michelle Pokrass and Miguel Castro and Miguel Oom Temudo de Castro and Mikhail Pavlov and Miles Brundage and Miles Wang and Minal Khan and Mira Murati and Mo Bavarian and Molly Lin and Murat Yesildal and Nacho Soto and Natalia Gimelshein and Natalie Cone and Natalie Staudacher and Natalie Summers and Natan LaFontaine and Neil Chowdhury and Nick Ryder and Nick Stathas and Nick Turley and Nik Tezak and Niko Felix and Nithanth Kudige and Nitish Keskar and Noah Deutsch and Noel Bundick and Nora Puckett and Ofir Nachum and Ola Okelola and Oleg Boiko and Oleg Murk and Oliver Jaffe and Olivia Watkins and Olivier Godement and Owen Campbell-Moore and Patrick Chao and Paul McMillan and Pavel Belov and Peng Su and Peter Bak and Peter Bakkum and Peter Deng and Peter Dolan and Peter Hoeschele and Peter Welinder and Phil Tillet and Philip Pronin and Philippe Tillet and Prafulla Dhariwal and Qiming Yuan and Rachel Dias and Rachel Lim and Rahul Arora and Rajan Troll and Randall Lin and Rapha Gontijo Lopes and Raul Puri and Reah Miyara and Reimar Leike and Renaud Gaubert and Reza Zamani and Ricky Wang and Rob Donnelly and Rob Honsby and Rocky Smith and Rohan Sahai and Rohit Ramchandani and Romain Huet and Rory Carmichael and Rowan Zellers and Roy Chen and Ruby Chen and Ruslan Nigmatullin and Ryan Cheu and Saachi Jain and Sam Altman and Sam Schoenholz and Sam Toizer and Samuel Miserendino and Sandhini Agarwal and Sara Culver and Scott Ethersmith and Scott Gray and Sean Grove and Sean Metzger and Shamez Hermani and Shantanu Jain and Shengjia Zhao and Sherwin Wu and Shino Jomoto and Shirong Wu and Shuaiqi and Xia and Sonia Phene and Spencer Papay and Srinivas Narayanan and Steve Coffey and Steve Lee and Stewart Hall and Suchir Balaji and Tal Broda and Tal Stramer and Tao Xu and Tarun Gogineni and Taya Christianson and Ted Sanders and Tejal Patwardhan and Thomas Cunninghman and Thomas Degry and Thomas Dimson and Thomas Raoux and Thomas Shadwell and Tianhao Zheng and Todd Underwood and Todor Markov and Toki Sherbakov and Tom Rubin and Tom Stasi and Tomer Kaftan and Tristan Heywood and Troy Peterson and Tyce Walters and Tyna Eloundou and Valerie Qi and Veit Moeller and Vinnie Monaco and Vishal Kuo and Vlad Fomenko and Wayne Chang and Weiyi Zheng and Wenda Zhou and Wesam Manassra and Will Sheu and Wojciech Zaremba and Yash Patil and Yilei Qian and Yongjik Kim and Youlong Cheng and Yu Zhang and Yuchen He and Yuchen Zhang and Yujia Jin and Yunxing Dai and Yury Malkov},
      year={2024},
      eprint={2410.21276},
      archivePrefix={arXiv},
      primaryClass={cs.CL},
      url={https://arxiv.org/abs/2410.21276}, 
}

@article{anthropic2024claude,
  title={The {C}laude 3 model family: Opus, {S}onnet, {H}aiku},
  author={Anthropic, {AI}},
  journal={Claude-3 Model Card},
  year={2024}
}

@misc{guan2025deliberative,
      title={Deliberative Alignment: Reasoning Enables Safer Language Models}, 
      author={Melody Y. Guan and Manas Joglekar and Eric Wallace and Saachi Jain and Boaz Barak and Alec Helyar and Rachel Dias and Andrea Vallone and Hongyu Ren and Jason Wei and Hyung Won Chung and Sam Toyer and Johannes Heidecke and Alex Beutel and Amelia Glaese},
      year={2025},
      eprint={2412.16339},
      archivePrefix={arXiv},
      primaryClass={cs.CL},
      url={https://arxiv.org/abs/2412.16339}, 
}

@inproceedings{ouyang2022training,
 author = {Ouyang, Long and Wu, Jeffrey and Jiang, Xu and Almeida, Diogo and Wainwright, Carroll and Mishkin, Pamela and Zhang, Chong and Agarwal, Sandhini and Slama, Katarina and Ray, Alex and Schulman, John and Hilton, Jacob and Kelton, Fraser and Miller, Luke and Simens, Maddie and Askell, Amanda and Welinder, Peter and Christiano, Paul F and Leike, Jan and Lowe, Ryan},
 booktitle = {Advances in Neural Information Processing Systems},
 editor = {S. Koyejo and S. Mohamed and A. Agarwal and D. Belgrave and K. Cho and A. Oh},
 pages = {27730--27744},
 publisher = {Curran Associates, Inc.},
 title = {Training language models to follow instructions with human feedback},
 url = {https://proceedings.neurips.cc/paper_files/paper/2022/file/b1efde53be364a73914f58805a001731-Paper-Conference.pdf},
 volume = {35},
 year = {2022}
}

@incollection{thurstone2017law,
  title={A law of comparative judgment},
  author={Thurstone, Louis L},
  booktitle={Scaling},
  pages={81--92},
  year={2017},
  publisher={Routledge}
}

@article{pan2025matters,
  title={What Matters in Data for {DPO}?},
  author={Pan, Yu and Cai, Zhongze and Chen, Guanting and Zhong, Huaiyang and Wang, Chonghuan},
  journal={arXiv preprint arXiv:2508.18312},
  year={2025}
}


\end{document}